\documentclass[journal, twocolumn, 10pt]{IEEEtran}
\usepackage{hyperref}
\usepackage{url}
\usepackage{graphicx}
\usepackage{epstopdf}
\usepackage{algorithm}
\usepackage{algorithmicx}
\usepackage{algpseudocode}
\usepackage{cite}
\usepackage{color}
\usepackage{amsmath,amssymb,amsfonts,amsthm}
\usepackage{setspace}
\usepackage{multirow}
\usepackage{mathtools}
\usepackage{subcaption}
\usepackage{fancyhdr}

\newtheorem{thm} {Theorem}

\newtheorem{lemma} {Lemma}

\newcommand{\asconv}{\overset{\text{a.s.}}{\longrightarrow}}
\newcommand{\ra}{\rightarrow}
\newcommand{\LB}{\left\{}
\newcommand{\RB}{\right\}}
\newcommand{\Lb}{\left[}
\newcommand{\Rb}{\right]}
\newcommand{\lb}{\left(}
\newcommand{\rb}{\right)}
\newcommand{\trace}{\textnormal{trace}}
\newcommand{\diag}{\textnormal{diag}}

\newcommand{\convd}{\overset{d}{\longrightarrow}}
\newcommand{\vectorize}{\textbf{vec}}

\newcommand{\nmin}{{n_{\min}}}
\newcommand{\nmax}{{n_{\max}}}

\newcommand{\wini}{w^{\textnormal{ini}}}

\newcommand{\hatmijl}{\widehat{m}_{ij}^{(\ell)}}
\newcommand{\hn}{\widehat{n}}

\newcommand{\htLBw}{\widehat{t}^{\mathbf{w}}_{\text{LB}}}

\newcommand{\tLB}{t_{\text{LB}}}
\newcommand{\tUBw}{t_{\text{UB}}^{\mathbf{w}}}
\newcommand{\tLBw}{t_{\text{LB}}^{\mathbf{w}}}
\newcommand{\tw}{t^{\mathbf{w}}}
\newcommand{\twstar}{{t^{\mathbf{w}}}^*}
\newcommand{\cwstar}{{c^{\mathbf{w}}}^*}

\newcommand{\hpijl}{\widehat{p}_{ij}^{(\ell)}}
\newcommand{\hpl}{\widehat{p}^{(\ell)}}

\newcommand{\htijl}{\widehat{t}_{ij}^{(\ell)}}
\newcommand{\htl}{\widehat{t}^{(\ell)}}
\newcommand{\htw}{\widehat{t}^{\mathbf{w}}}
\newcommand{\htwmax}{\widehat{t}^{\mathbf{w}}_{\max}}
\newcommand{\htmaxl}{\widehat{t}^{(\ell)}_{\max}}
\newcommand{\htinil}{\widehat{t}^{(\ell)}_{\textnormal{ini}}}

\newcommand{\pl}{p^{(\ell)}}
\newcommand{\tl}{t^{(\ell)}}
\newcommand{\twmax}{t^{\mathbf{w}}_{\max}}
\newcommand{\twmin}{t^{\mathbf{w}}_{\min}}
\newcommand{\tlmax}{t^{(\ell)}_{\max}}

\newcommand{\pijl}{p_{ij}^{(\ell)}}
\newcommand{\tijl}{t_{ij}^{(\ell)}}

\newcommand{\bone}{\mathbf{1}}
\newcommand{\bzero}{\mathbf{0}}

\newcommand{\bw}{\mathbf{w}}
\newcommand{\bwini}{\mathbf{w}^{\textnormal{ini}}}

\newcommand{\bsw}{\mathbf{s}^\bw}
\newcommand{\ba}{\mathbf{a}}

\newcommand{\bv}{\mathbf{v}}

\newcommand{\bvk}{\mathbf{v_k}}

\newcommand{\bnu}{\boldsymbol{\nu}}

\newcommand{\bB}{\mathbf{B}}

\newcommand{\Wbar}{\overline{W}}

\newcommand{\Wbarijl}{\overline{W}_{ij}^{(\ell)}}
\newcommand{\Wbarl}{\overline{W}^{(\ell)}}

\newcommand{\hWbarijl}{\widehat{\Wbar}_{ij}^{(\ell)}}
\newcommand{\hWbarl}{\widehat{\Wbar}^{(\ell)}}

\newcommand{\SK}{S_{2:K}}
\newcommand{\bX}{\mathbf{X}}
\newcommand{\bY}{\mathbf{Y}}
\newcommand{\bZ}{\mathbf{Z}}

\newcommand{\bI}{\mathbf{I}}

\newcommand{\bO}{\mathbf{O}}
\newcommand{\bU}{\mathbf{U}}

\newcommand{\btLw}{\widetilde{\mathbf{L}}^\bw}
\newcommand{\btY}{\widetilde{\mathbf{Y}}}

\newcommand{\bDelta}{\mathbf{\Delta}}

\newcommand{\bAl}{\mathbf{A}^{(\ell)}}
\newcommand{\bAw}{\mathbf{A}^{\bw}}
\newcommand{\bV}{\mathbf{V}}

\newcommand{\bWl}{\mathbf{W}^{(\ell)}}
\newcommand{\bWw}{\mathbf{W}^{\bw}}
\newcommand{\bF}{\mathbf{F}}
\newcommand{\bFl}{\mathbf{F}^{(\ell)}}
\newcommand{\bFw}{\mathbf{F}^\bw}
\newcommand{\bM}{\mathbf{M}}
\newcommand{\bD}{\mathbf{D}}
\newcommand{\bL}{\mathbf{L}}
\newcommand{\bLw}{\mathbf{L}^\bw}
\newcommand{\bLl}{\mathbf{L}^{(\ell)}}
\newcommand{\bC}{\mathbf{C}}
\newcommand{\bCl}{\mathbf{C}^{(\ell)}}
\newcommand{\bLambda}{\mathbf{\Lambda}}

\newcommand{\bSw}{\mathbf{S}^\bw}
\newcommand{\bSl}{\mathbf{S}^{(\ell)}}

\newcommand{\bWt}{\widetilde{\mathbf{W}}}
\newcommand{\cS}{\mathcal{S}}
\newcommand{\cC}{\mathcal{C}}
\newcommand{\cW}{\mathcal{W}}

\newcommand{\cV}{\mathcal{V}}
\newcommand{\cE}{\mathcal{E}}
\newcommand{\cT}{\mathcal{T}}

\newcommand{\bbR}{\mathbb{R}}

\newcommand{\hbLl}{\widehat{\bL}^{(\ell)}}
\newcommand{\hbCijl}{\widehat{\bC}_{ij}^{(\ell)}}

\newcommand{\hbFijl}{\widehat{\bF}_{ij}^{(\ell)}}

\begin{document}

\title{Multilayer Spectral Graph Clustering via Convex Layer Aggregation: Theory and Algorithms
	}

\author{Pin-Yu~Chen and Alfred O. Hero III,~\emph{Fellow},~\emph{IEEE}
	\thanks{P.-Y. Chen is with AI Foundations, IBM Thomas J. Watson Research Center, Yorktown Heights, NY 10598, USA. Email : pin-yu.chen@ibm.com.
		 A. O. Hero is with the Department of Electrical Engineering and Computer Science, University of Michigan, Ann Arbor, MI 48109, USA. Email : hero@umich.edu.}
	\thanks{This work was conducted while P.-Y. Chen was at the  University of Michigan, Ann Arbor, and 
		 has been partially supported by the Army Research Office
		(ARO), grants  W911NF-15-1-0479 and W911NF-15-1-0241, and by the Consortium for Verification
		Technology under Department of Energy National Nuclear Security Administration,
		award DE-NA0002534. Part of this work was presented at IEEE GlobalSIP 2016.}
}

\maketitle
\thispagestyle{empty}
\begin{abstract}
Multilayer graphs are commonly used for representing different relations between entities and handing heterogeneous data processing tasks. Non-standard multilayer graph clustering methods are needed for assigning clusters to a common multilayer node set and for combining information from each layer.  This paper presents a multilayer spectral graph clustering (SGC) framework that performs convex layer aggregation. Under a multilayer signal plus noise model, we provide a phase transition analysis of clustering reliability. Moreover, we use the phase transition criterion to propose a multilayer iterative model order selection algorithm (MIMOSA) for multilayer SGC, which features automated cluster assignment and layer weight adaptation, and provides statistical clustering reliability guarantees. Numerical simulations on synthetic multilayer graphs verify the phase transition analysis, and experiments on real-world multilayer graphs show that MIMOSA is competitive or better than other clustering methods. 
\end{abstract}


\section{Introduction}
Multilayer graphs provide a framework for representing multiple types of relations between entities, represented as nodes. In a multilayer graph each layer describes a specific type of relation among pairs of nodes that are shared across layers. For example, in multi-relational social networks,
two layers might correspond to friendship relations and business relations, respectively.  
In temporal networks, each layer might correspond to a snapshot of the entire network at a sampled time instant. Multilayer graphs can be incorporated into in many signal processing and data mining techniques, including inference of mixture models \cite{Oselio14,xu2014dynamic}, tensor decomposition \cite{Domenico13tensor}, information extraction \cite{oselio2015information}, multi-view learning and processing \cite{zhou2007spectral}, graph wavelet transforms \cite{leonardi2013tight}, principal component analysis and dictionary learning \cite{benzi2016principal,CPY16ICASSP}, anomaly detection \cite{park2013anomaly}, and community detection \cite{kivela2014multilayer,kim2015community}, among others.

The objective of multilayer graph clustering is to find a consensus cluster assignment on each node in the common node set by combining connectivity patterns in each layer.
Multilayer graph clustering differs from single-layer graph clustering in several respects: 
(1) the information about cluster membership must be aggregated from multiple layers; (2) the performance of  multilayer graph clustering will depend on the proportion of noisy edges across layers. This paper proposes a multilayer spectral graph clustering (SGC) algorithm that uses convex layer aggregation. Specifically, the algorithm performs SGC on an weighted average of the adjacency matrices of the layers, where the weights are non-negative and sum to one. 
We establish phase transitions in multilayer graph clustering in the convex layer-aggregated graph as a function of the noisy edge connection parameters of each layer under a multilayer signal plus noise model. Our phase transition analysis shows that when one sweeps over noise levels, there exists a critical threshold below (above) which multilayer SGC will yield correct (incorrect) clusters.
This critical phase transition threshold depends on the layer weights used to aggregate the multilayer graph into a single-layer graph in addition to the topology of the multilayer graph. 
Numerical experiments on synthetic multilayer graphs are conducted to verify the phase transitions of the proposed method.
Moreover, we propose a multilayer iterative model order selection algorithm (MIMOSA) that incorporates automated layer weight adaptation and cluster assignment. Experimental results on real-world multilayer graphs show that MIMOSA has competitive clustering performance  to (1) the baseline approach of assigning uniform layer weights, (2) the greedy multilayer modularity maximization method \cite{mucha2010community}, and (3) the subspace approach \cite{dong2014clustering}.

This paper makes two principal contributions. First, under a general multilayer signal plus noise model, we establish a phase transition on the performance of multilayer SGC. Fixing the within-cluster edges (signals) and varying the parameters governing the between-cluster edges (noises), we show that the clustering accuracy of multilayer SGC can be separated into two regimes: a reliable regime where high clustering accuracy can be guaranteed, and an unreliable regime where high clustering accuracy is impossible. Moreover, we specify upper and lower bounds on the critical noise value that separates these two regimes, which is an analytical function of  the signal strength, the number of clusters, the cluster size distributions, and the layer weights for convex layer aggregation. The bounds become exact in the case of identical cluster sizes. The analysis specifies the interplay between the layer weights, the multilayer graph connectivity structure in terms of eigenspectrum, and the performance of multilayer SGC via convex layer aggregation. The analysis also provides a criterion for assessing the quality of clustering results, which leads to the second contribution: the introduction of a new multilayer clustering algorithm with automated model order selection (number of clusters).  This algorithm, called the multilayer iteration model order selection algorithm (MIMOSA), selects both the model order and the layer weights and results in improved performance. 
MIMOSA incrementally increases the number of clusters, adapts layer weight assignment, and adopts a series of statistical clustering reliability tests. To illustrate the proposed MIMOSA approach, we apply it 
to several real-world multilayer graphs pertaining to social, biological, collaboration and  transportation networks.

The rest of this paper is organized as follows. Sec. \ref{sec_related} reviews related work on multilayer graph clustering.
Sec. \ref{sec_RIM_ML} introduces the multilayer signal plus noise model for multilayer graphs, and presents the mathematical formulation of multilayer SGC via convex layer aggregation. Sec. \ref{sec_MIMOSA_THM} provides  performance analysis of the proposed multilayer SGC under a multilayer signal plus noise model. We specify a breakdown condition for the success of multilayer SGC, and establish a phase transition on the clustering accuracy of multilayer SGC under a block-wise identical noise model and a block-wise non-identical noise model, respectively.
Sec. \ref{sec_MIMOSA} describes the proposed MIMOSA approach for automated multilayer graph clustering. Sec. \ref{sec_num_ML} presents numerical results that verify the phase transition analysis. Sec. \ref{sec_MIMOSA_data} compares the performance of MIMOSA with two other automated multilayer graph clustering methods on 9 real-world multilayer graph datasets. Finally, Sec. \ref{sec_conclusion} concludes this paper. 

\section{Related Work}
\label{sec_related}
 Graph clustering, also known as community detection, on multilayer graphs aims to find a consensus cluster assignment on each node in the common node set shared by different layers.
Layer aggregation has been a principal method for processing and mining multilayer graphs \cite{cai2005community,tang2009uncoverning,wu2015discovering,tang2012community,de2015structural,Taylor16,kim2016differential}, as it transforms a multilayer graph into a single aggregated graph,  facilitating application of  data analysis techniques designed for single-layer graphs.
Extending the stochastic block model (SBM) for graph clustering in single-layer graphs \cite{Holland83}, a multilayer SBM has been proposed for graph clustering on multilayer graphs \cite{han2015consistent,paul2015community,barbillon2016stochastic,PhysRevX.6.011036,Taylor16,Stanley16}. Under the assumption of two equally-sized clusters, the authors in \cite{Taylor16} show that if layer aggregation is used and if each layer is an independent realization of a common SBM,  the inferential limit for cluster detectability decays at rate $O(L^{-\frac{1}{2}})$, where $L$ is the number of layers. In \cite{Stanley16}, a layer selection method based on a multilayer SBM is proposed to improve the performance of graph clustering by identifying a subset of coherent layers. However, the multilayer SBM assumes homogeneous connectivity structure for within-cluster and between-cluster edges in each layer, and it also assumes layer-wise independence.

In addition to inference approaches based on the multilayer SBM, other methods have been proposed for graph clustering on multilayer graphs, including information-theoretic approaches \cite{papalexakis2013more,PhysRevE.92.042806}, k-nearest neighbor method \cite{greene2013producing}, nonnegative matrix factorization \cite{ni2015flexible},  flow-based approach \cite{PhysRevX.5.011027}, linked matrix factorization \cite{tang2009clustering}, random walk \cite{kuncheva2015community}, tensor decomposition \cite{Domenico13tensor}, subspace methods \cite{dong2012clustering,dong2014clustering}, subgraph mining with edge labels \cite{boden2012mining}, and greedy multilayer modularity maximization \cite{mucha2010community}.
More details on multilayer graph models can be found in the recent survey papers for graph clustering on multilayer graphs \cite{kivela2014multilayer,kim2015community}.

It is worth mentioning that the methods proposed in many of the aforementioned publications require the knowledge of the number of clusters (model order) for graph clustering, especially for matrix decomposition-based methods \cite{tang2009clustering,Domenico13tensor,ni2015flexible,dong2012clustering,dong2014clustering} and multilayer SBM \cite{han2015consistent,paul2015community,barbillon2016stochastic,PhysRevX.6.011036,Taylor16,Stanley16}. However, in many practical cases the model order is not known. Although many model order selection methods have been proposed for  single-layer graphs \cite{zelnik2004self,blondel2008fast,Krzakala2013,CPY16AMOS}, little has been developed for  multilayer graphs. Moreover, many layer aggregation methods assign uniform weights over layers such that the aggregated graph is insensitive to the quality of clusters in each layer \cite{tang2009uncoverning,tang2012community,Taylor16}.
This paper studies the sensitivity of the clustering accuracy to layer weights under a multilayer signal plus noise model. We then propose a model order selection algorithm featuring layer weight adaptation that automatically finds the minimal model order that meets statistical clustering reliability guarantees.

\section{Multilayer Graph Model and Spectral Graph Clustering via Convex Layer Aggregation}
\label{sec_RIM_ML}

\subsection{Multilayer graph model}
\label{subsec_Multi_RIM}
Throughout this paper, we consider a multilayer graph model consisting of $L$ layers representing different relationships among a common node set $\cV$ of $n$ nodes. The graph in the $\ell$-th layer is an undirected graph with nonnegative edge wights, which is denoted by $G_\ell=(\cV,\cE_{\ell})$, where $\cE_{\ell}$ is the set of weighted edges in the $\ell$-th layer. The $n \times n$ binary symmetric adjacency matrix $\bAl$ is used to represent the connectivity structure of $G_\ell$. The entry $[\bAl]_{uv}=1$ if nodes $u$ and $v$ are connected in the $\ell$-th layer, and $[\bAl]_{uv}=0$ otherwise. Similarly, the $n \times n$ nonnegative symmetric weight matrix $\bWl$ is used to represent the edge weights in $G_\ell$, where $\bWl$ and $\bAl$ have the same zero structure.

We assume each layer in the multilayer graph is a (possibly correlated) representation of a common set of disjoint $K$ clusters that partitions the node set $\cV$, where the $k$-th cluster has cluster size $n_k$ such that $\sum_{k=1}^K n_k=n$, and 
$\nmin=\min_{k \in \{1,\ldots,K\}} n_k$ and $\nmax=\max_{k \in \{1,\ldots,K\}} n_k$ denote the smallest and largest cluster size, respectively.
Specifically, the adjacency matrix $\bAl$ of $G_\ell$ in the $\ell$-th layer can be represented as 
\begin{align}                                                              \label{eqn_network_model_multilayer}
\bAl= \begin{bmatrix}
\bAl_1          & \bCl_{12} & \bCl_{13} & \cdots & \bCl_{1K}           \\
\bCl_{21}       & \bAl_2    & \bCl_{23} & \cdots & \bCl_{2K} \\
\vdots         & \vdots   & \ddots   & \vdots  & \vdots  \\
\vdots         & \vdots   & \vdots   & \ddots  & \vdots  \\
\bCl_{K1}       & \bCl_{K2} & \cdots   & \cdots  & \bAl_{K}
\end{bmatrix},
\end{align}
where $\bAl_k$ is an $n_k \times n_k$ binary symmetric matrix denoting the adjacency matrix of within-cluster edges of the $k$-th cluster in the $\ell$-th layer, and $\bCl_{ij}$  is an $n_i \times n_j$ binary rectangular matrix denoting the adjacency matrix of between-cluster edges of clusters $i$ and $j$  in the $\ell$-th layer, $1 \leq i,j \leq K$, $i \neq j$, 
and $\bCl_{i j}={\bCl_{ji}}^T$.

Similarly, the edge weight matrix $\bWl$ of the $\ell$-th layer  can be represented as 

\begin{align}                                                              \label{eqn_network_model_multilayer_weight}
\bWl= \begin{bmatrix}
\bWl_1          & \bFl_{12} & \bFl_{13} & \cdots & \bFl_{1K}           \\
\bFl_{21}       & \bWl_2    & \bFl_{23} & \cdots & \bFl_{2K} \\
\vdots         & \vdots   & \ddots   & \vdots  & \vdots  \\
\vdots         & \vdots   & \vdots   & \ddots  & \vdots  \\
\bFl_{K1}       & \bFl_{K2} & \cdots   & \cdots  & \bWl_{K}
\end{bmatrix},
\end{align}
where  $\bWl_k$ is an $n_k \times n_k$ nonnegative symmetric matrix denoting the edge weights of within-cluster edges of the $k$-th cluster in the $\ell$-th layer, and $\bFl_{ij}$  is an $n_i \times n_j$ nonnegative rectangular matrix denoting the edge weights  of between-cluster edges of clusters $i$ and $j$  in the $\ell$-th layer, $1 \leq i,j \leq K$, $i \neq j$, and $\bFl_{ij}={\bFl_{ji}}^T$.

\subsection{Multilayer signal plus noise model}
\label{subsec_ML_signal_noise}

Using the cluster-wise block representations of the adjacency and edge weight matrices for the multilayer graph model described in (\ref{eqn_network_model_multilayer}) and  (\ref{eqn_network_model_multilayer_weight}), we propose a signal-plus-noise model for $\bAl$ and $\bWl$ to analyze the effect of convex layer aggregation on graph clustering. 
Specifically, for each layer we assume the connectivity structure and edge weight distributions follow the random interconnection model (RIM)  \cite{CPY16AMOS}. 
In RIM, the signal of the $k$-th cluster in the $\ell$-th layer is the connectivity structure in terms of eigenspectrum and weights of the within-cluster edges represented by the matrices $\bAl_k$ and $\bWl_k$, respectively. The RIM imposes no distributional assumption on the within-cluster edges. The noise between clusters $i$ and $j$ in the $\ell$-th layer is caused by random between-cluster edges, which are represented by the matrices $\bCl_{ij}$ and $\bFl_{ij}$, respectively.

Throughout this paper, we assume the connectivity of a between-cluster edge (i.e., the noise) in each layer is independently drawn from a 
layer-wise and block-wise independent Bernoulli distribution. Specifically, each entry in $\bCl_{ij}$ representing the existence of an edge between clusters $i$ and $j$ in the $\ell$-th layer 
is an independent realization of a Bernoulli random variable with edge connection probability $\pijl \in [0,1]$ that is layer-wise and block-wise independent. In addition, given the existence of an edge $(u,v)$ between clusters $i$ and $j$ in the $\ell$-th layer,
the entry $[\bFl_{ij}]_{uv}$ representing the corresponding edge weight is independently drawn from a nonnegative distribution with mean $\Wbarijl$ and bounded fourth moment that is  layer-wise and block-wise independent. The assumption of  bounded fourth moment is required for the phase transition analysis established in Sec. \ref{sec_MIMOSA_THM}.

For the $\ell$-th layer, the noise accounting for the between-cluster edges is said to be \textit{block-wise identical} if the noise parameters $\pijl=\pl$ and $\Wbarijl=\Wbarl$ for every cluster pair $i$ and $j$, $i \neq j$. Otherwise it is said to be \textit{block-wise non-identical}. The effect of these two noise models on multilayer spectral graph clustering will be studied in Sec. \ref{sec_MIMOSA_THM}.

\subsection{Multilayer spectral graph clustering via convex layer aggregation}

\vspace{-1.2mm}

Let $\bw=[w_1,\ldots,w_L]^T \in \cW_{L}$ be an $L \times 1$ column vector representing the layer weight vector for convex layer aggregation, where $\cW_{L}=\{\bw: w_\ell \geq 0,~\sum_{\ell=1}^{L} w_\ell=1\}$ is the set of feasible layer weight vectors. The single-layer graph obtained via convex layer aggregation with layer weight vector $\bw$ is denoted by $G^{\bw}$. The (weighted) adjacency matrix and the edge weight matrix of $G^{\bw}$ are denoted by $\bAw$ and $\bWw$, respectively, where  $\bAw=\sum_{\ell=1}^{L} w_\ell \bAl$ and $\bWw=\sum_{\ell=1}^{L} w_\ell \bWl$. The graph Laplacian matrix $\bLw$ of $G^\bw$ is defined as  $\bLw=\bSw-\bWw=\sum_{\ell=1}^{L} w_\ell \bLl$, where $\bSw=\diag(\bsw)$ is a diagonal matrix, $\bsw=\bWw \bone_n$ is the vector of nodal strength of $G^\bw$,  $\bone_n$ is the $n \times 1$ column vector of ones, and $\bLl$ is the graph Laplacian matrix of $G_\ell$. Similarly, the graph Laplacian matrix $\bL_k^\bw$ accounting for the within-cluster edges of the $k$-th cluster in $G^\bw$ is defined as 
$\bL_k^\bw=\bSw_k-\bWw_k=\sum_{\ell=1}^{L} w_\ell \bLl_k$, where  $\bWw_k=\sum_{\ell=1}^{L} w_\ell \bWl_k$, $\bSw_k=\diag(\bWw_k\bone_{n_k})$, and $\bLl_k=\bSl_k-\bWl_k$. The $i$-th smallest eigenvalue of $\bLw$ is denoted by $\lambda_i(\bLw)$. Based on the definition of $\bLw$, the smallest eigenvalue $\lambda_1(\bLw)$ of $\bLw$ is 0, since $\bLw \bone_n=\bzero_n$, where $\bzero_n$ is the $n \times 1$ column vector of zeros.

Spectral graph clustering (SGC) \cite{Luxburg07} partitions the nodes in $G^\bw$ 
into $K$ ($K \geq 2$) clusters based on the $K$ eigenvectors associated with the $K$ smallest eigenvalues of $\bLw$. Specifically, SGC first transforms each node in $G^\bw$ to a $K$-dimensional vector in the subspace spanned by these eigenvectors, and then implements K-means clustering \cite{hartigan1979algorithm} on these vectors to group the nodes in $G^\bw$ into $K$ clusters. For analysis purposes, throughout this paper we assume $G^\bw$ is a connected graph. If $G^\bw$ is disconnected,  SGC can be applied to each connected component in $G^\bw$. Moreover, if  $G^\bw$ is connected, $\lambda_i(\bLw)>0$ for all $i \geq 2$. That is, the second to the $n$-th smallest eigenvalue of $\bLw$ are all positive \cite{Fiedler73}. In addition, the eigenvector associated with the smallest eigenvalue $\lambda_1(\bLw)$ provides no information about graph clustering since it is proportional to a constant vector, the vector of ones  $\bone_n$.

Let $\bY \in \mathbb{R}^{n \times (K-1)}$ denote the eigenvector matrix where its $k$-th column is the $(k+1)$-th eigenvector associated with $\lambda_{k+1}(\bLw)$, $1 \leq k \leq K-1$. By the Courant-Fischer theorem \cite{jennings1992matrix}, $\bY$ is  the solution of the minimization problem 
\begin{align}
\label{eqn_spectral_clustering_multi_formulation_ML}
&\SK(\bLw)=\min_{\bX \in \mathbb{R}^{n \times (K-1)}} \trace(\bX^T \bLw \bX), \nonumber \\
	&\text{subjec~to}~\bX^T \bX= \bI_{K-1},~\bX^T \bone_n=\bzero_{K-1}, 
\end{align}
where the optimal value $\SK(\bLw)=\trace(\bY^T \bLw \bY)$ in (\ref{eqn_spectral_clustering_multi_formulation_ML}) is the partial eigenvalue sum $\SK(\bLw)=\sum_{k=2}^{K} \lambda_k(\bLw)$, $\bI_{K-1}$ is the $(K-1) \times (K-1)$ identity matrix, and the constraints in  (\ref{eqn_spectral_clustering_multi_formulation_ML}) impose  orthonormality and centrality on the eigenvectors. In summary, multilayer SGC via convex layer aggregation works by computing the eigenvector matrix $\bY$ from $\bLw$ of $G^\bw$, and implementing K-means clustering on the rows of $\bY$ to group the nodes into $K$ clusters.

\section{Performance Analysis of Multilayer Spectral Graph Clustering via Convex Layer Aggregation}
\label{sec_MIMOSA_THM}

In this section, we establish three theorems on the performance of multilayer spectral graph clustering (SGC) via convex layer aggregation, which generalizes the phase transition analysis established in \cite{CPY16AMOS} for single-layer graphs. The novelty of the analysis presented in this section is the incorporation of the effect of layer weights into multilayer SGC. One obtains the results in \cite{CPY16AMOS} as a special case of the analysis presented in this section when there is only one layer (i.e., $L=1$) and therefore the layer weight vector $\bw$ reduces to a unit scalar. 
	To assist comparison, in this section we use a similar, but abbreviated, presentation structure as in \cite{CPY16AMOS} for our phase transition analysis\footnote{In Sec. \ref{sec_MIMOSA_THM}, there are a number of limit theorems stated about the behavior of random matrices and vectors whose dimensions go to infinity as the sizes $\{n_k\}_{k=1}^K$ of the clusters go to infinity while their relative sizes $n_k/n_{k^\prime}$ are held constant. For simplicity and convenience, the limit theorems are often stated in terms of the finite, but arbitrarily large, dimensions $n_k$, $k=1,2,\ldots, K$. For any two matrices $\bX$ and $\widetilde{\bX}$ of the same dimension, The notation $\bX \ra \widetilde{\bX}$ means convergence in the spectral norm \cite{Tropp_matrix_concentrate}. The notation $\bX \asconv \widetilde{\bX}$  means $\bX \ra \widetilde{\bX}$ almost surely.}. The proofs are given in the supplementary file. 
  The analysis provides a theoretical framework for multilayer SGC and allows us to evaluate the quality of clustering results in terms of a signal-to-noise (SNR) ratio that falls out of the established theorems. This SNR is then used for determining the number of clusters and selecting layer weights in the algorithm proposed in Sec. \ref{sec_MIMOSA}.  

The first theorem (Theorem \ref{thm_impossible_ML}) specifies the interplay between layer weights and the success of multilayer SGC by establishing a condition under which multilayer SGC  fails to correctly identify clusters under the multilayer signal plus noise model in Sec. \ref{subsec_ML_signal_noise}  due to inconsistent rows in the eigenvector matrix $\bY$. The condition is called a ``breakdown condition'' and can be used as a test for identifiability of a given cluster configuration in the multilayer SGC problem.

The second theorem (Theorem \ref{thm_spec_ML}) establishes phase transitions on the clustering performance of multilayer SGC under the block-wise identical noise model for a given layer weight vector $\bw$. Under the block-wise identical noise model, define $\tl=\pl \cdot \Wbarl$ to be the noise level of the $\ell$-th layer and 
let $t^\bw=\sum_{\ell=1}^{L} w_\ell \cdot \tl$ be the aggregated noise level via convex layer aggregation.
We show that for each $\bw \in \cW_{L}$ there exists a critical value $\twstar$ of $t^\bw$ such that if $t^\bw < \twstar$, multi-layer SGC can correctly identify the clusters, and if  $t^\bw > \twstar$, reliable multi-layer SGC is not possible.

The third theorem (Theorem \ref{thm_principal_angle_ML}) extends the phase transition analysis of the block-wise identical noise model to the block-wise non-identical noise model. Under the block-wise non-identical noise model, define $\tlmax=\max_{i,j, i \neq j} \pijl \cdot \Wbarijl$ as the maximum noise level of the $\ell$-th layer and let $\twmax=\sum_{\ell=1}^{L} w_\ell \cdot \tlmax$. Then for each $\bw \in \cW_{L}$ we show that reliable clustering results can be guaranteed provided that $\twmax < \twstar$, where $\twstar$ is the critical value for phase transition under the block-wise identical noise model.

\subsection{Breakdown condition for multilayer SGC via convex layer aggregation}
 
Under the multilayer signal plus noise model in Sec. \ref{subsec_ML_signal_noise}, let $\tijl=\Wbarijl \cdot \pijl$ be the noise level between clusters $i$ and $j$ in the $\ell$-th layer, $1 \leq i,j \leq K$, $i \neq j$, and $1 \leq \ell \leq L$.
The following theorem establishes 
 a general breakdown condition under which multilayer SGC fails to correctly identify the clusters. 
 
\begin{thm}[general breakdown condition]~\\
	\label{thm_impossible_ML}
	Let $\bWt^\bw$ be the $(K-1) \times (K-1)$ matrix with $(i,j)$-th entry
	\begin{align}
	[\bWt^\bw]_{ij}=\left\{
	\begin{array}{ll}
    \sum_{\ell=1}^{L} w_\ell \Lb \lb  n_i+n_K \rb t_{iK}^{(\ell)}+\sum_{z=1,z \neq i}^{K-1} n_z t_{iz}^{(\ell)}  \Rb,& \\  \text{~if~} i=j; & \nonumber \\
    \sum_{\ell=1}^{L} w_\ell	n_i\cdot \lb t^{(\ell)}_{iK}-t^{(\ell)}_{ij} \rb, \text{~if~} i \neq j.
	\end{array}
	\right.	
	\end{align}
	 The following holds almost surely as $n_k \ra \infty$~$\forall~k$ and $\frac{\nmin}{\nmax} \ra c >0$. If for any layer weight vector $\bw \in \cW_L$, $\lambda_i \lb \frac{\bWt^\bw}{n} \rb \neq \lambda_j \lb \frac{\bLw}{n} \rb$ for all $i = 1,2,\ldots,K-1$ and $j =2,3,\ldots,K$, then multilayer SGC cannot be successful.		
\end{thm}

Theorem \ref{thm_impossible_ML} specifies the interplay between the layer weight vector $\bw$ and the accuracy of multilayer SGC. Different from the case of single-layer graphs (i.e., $L=1$ and hence $\bw=1$) such that the layer weight has no effect on the performance of SGC, Theorem \ref{thm_impossible_ML} states that 
multilayer SGC cannot be successful if every possible layer weight vector $\bw \in \cW_L$ leads to distinct $K-1$ smallest nonzero eigenvalues of the matrices  $\frac{\bWt^\bw}{n}$ and $\frac{\bLw}{n}$. It also suggests that the selection of layer weight vector affects the performance of multilayer SGC.

\subsection{Phase transitions in multilayer SGC under block-wise identical noise}

Under the multilayer signal plus noise model in Sec. \ref{subsec_ML_signal_noise}, if we further assume the between-cluster edges in each layer follow a block-wise identical distribution, then the noise level in the $\ell$-th layer can be characterized by the parameter $\tl=\pl \cdot \Wbarl$, where $\pl \in [0,1]$ is the edge connection parameter and $\Wbarl>0$ is the mean of the between-cluster edge weights in the $\ell$-th layer.
Under the block-wise identical noise model and given a layer weight vector $\bw \in \cW_{L}$, let $\tw=\sum_{\ell=1}^{L} w_\ell \tl$ denote the aggregated noise level of the graph $G^\bw$.
Theorem \ref{thm_spec_ML} below establishes phase transitions in the eigendecomposition of the graph Laplacian matrix $\bLw$ of the graph $G^\bw$. We show that there exists a critical value $\twstar$ such that the $K$ smallest eigenpairs of  $\bLw$ that are used for multilayer SGC have different characteristics when $\tw < \twstar$ and $\tw > \twstar$.
 In particular, we show that the solution to the minimization problem in (\ref{eqn_spectral_clustering_multi_formulation_ML}), the eigenvector matrix $\bY=[\bY_1^T,\bY_2^T,\ldots,\bY_K^T]^T \in \bbR^{n \times (K-1)}$, where its rows $\bY_k \in \bbR^{n_k \times (K-1)}$ index the nodes in cluster $k$, has cluster-wise separability when  $\tw < \twstar$. This means that, under this condition,  
the rows of each $\bY_k$ are identical with columns that are cluster-wise distinct.  On the other hand, when  $\tw >  \twstar$ the row-wise average of each matrix $\bY_k$ is a zero vector and hence the clusters cannot be perfectly separated by inspecting the eigenvector matrix $\bY$.

\begin{thm}[block-wise identical noise]~\\
	\label{thm_spec_ML}
	Let $\bY=[\bY_1^T,\bY_2^T,\ldots,\bY_K^T]^T$ be the solution of the minimization problem in (\ref{eqn_spectral_clustering_multi_formulation_ML}) and
	let $\cwstar=\min_{k \in \{1,2,\ldots,K\}} \LB \frac{\SK(\bLw_k)}{n}  \RB$, where $\bLw_k=\sum_{\ell=1}^{L} w_\ell \bLl_k$.
	Given a layer weight vector $\bw \in \cW_L$, under the block-wise identical noise model with aggregated noise level $\tw=\sum_{\ell=1}^{L} w_\ell \tl=\sum_{\ell=1}^{L} w_\ell  \pl \Wbarl$, 
	there exists a critical value $\twstar$ such that the following holds almost surely as $n_k \ra \infty$~$\forall~k$ and $\frac{\nmin}{\nmax} \ra c >0$: \\	
	\textnormal{(a)}~$ \left\{
	\begin{array}{ll}
	\textnormal{If~} \tw \leq \twstar,~ \frac{\SK(\bLw)}{n} = (K-1)\tw; \\
	\textnormal{If~} \tw > \twstar,~ \cwstar + (K-1) \lb 1-\frac{\nmax}{n} \rb \tw  \leq  	\frac{\SK(\bLw)}{n} \\ 
	~~~~~~~~~~~~~~~~\leq \cwstar + (K-1) \lb 1-\frac{\nmin}{n} \rb \tw.  \\
	\end{array}
	\right.$ \\
	In particular, if $\tw > \twstar \textnormal{~and~} c=1,~ \frac{\SK(\bLw)}{n} = \cwstar +\frac{(K-1)^2}{K} \tw.$  \\	
	\textnormal{(b)}~$\left\{	
	\begin{array}{ll}
	\textnormal{If~} \tw < \twstar,~\bY_k = \bone_{n_k} \bone_{K-1}^T \bV_k\\
	~~~~~~~~~~~~~~~~~~~~=\Lb v^k_1 \bone_{n_k},v^k_2 \bone_{n_k},\ldots,v^k_{K-1} \bone_{n_k} \Rb,~ \\ 
	~~~~~~~~~~~~~~~~~~~~~\forall~k \in \{1,2,\ldots,K\}; \\
	\textnormal{If~} \tw > \twstar,~
	\bY_k^T\bone_{n_k} = \bzero_{K-1},~\forall~k \in \{1,2,\ldots,K\}; \\
	\textnormal{If~} \tw = \twstar,~\forall~k \in \{1,2,\ldots,K\},~ \bY_k =\bone_{n_k} \bone_{K-1}^T \bV_k \\
	~~~~~~~~~~~~~~~\textnormal{~or~} \bY_k^T\bone_{n_k} = \bzero_{K-1},
	\end{array}
	\right.$ \\
	where $\bV_k=\diag(v^k_1, v^k_2,\ldots, v^k_{K-1}) \in \mathbb{R}^{(K-1) \times (K-1)}$ is a diagonal matrix. \\
	In particular, when $\tw < \twstar$, $\bY$ has the following properties:\\
	\textnormal{(b-1)} The columns of $\bY_k$ are constant vectors. \\
	\textnormal{(b-2)} Each column of $\bY$ has at least two nonzero cluster-wise constant components, and these constants have alternating signs such that their weighted sum equals $0$ (i.e., $\sum_{k} n_k v^k_j = 0,~\forall~j \in\{1,2,\ldots,K-1\}$). \\
	\textnormal{(b-3)} No two columns of $\bY$ have the same sign on the cluster-wise nonzero components.	 \\	
	Finally, $\twstar$ satisfies: \\	
	\textnormal{(c)}~$\tLBw \leq \twstar \leq \tUBw$, where \\
	$\left\{
	\begin{array}{ll}
	\tLBw = \frac{\min_{k \in \{1,2,\ldots,K\}} \SK(\bLw_k)}{(K-1)\nmax}; \\
	\tUBw  = \frac{\min_{k \in \{1,2,\ldots,K\}} \SK(\bLw_k)}{(K-1)\nmin}.
	\end{array}
	\right.$ \\
	In particular, 	$\tLBw=\tUBw$ when $c=1$.	
\end{thm}	

Theorem \ref{thm_spec_ML} (a) establishes a phase transition in the increase
 of the normalized partial eigenvalue sum $\frac{\SK(\bLw)}{n}$ with respect to the aggregated noise level $\tw$. When $\tw \leq \twstar$ the quantity $\frac{\SK(\bLw)}{n}$ is exactly  $(K-1)\tw$. When $\tw> \twstar$ the slope in $\tw$ of $\frac{\SK(\bL)}{n}$  changes and the intercept $c^*=\min_{k \in \{1,2,\ldots,K\}} \LB \frac{\SK(\bLw_k)}{n}  \RB=\min_{k \in \{1,2,\ldots,K\}} \LB \frac{\sum_{\ell=1}^{L} w_\ell \SK(\bLl_k)}{n}  \RB$ depends on the cluster having the smallest aggregated partial eigenvalue sum given a layer weight vector $\bw$. In particular, when all clusters have the same size (i.e., $\nmax=\nmin=\frac{n}{K}$) so that $c=1$, $\frac{\SK(\bL)}{n}$ undergoes a slope change from $K-1$ to $\frac{(K-1)^2}{K}$ at the critical value $\tw=\twstar$.  The visual illustration of Theorem \ref{thm_spec_ML} (a) is displayed in Fig. \ref{Fig_slope} of the supplementary material.

Theorem \ref{thm_spec_ML} (b) establishes a phase transition in cluster-wise separability of the eigenvector matrix $\bY$ for multilayer SGC. When $\tw < \twstar$, the conditions (b-1) to (b-3) imply that the rows of the cluster-wise components $\{\bY_k\}_{k=1}^K$ are coherent, and hence the row vectors in $\bY$ possess cluster-wise separability. On the other hand, when $\tw > \twstar$,
the row sum of each $\bY_k$ is a zero vector, making $\mathbf Y_k$ incoherent. This means that the entries of each column in $\bY_k$ have alternating signs  and the centroid of the row vectors in $\bY_k$ is centered at the origin.
Therefore, K-means clustering on the rows of $\bY$ yields incorrect clusters.

Theorem  \ref{thm_spec_ML} (c) establishes upper and lower bounds on the critical threshold value $\twstar$ of the aggregated noise level $\tw$ given a layer weight vector $\bw$. These bounds are determined by the cluster having the smallest aggregated partial eigenvalue sum $\SK(\bLw_k)=\sum_{\ell=1}^{L} w_\ell \SK(\bLl_k)$, the number of clusters $K$, and the largest and smallest cluster size ($\nmax$ and $\nmin$). When all cluster sizes are identical (i.e., $c=1$), these bounds become tight (i.e., $\tLBw=\tUBw$). Moreover, by the nonnegativity of the layer weights we can obtain a universal lower bound on $\tLBw$ for any $\bw \in \cW_L$, which is 
\begin{align}
\label{eqn_LB_tLB}
	\tLBw &= \frac{\min_{k \in \{1,2,\ldots,K\}} \SK(\bLw_k)}{(K-1)\nmax} \nonumber \\
	&\geq \frac{\min_{k \in \{1,2,\ldots,K\}}  \min_{\ell \in \{1,2,\ldots,L\}}\SK(\bLl_k)}{(K-1)\nmax}. 
\end{align}
Since $\SK(\bLl_k)$ is a measure of connectivity for cluster $k$ in the $\ell$-th layer, the lower bound of $\tLBw$ in (\ref{eqn_LB_tLB}) implies that the performance of multilayer SGC is indeed affected by the least connected cluster among all $K$ clusters and across $L$ layers. Specifically, if the graph in each layer is unweighted and $K=2$, then $\SK(\bLl_k)=\lambda_2(\bLl_k)$ reduces to the algebraic connectivity  of cluster $k$ in the $\ell$-th layer. Similarly, a universal upper bound on $\tUBw$ for any $\bw \in \cW_L$ is 
\begin{align}
\label{eqn_LB_tUB}
	\tUBw  \leq \frac{\min_{k \in \{1,2,\ldots,K\}} \max_{\ell \in \{1,2,\ldots,L\}} \SK(\bLl_k)}{(K-1)\nmin}.
\end{align}

\subsection{Phase transitions in multilayer SGC under block-wise non-identical noise}
Under the block-wise non-identical noise model, the noise level of between-cluster edges between clusters $i$ and $j$ in the $\ell$-th layer is characterized by the parameter $\tijl=\pijl \cdot \Wbarijl$, $1 \leq i,j \leq K$, $i \neq j$, and $1 \leq \ell \leq L$. Let $\tlmax=\max_{1\leq i,j \leq K,~i \neq j} \tijl$ be the maximum noise level in the $\ell$-th layer and let $\twmax=\sum_{\ell=1}^{L} w_\ell \tlmax$ denote the aggregated maximum noise level given a layer weight vector $\bw \in \cW_L$.

Let $\bY \in \mathbb{R}^{n \times (K-1)}$ be the  eigenvector matrix of $\bLw$ under the
block-wise non-identical noise model, and let  $\btY \in \mathbb{R}^{n \times (K-1)}$ be the eigenvector matrix of the graph Laplacian $\btLw$ of  another graph generated by the block-wise identical noise model with aggregated noise level $\tw$, which is independent of $\bL$.   Theorem \ref{thm_principal_angle_ML} below specifies the distance between the subspaces spanned by the columns of $\bY$ and $\btY$ by inspecting their principal angles \cite{Luxburg07}. Specifically, since $\bY$ and $\btY$ both have orthonormal columns, the vector $\ba$ of $K-1$ principal angles between their column spaces is $\ba=[\cos^{-1}\sigma_1(\bY^T \btY),\ldots,\cos^{-1}\sigma_{K-1}(\bY^T \btY)]^T$, where $\sigma_k(\bM)$ is the $k$-th largest singular value of a real rectangular matrix $\bM$.
Let $\mathbf{\Theta}(\bY,\btY)=\diag(\ba)$, and let $\sin\mathbf{\Theta}(\bY,\btY)$ be defined entrywise. 
 When $\tw < \twstar$, Theorem \ref{thm_principal_angle_ML} provides an upper bound on the Frobenius norm of $\sin\mathbf{\Theta}(\bY,\btY)$, which is denoted by $\| \sin\mathbf{\Theta}(\bY,\btY) \|_F$. Moreover, if $\twmax< \twstar$, where $\twstar$ is the critical threshold value for the block-wise identical noise model as specified in Theorem \ref{thm_spec_ML}, then $\| \sin\mathbf{\Theta}(\bY,\btY) \|_F$ can be further bounded.

\begin{thm}[block-wise non-identical noise]~\\
	\label{thm_principal_angle_ML}
	Under the multilayer signal plus noise model in Sec. \ref{subsec_ML_signal_noise} with maximum noise level $\{\tlmax\}_{\ell=1}^L$ for each layer, given a layer weight vector $\bw \in \cW_L$, let $\twstar$  be 
   be the critical threshold value for the block-wise identical noise model specified by Theorem \ref{thm_spec_ML}, and define $\delta_{\tw,n}=\min\{\tw,|\lambda_{K+1}(\frac{\bLw}{n})-\tw|\}$. 
	For a fixed $\tw$, if $\tw < \twstar$ and $\delta_{\tw,n} \ra \delta_{\tw} > 0$ as $n_k \ra \infty$~$\forall~k$,
	the following statement holds almost surely as
	$n_k \ra \infty$~$\forall~k$ and $\frac{\nmin}{\nmax} \ra c >0$:
	\begin{align}
	\label{eqn_principal_angle_bound_ML}
	\|\sin\mathbf{\Theta}(\bY,\btY)\|_F \leq \frac{\| \bLw - \btLw \|_F}{n \delta_{\tw}}.
	\end{align}
	Furthermore, let $\twmax=\sum_{\ell=1}^L w_\ell \tlmax$. If $\twmax < \twstar$,
	\begin{align}
	\label{eqn_principal_angle_bound_2_ML}
	\|\sin\mathbf{\Theta}(\bY,\btY)\|_F \leq \min_{\tw \leq \twmax} \frac{\| \bLw - \btLw \|_F}{n \delta_{\tw}}.
	\end{align}
\end{thm}

 Theorem \ref{thm_principal_angle_ML} shows that the subspace distance  $\|\sin\mathbf{\Theta}(\bY,\btY)\|_F$ is upper bounded by (\ref{eqn_principal_angle_bound_ML}), where $\btY$ is the eigenvector matrix of $\btLw$ under the block-wise identical noise model when its aggregated noise level $\tw < \twstar$. Furthermore, if the aggregated maximum noise level $\twmax < \twstar$, then a tight upper bound on $\|\sin\mathbf{\Theta}(\bY,\btY)\|_F$ can be obtained by (\ref{eqn_principal_angle_bound_2_ML}).
Therefore, using the phase transition results of the cluster-wise separability in $\btY$  as established in Theorem \ref{thm_spec_ML} (b), when $\twmax < \twstar$,
cluster-wise separability in $\bY$ can be expected provided that $\|\sin\mathbf{\Theta}(\bY,\btY)\|_F $ is small.

\section{MIMOSA: Multilayer Iterative Model Order Selection Algorithm}
\label{sec_MIMOSA} 
 
The phase transition analysis established in Sec. \ref{sec_MIMOSA_THM} shows that under the multilayer signal plus noise model in Sec. \ref{subsec_ML_signal_noise}, the performance of multilayer spectral graph clustering (SGC) via convex layer aggregation can be separated into two regimes: a reliable regime where high clustering accuracy is guaranteed, and an unreliable regime where high clustering accuracy is impossible. We have specified the critical threshold value of the aggregated noise level that separates these two regimes, and have shown that the assigned layer weight vector $\bw$ for convex layer aggregation indeed affects the accuracy of multilayer SGC. 

In this section, we use the established phase transition criterion to propose a multilayer SGC algorithm, which we call multilayer iterative model order selection algorithm (MIMOSA). MIMOSA is a multilayer SGC algorithm that features automated model order selection for determining the number of clusters ($K$) and the layer weight vector $\bw$. It works by
incrementally partitioning the aggregated graph $G^\bw$ into $K$ clusters, adjusting the layer weight vector, and finding the minimal number of clusters such that the output clusters are estimated to be in the reliable regime. The flow diagram of MIMOSA is displayed in Fig. \ref{Fig_MIMOSA_flow}, and the complete algorithm is summarized in Algorithm \ref{algo_MIMOSA}. Since part of MIMOSA uses the same statistical testing methods developed for single-layer graphs in \cite{CPY16AMOS}, the
details on the V-test and Wilk's test are omitted. The interested reader can refer to Sec. V of \cite{CPY16AMOS}.

\subsection{Input data}
The input data for MIMOSA is summarized as follows. (1) a multilayer graph $\{G_\ell\}_{\ell=1}^{L}$ of $L$ layers, where each layer $G_\ell$ is an undirected weighted graph. (2) an initial layer weight vector $\bwini \in \cW_L$. $\bwini$ can be specified according to domain knowledge, or it can be a uniform vector such that $w_\ell=\frac{1}{L}$~$\forall~\ell$. (3) a layer weight adaptation coefficient set $\cT=\{ \tau_z \}_{z=1}^{|\cT|}$. The coefficients in $\cT$ play a role in the process of layer weight adaptation in Sec. \ref{subsec_layer_weight_adap}. (4)  a p-value significance level $\eta$ that is used for 
the block-wise homogeneity test in Sec. \ref{subsec_block_hom_test}.
(5) confidence interval parameters $\{\alpha_\ell\}_{\ell=1}^L$ of each layer under the block-wise identical noise model  for clustering reliability evaluation in Sec. \ref{subsec_identical_test}.
(6) confidence interval parameters $\{\alpha^\prime_\ell\}_{\ell=1}^L$ of each layer under the block-wise non-identical noise model   for clustering reliability evaluation in Sec. \ref{subsec_nonidentical_test}.

\begin{figure}[t!]		
	\centering
	\includegraphics[scale=0.45]{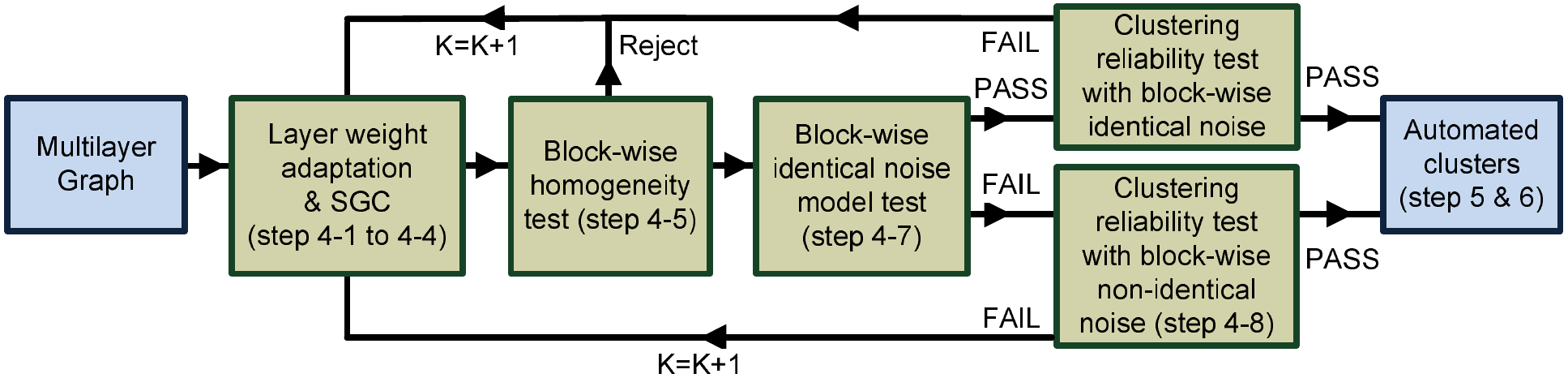}
	\caption{Flow diagram of the proposed multilayer iterative model order selection algorithm (MIMOSA) for multilayer spectral graph clustering  (SGC).}
	\label{Fig_MIMOSA_flow}     
\vspace{-4mm}
\end{figure}

\subsection{Layer weight adaptation}
\label{subsec_layer_weight_adap}
Given an initial layer weight vector $\bwini$ and the number of clusters $K$ in the iterative process (step 4) of MIMOSA, 
we propose to adjust the layer weight vector $\bw$ for convex layer aggregation by estimating the noise level $\{\htinil\}_{\ell=1}^L$ under the block-wise identical noise model in Sec. \ref{subsec_ML_signal_noise}. Specifically, given $K$ clusters $\{\cC_k^{\bwini}\}_{k=1}^K$  of size $\{\hn_k\}_{k=1}^K$ via multilayer SGC with $\bwini$, let $\{ \hbCijl \}$ and $\{ \hbFijl \}$  be the interconnection matrix and edge weight matrix of $\{\cC_k^{\bwini}\}_{k=1}^K$, respectively, for $1 \leq i,j \leq K$, $i \neq j$, and $1 \leq \ell \leq L$. Then the  noise level estimator under the block-wise identical noise model is 
\begin{align}
\label{eqn_MIMOSA_ini_est}
\htinil=\hpl \cdot \hWbarl,
\end{align}
for $\ell \in \{1,2,\ldots,L\}$, where $\hpl=\frac{\sum_{i=1}^K \sum_{j=i+1}^K \hatmijl}{\sum_{i=1}^K \sum_{j=i+1}^K {\hn_i \hn_j}}$  is the maximum likelihood estimator (MLE) of $\pl$, $\hatmijl=\bone_{\hn_i}^T \hbCijl \bone_{\hn_j}$ is the number of between-cluster edges of clusters $i$ and $j$ in the $\ell$-th layer, and $\hWbarl$ is the average of between-cluster edge weights in the $\ell$-th layer.

Since the estimates  $\{\htinil\}_{\ell=1}^L$ reflect the noise level in each layer, we propose to adjust the layer weight vector $\bw \in \cW_L$ 
with a nonnegative regularization parameter $\tau \in \cT$. The adjusted $\bw$ layer weight vector is inversely proportional to the estimated noise level, which is  defined as
\begin{align}
\label{eqn_layer_weight_adjust}
w_\ell \propto   \frac{\wini_\ell}{1+\tau \cdot \htinil},
\end{align}
for $\ell \in \{1,2,\ldots,L\}$. Note that if $\tau=0$, then $\bw$ reduces to $\bwini$. In addition, larger $\tau$ further penalizes the layers of high noise level by assigning less weight for convex layer aggregation. In addition, to enable the computation of the function $\min_{k\in\{1,2,\ldots,K\}} \SK(\sum_{\ell=1}^{L} w_\ell \cdot \bLl_k)$ for clustering reliability test in the following step of MIMOSA, the detected clusters $\{\cC_k^\bw\}_{k=1}^K$ are deemed unreliable if the size of any detected cluster is less than $K$.

\begin{algorithm*}[]
	\caption{Multilayer iterative model order selection algorithm (MIMOSA) for multilayer SGC}
	\label{algo_MIMOSA}
	\begin{algorithmic}
		\State \textbf{Input:} 
		\State (1) a multilayer graph $\{G_{\ell}\}_{\ell=1}^L$
		\State (2) an initial layer weight vector $\bwini \in \cW_L$
		\State (3) a layer weight adaptation coefficient set $\cT=\{ \tau_z \}_{z=1}^{|\cT|}$
		\State (4) a p-value significance level $\eta$
		\State (5) confidence interval parameters $\{\alpha_\ell\}_{\ell=1}^L$ under the block-wise identical noise model for each layer
		\State (6) confidence interval parameters $\{\alpha^\prime_\ell\}_{\ell=1}^L$ under the block-wise non-identical noise model for each layer
		\State \textbf{Output:} $K$ clusters $\{\cC_k\}_{k=1}^K$
		\State Initialization: $K=2$. Flag $=1$. $\cW_{\textnormal{reliable}}=\varnothing$.
		\While{Flag$=1$}
		\State 1. Compute $\bY \in \bbR^{n \times (K-1)}$ of $\bL^{\bwini}$ 
		\State 2. Obtain $K$ clusters $\{\cC_k^{\bwini}\}_{k=1}^K$ by implementing K-means algorithm on the rows of $\bY$
		\State 3. Estimate the noise level $\{\htinil\}_{\ell=1}^L$ from (\ref{eqn_MIMOSA_ini_est})
		\State 4. Layer weight adaptation and multilayer SGC reliability tests:
		\For{$z=1$ to $|\cT|$}
		\State 4-1.  Layer weight adaptation: $ w_\ell  \leftarrow \wini_\ell \cdot (1+\tau_z \cdot \htl)^{-1}$, $\forall~\ell \in \{1,2,\ldots,L\}$
		\State 4-2.  Layer weight normalization: $ w_\ell \leftarrow  \frac{w_\ell}{\sum_{\ell^\prime=1}^{L} w_{\ell^\prime}}$, $\forall~\ell \in \{1,2,\ldots,L\}$
		\State 4-3. Compute $\bY \in \bbR^{n \times (K-1)}$ of $\bL^{\bw}$
		\State 4-4. Obtain $K$ clusters $\{\cC_k^{\bw}\}_{k=1}^K$ by implementing K-means algorithm on the rows of $\bY$
		\State 4-5. \textbf{Block-wise homogeneity test:} calculate p-value($i,j,\ell$),~$\forall~i,j,\ell,$ $1 \leq i,j \leq K$, $i \neq j$, and $1 \leq \ell \leq L$
		\If{p-value($i,j,\ell$) $\leq \eta$ for some ($i,j,\ell$)}
		\State Go back to step 4-1 with $z = z+1$
		\EndIf	
		\State 4-6. Estimate the noise level $\{\htijl\}$ for all $i,j,\ell$ and estimate  $\htLBw$ from (\ref{eqn_estimate_htLBw})
		\State 4-7. \textbf{Block-wise identical noise test:} estimate the aggregated noise level $\htw=\sum_{\ell=1}^L w_\ell \cdot \htl$		
		\If{$\htl$ lies in the $100(1-\alpha_\ell) \%$ confidence interval~$\forall~\ell$}
		\If{$\htw < \htLBw$}				 
		\State Flag$=0$. $\cW_{\textnormal{reliable}}=\cW_{\textnormal{reliable}} \cup \{\bw\}$. 
		\EndIf    
		\ElsIf{$\htl$ does not lie in the $100(1-\alpha_\ell) \%$ confidence interval for some $\ell$}
		\State 4-8. \textbf{Block-wise non-identical noise test:} estimate the aggregated maximum noise level $\htwmax=\sum_{\ell=1}^{L} w_\ell \htmaxl$
		\If{$\prod_{i=1}^K \prod_{j=i+1}^K F_{ij} ( \frac{\htLBw}{\hWbarijl},\hpijl ) \geq 1-\alpha^\prime_{\ell}$~$\forall~\ell$}	
		\If{$\htwmax < \htLBw$}
		\State Flag$=0$. $\cW_{\textnormal{reliable}}=\cW_{\textnormal{reliable}} \cup \{\bw\}$.	
		\EndIf
		\EndIf
		\EndIf
		\State  Go back to step 4-1 with $z=z+1$
		\EndFor
		\If{Flag$=1$}
		\State Go back to step 1 with $K=K+1$
		\EndIf	
		\EndWhile
		\State 5. SNR criterion: select $\bw^*=\arg \max_{\bw \in \cW_{\textnormal{reliable}}} \frac{\htLBw}{\htw}$
		\State 6. Output final clustering result: $\{\cC_k\}_{k=1}^K \leftarrow \{\cC_k^{\bw^*}\}_{k=1}^K$ 							
	\end{algorithmic}
\end{algorithm*}

\subsection{Block-wise homogeneity test}
\label{subsec_block_hom_test}

Given $K$ clusters $\{\cC_k^{\bw}\}_{k=1}^K$ with respect to a layer weight vector $\bw$ in the iterative process (step 4) of MIMOSA, we implement a block-wise homogeneity test for each block $\hbCijl$ accounting for the interconnection matrix of clusters $i$ and $j$ in the $\ell$-th layer, in order to test the assumption of the block-wise homogeneity noise model as assumed in Sec. \ref{subsec_ML_signal_noise}, which is the cornerstone of the phase transition results established in Sec. \ref{sec_MIMOSA_THM}.

In particular, we use the V-test developed in Algorithm 1 of \cite{CPY16AMOS} to test the assumption of  block-wise homogeneity noise model. Given $x$ independent binomial random variables, the V-test tests that they are all identically distributed \cite{potthoff1966testing}. Here we apply the V-test to the row sums of $\hbCijl$. 
 The block-wise homogeneity test on $\hbCijl$ rejects the block-wise homogeneous hypothesis if its \text{p-value}$(i,j,\ell) \leq \eta$, where $\eta$ is the desired single comparison significance level.

	In step 4-5 of MIMOSA, the layer weight vector $\bw$ and the corresponding clusters $\{\cC_k^{\bw}\}_{k=1}^K$ are deemed unreliable if 
	there exists some $\hbCijl$ such that its p-value does not exceed the significance level.

\subsection{Clustering reliability test under the block-wise identical noise model}
\label{subsec_identical_test}
In the iterative process of step 4 in MIMOSA, if every interconnection matrix $\hbCijl$ passes the block-wise homogeneity test in Sec. \ref{subsec_block_hom_test}, the identified clusters $\{ \cC_k^\bw\}_{k=1}^K$ are then used to test the clustering reliability under the block-wise identical noise model in Sec. \ref{subsec_ML_signal_noise}. In particular, for each layer $\ell$, we first estimate the noise level parameter $\hpijl$ for every cluster pair $i$ and $j$ as $\hpijl=\frac{\hatmijl}{{\hn_i \hn_j}}$, where $\hpijl$ is an MLE of $\pijl$.
We then use the generalized log-likelihood ratio test (GLRT) developed in Sec. V-C. of \cite{CPY16AMOS} to specify an asymptotic  $100(1-\alpha_\ell) \%$ confidence interval for $\pl$ accounting for the block-wise identical noise level parameter for each layer. In particular, the GLRT is a test statistic of the null hypothesis \textit{all block-wise noises are independent and identical versus the alternative hypothesis \textit{all block-wise noises are independent but not identical}.}

If the estimated block-wise identical noise level parameter $\hpl=\frac{\sum_{i=1}^K \sum_{j=i+1}^K \hatmijl}{\sum_{i=1}^K \sum_{j=i+1}^K {\hn_i \hn_j}}$ is within the $100(1-\alpha_\ell) \%$ confidence interval 
for every $\ell$, then the clusters $\{ \cC_k^\bw\}_{k=1}^K$ satisfy the block-wise identical noise model, and therefore we can apply the phase transition results in Theorem \ref{thm_spec_ML} to evaluate the clustering reliability. In particular, we compare the estimated aggregated noise level $\htw$ with the estimated phase transition lower bound $\htLBw$ of $\tLBw$ in Theorem \ref{thm_spec_ML} (c), where $\htw=\sum_{\ell=1}^{L} w_\ell \htl=\sum_{\ell=1}^{L} w_\ell \cdot \hpl \cdot \hWbarl$, and 
\begin{align}
\label{eqn_estimate_htLBw}
\htLBw= \frac{\min_{k\in\{1,2,\ldots,K\}} \SK(\sum_{\ell=1}^{L} w_\ell \cdot \hbLl_k)}{(K-1) \cdot \widehat{n}_{\max}},
\end{align}
where
$\hbLl_k$ is the graph Laplacian matrix of within-cluster edges of cluster $\cC_k^{\bw}$ in the $\ell$-th layer, $\SK(\sum_{\ell=1}^{L} w_\ell \cdot \hbLl_k)=\sum_{z=2}^{K} \lambda_{z}(\sum_{\ell=1}^{L} w_\ell \cdot \hbLl_k)$, and $\widehat{n}_{\max}=\max_{k\in\{1,2,\ldots,K\}} \hn_k$. 
Therefore, using Theorem \ref{thm_spec_ML},  the clusters $\{ \cC_k^\bw\}_{k=1}^K$ are deemed reliable if $\htw < \htLBw$, since the eigenvector matrix $\bY$ used for multilayer SGC possesses cluster-wise separability. The lower bound in (\ref{eqn_estimate_htLBw}) also specifies the effect of cluster size on clustering reliability test. Ignoring the term in the numerator, a set of imbalanced clusters having larger $\widehat{n}_{\max}$ leads to smaller $\htLBw$ and hence implies a more difficult clustering problem.

	\begin{figure*}[t]
		\centering
		\begin{subfigure}[b]{0.25\linewidth}
			\includegraphics[width=\textwidth]{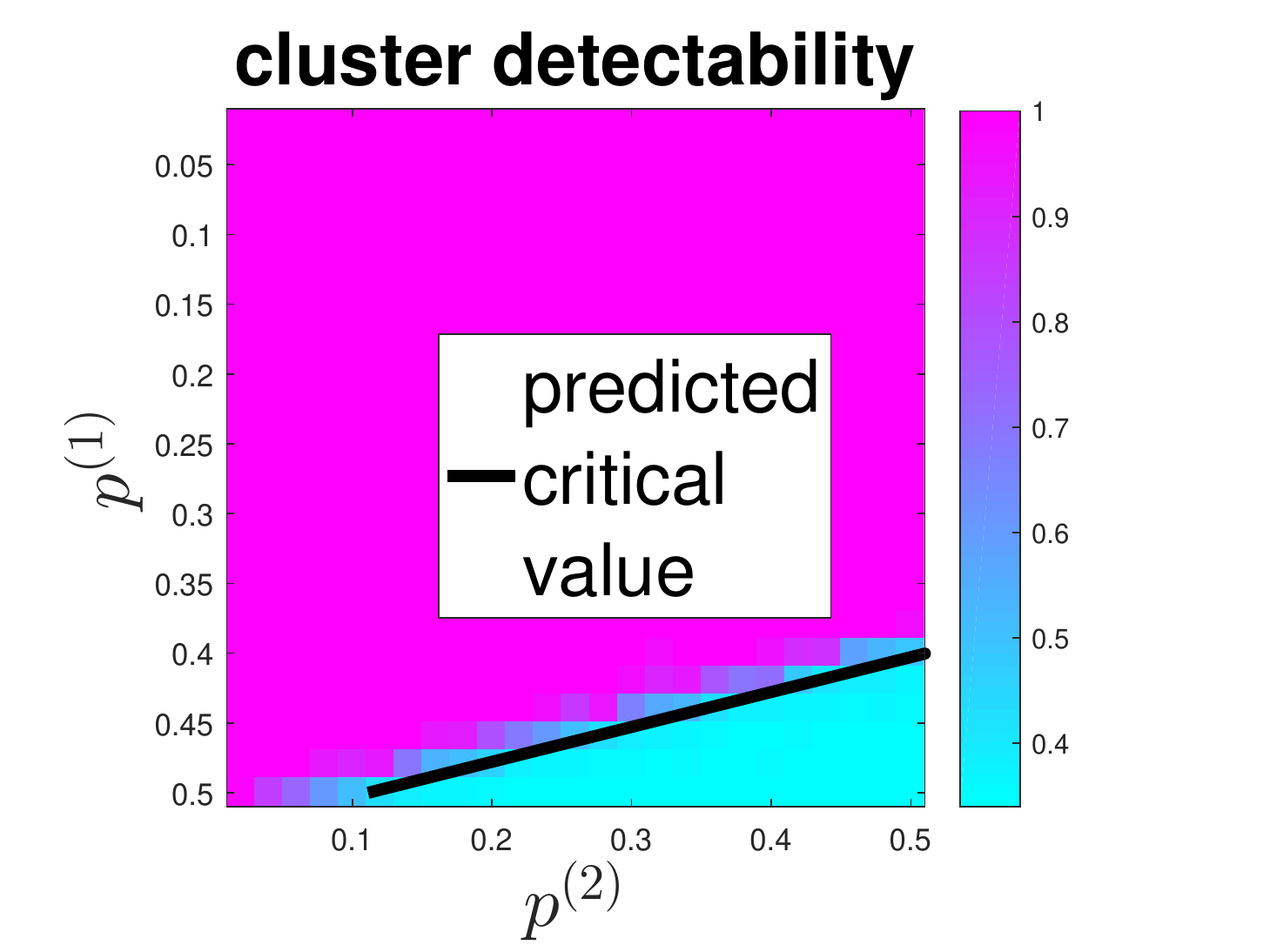}
			\caption{$(w_1,w_2)=(0.8,0.2)$}
		\end{subfigure}%
		\centering
		\begin{subfigure}[b]{0.245\linewidth}
			\includegraphics[width=\textwidth]{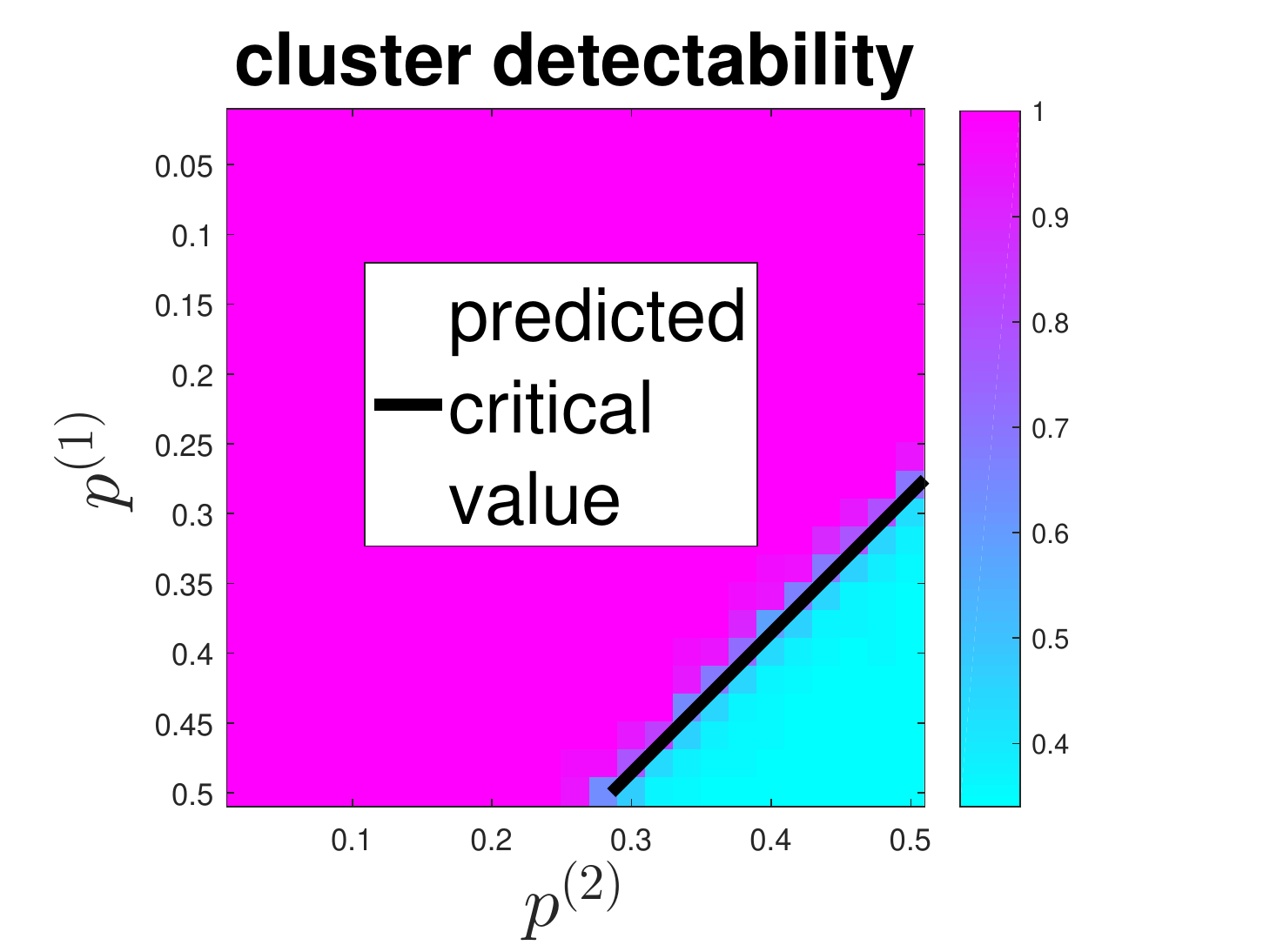}
			\caption{$(w_1,w_2)=(0.5,0.5)$}
		\end{subfigure}		
		\centering
		\begin{subfigure}[b]{0.245\linewidth}
			\includegraphics[width=\textwidth]{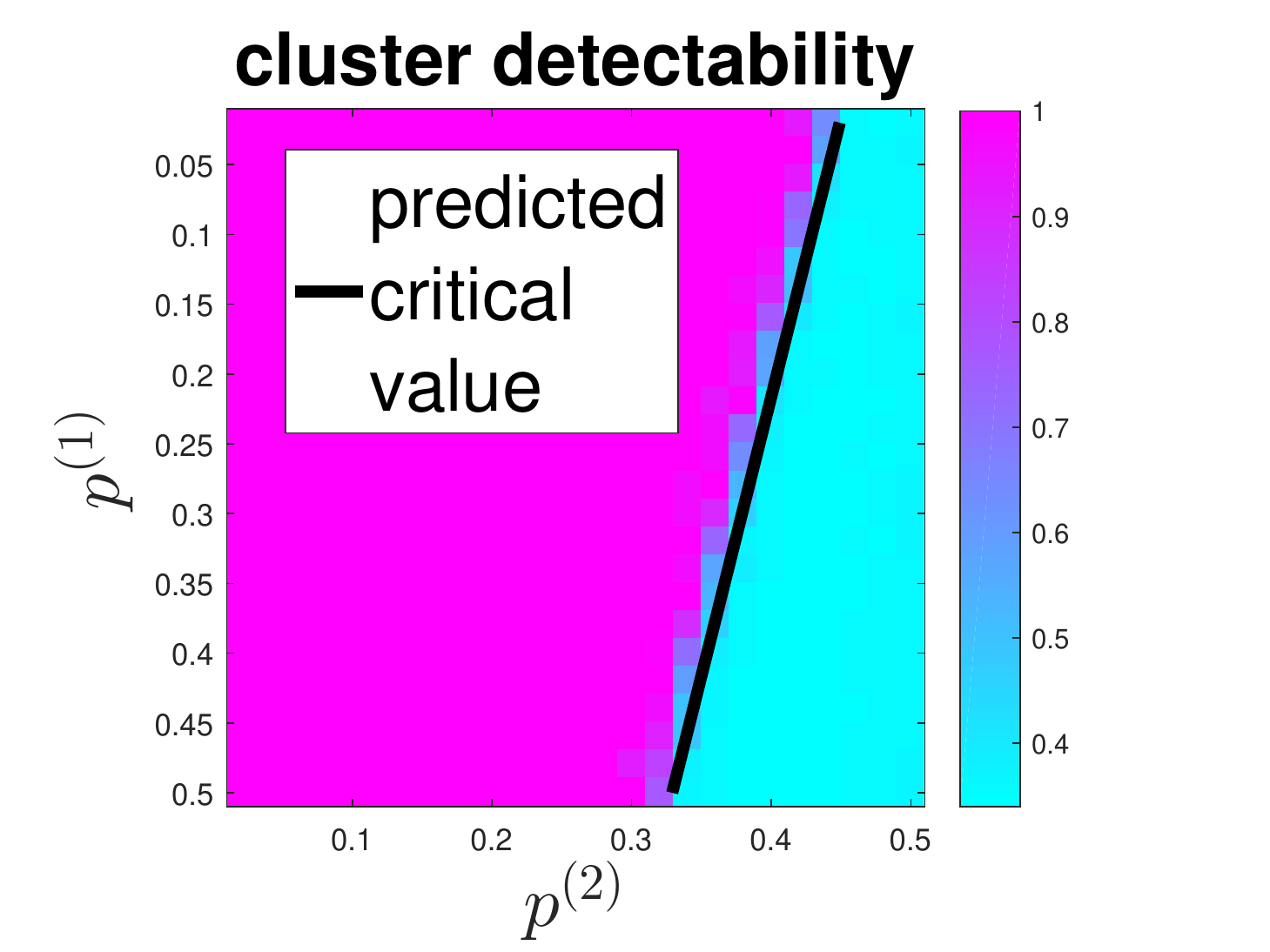}
			\caption{$(w_1,w_2)=(0.2,0.8)$}
		\end{subfigure}		
		\centering
		\begin{subfigure}[b]{0.245\linewidth}
			\includegraphics[width=\textwidth]{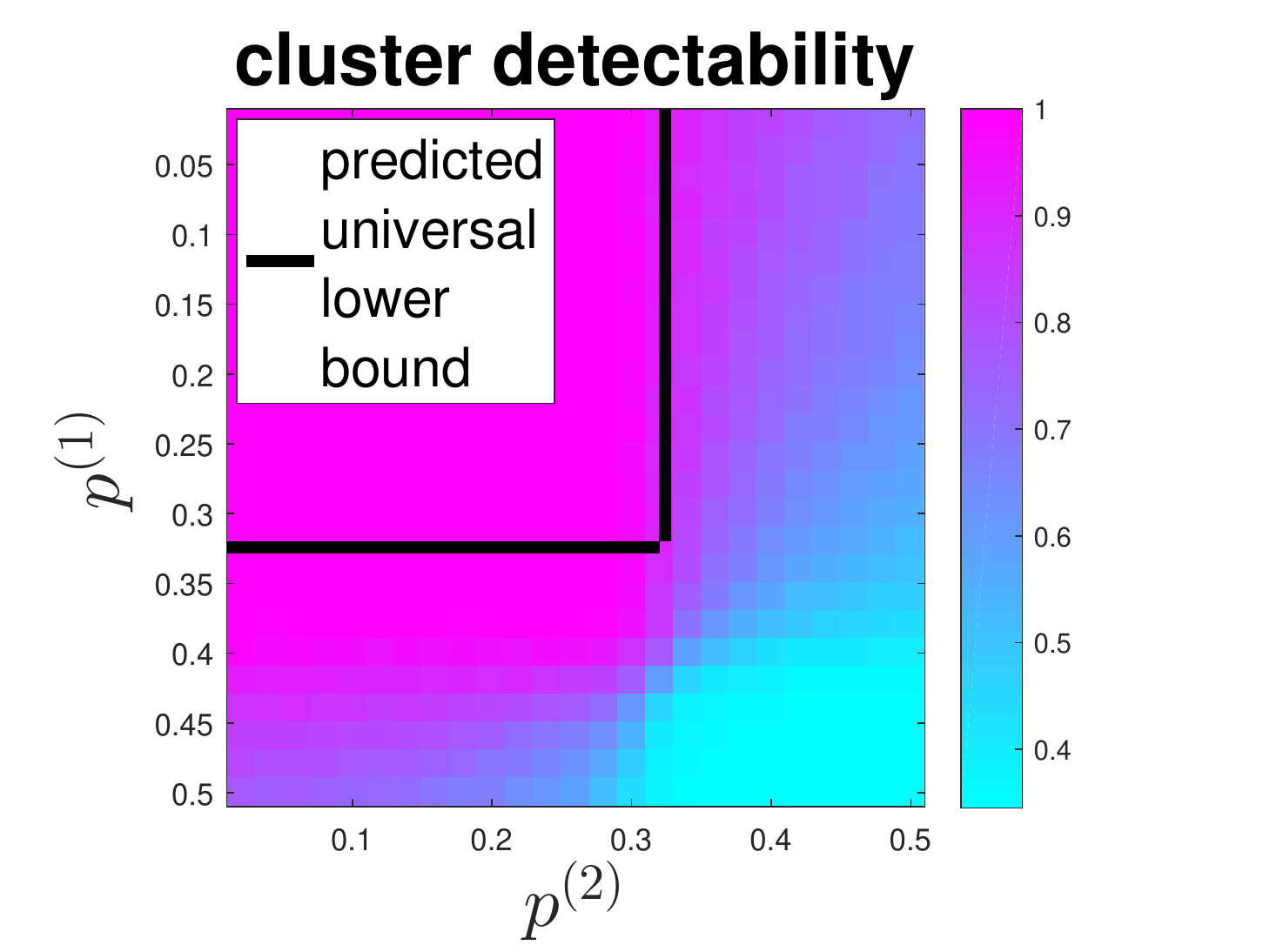}
			\caption{average over $\bw$}
		\end{subfigure}					
		\caption{Phase transitions in the accuracy of multilayer SGC with respect to different layer weight vector $\bw=[w_1~w_2]^T$ for the two-layer correlated graph model. $n_1=n_2=n_3=1000$, $q_{11}=0.3$, $q_{10}=0.2$, $q_{01}=0.1$, and $q_{00}=0.4$. The results are averaged over 10 runs.  In Fig. \ref{Fig_two_layer_detectability} (a)-(c), for a given $\bw$, the variations in the noise level $\{\pl\}_{\ell=1}^2$ indeed separates the accuracy of multilayer SGC into a reliable regime and an unreliable regime. Furthermore, the critical value that separates these two regimes is successfully predicted by Theorem \ref{thm_spec_ML}.
			Fig. \ref{Fig_two_layer_detectability} (d) shows the geometric mean over $\bw$, where $w_1$ is uniformly drawn from $[0,1]$ with unit interval $0.1$. There is a universal region of perfect cluster detectability that includes the region specified by the universal phase transition lower bound in (\ref{eqn_LB_tLB}).   			 
			}
		\label{Fig_two_layer_detectability}
		\vspace*{-4mm}
	\end{figure*}

\subsection{Clustering reliability test under the block-wise non-identical noise model}
\label{subsec_nonidentical_test}
In the iterative process of step 4 in MIMOSA, if every interconnection matrix $\hbCijl$ passes the block-wise homogeneity test in Sec. \ref{subsec_block_hom_test}, but some layers fail the clustering reliability test under the block-wise identical noise model in Sec. \ref{subsec_identical_test}, the identified clusters $\{ \cC_k^\bw\}_{k=1}^K$ are then used to test the clustering reliability under the block-wise non-identical noise model in Sec. \ref{subsec_ML_signal_noise} based on Theorem \ref{thm_principal_angle_ML}. Given a layer weight vector $\bw$, the noise level estimates $\{ \htijl \}$,
and the estimate $\htLBw$ of the phase transition lower bound in (\ref{eqn_estimate_htLBw}), we  compare the maximum noise level $\htmaxl=\max_{1\leq i,j \leq K, i\neq j} \htijl$ with $\htLBw$ for each layer $\ell$. 
	In the supplementary file we show that if 
	the estimated maximum noise level $\htmaxl$ of each layer $\ell$ satisfies a certain condition (condition (\ref{eqn_spectral_multi_confidence_interval_ingomogeneous_RIM_ML}) in the supplementary file), then if the aggregated maximum noise level $\htwmax=\sum_{\ell=1}^{L} w_\ell \htmaxl < \htLBw$, by Theorem \ref{thm_principal_angle_ML} the identified clusters $\{\cC_k\}_{k=1}^K$ are deemed reliable with high probability.

\subsection{A signal-to-noise ratio criterion for final clustering results}

In step 4 of MIMOSA, given the number of clusters $K$, if MIMOSA finds any feasible layer weight vector that passes the clustering reliability tests in Sec. \ref{subsec_identical_test} or Sec. \ref{subsec_nonidentical_test}, it then stores the vector in the set $\cW_{\textnormal{feasible}}$, and stops increasing $K$. This means that MIMOSA has identified a set of reliable clustering results 
of the same number of clusters $K$ based on the clustering reliability tests. To select the best clustering result from the feasible set, in step 5 we use the phase transition results established in Sec. \ref{sec_MIMOSA_THM} to define a signal-to-noise ratio (SNR) for each clustering result, which is 
\begin{align}
\label{eqn_SNR_MIMOSA}
\textnormal{SNR}^\bw=\frac{\htLBw}{\htw}.
\end{align}
$\htLBw$ can be viewed as the aggregated signal strength of within-cluster edges, and $\htw$ is the 
the aggregated noise level across layers. Therefore, the final clustering result is the clusters $\{\cC_k^{\bw^*}\}_{k=1}^K$, where $\bw^*=\arg \max_{\bw \in \cW_{\textnormal{feasible}}} \textnormal{SNR}^\bw$ is the layer weight vector having the largest SNR in the set $\cW_{\textnormal{feasible}}$.

\subsection{Computational complexity analysis}
The overall computational complexity of MIMOSA is  $O(|\cT|K^3(\widetilde{m}+n))$, where $K$ is the number of output clusters, $n$ is the number of nodes, and $\widetilde{m}=\sum_{\ell=1}^{L}|\cE_\ell|$ is the sum of total number of edges in each layer. The analysis is as follows.

Fixing model order $K$ and regularization parameter $\tau \in \cT$ in the MIMOSA iteration, as displayed in Fig.
\ref{Fig_MIMOSA_flow}, there are three main contributions to the computational complexity of MIMOSA: (i) Incremental eigenpair computation - acquiring an additional smallest eigenvector for augmenting $\bY$ of $\bLw$ takes $O(\widetilde{m}+n)$ operations via power iteration \cite{CPY_16KDDMLG,wu2016primme_svds}, since the maximum number of nonzero entries in  $\bLw$ is $\widetilde{m}+n$. 
(ii) Parameter estimation - estimating the RIM parameters $\{\pijl \}$ and $\{\Wbarijl\}$ takes $O(\widetilde{m})$ operations since they only depend on the number of edges and edge weights in each layer. Estimating  $\tLB$ takes $O(K(\widetilde{m}+n) \cdot K)=O(K^2(\widetilde{m}+n))$ operations for computing the numerator in (\ref{eqn_estimate_htLBw}). (iii) K-means clustering -  $O(nK^2)$ operations \cite{Zaki.Jr:14} for clustering $n$ data points of  dimension $K-1$ into $K$ groups.
Unfixing $\tau$, iterating this process over the elements in $\cT$ takes $O(|\cT|K^2(\widetilde{m}+n))$ operations. 
Finally, if MIMOSA outputs $K$ clusters, then the overall computational complexity is $O(|\cT|K^3(\widetilde{m}+n))$.

\section{Numerical Experiments}
\label{sec_num_ML}
To validate the phase transition results in the accuracy of multilayer SGC via convex layer aggregation established in Sec. \ref{sec_MIMOSA_THM}, we generate synthetic multilayer graphs from a two-layer correlated multilayer graph model. Specifically, we generate edge connections within and between $K=3$ equally-sized ground-truth clusters on $L=2$ layers $G_1$ and $G_2$. The two layers $G_1$ and $G_2$ are correlated since their edge connections are generated in the following manner.
For every node pair ($u,v$) of the same cluster, with probability $q_{11}$ there is a within-cluster edge ($u,v$) in $G_1$ and $G_2$, with probability  $q_{10}$ there is a within-cluster edge  ($u,v$) in $G_1$ but not in $G_2$, with probability  $q_{01}$ there is a within-cluster edge  ($u,v$) in $G_2$ but not in $G_1$, and with probability  $q_{00}$ there is no edge  ($u,v$) in $G_1$ and $G_2$.
These four parameters $\{q_{xy}\}_{x,y\in\{0,1\}}$are nonnegative and sum to $1$. For between-cluster edges, we adopt the block-wise identical noise model in Sec. \ref{subsec_ML_signal_noise} such that for each layer $\ell$, the edge connection between every node pair from different clusters is an i.i.d. Bernoulli random variable with parameter $\pl$.

\begin{figure*}[t]
	\centering
	\begin{subfigure}[b]{0.24\linewidth}
		\includegraphics[width=\textwidth]{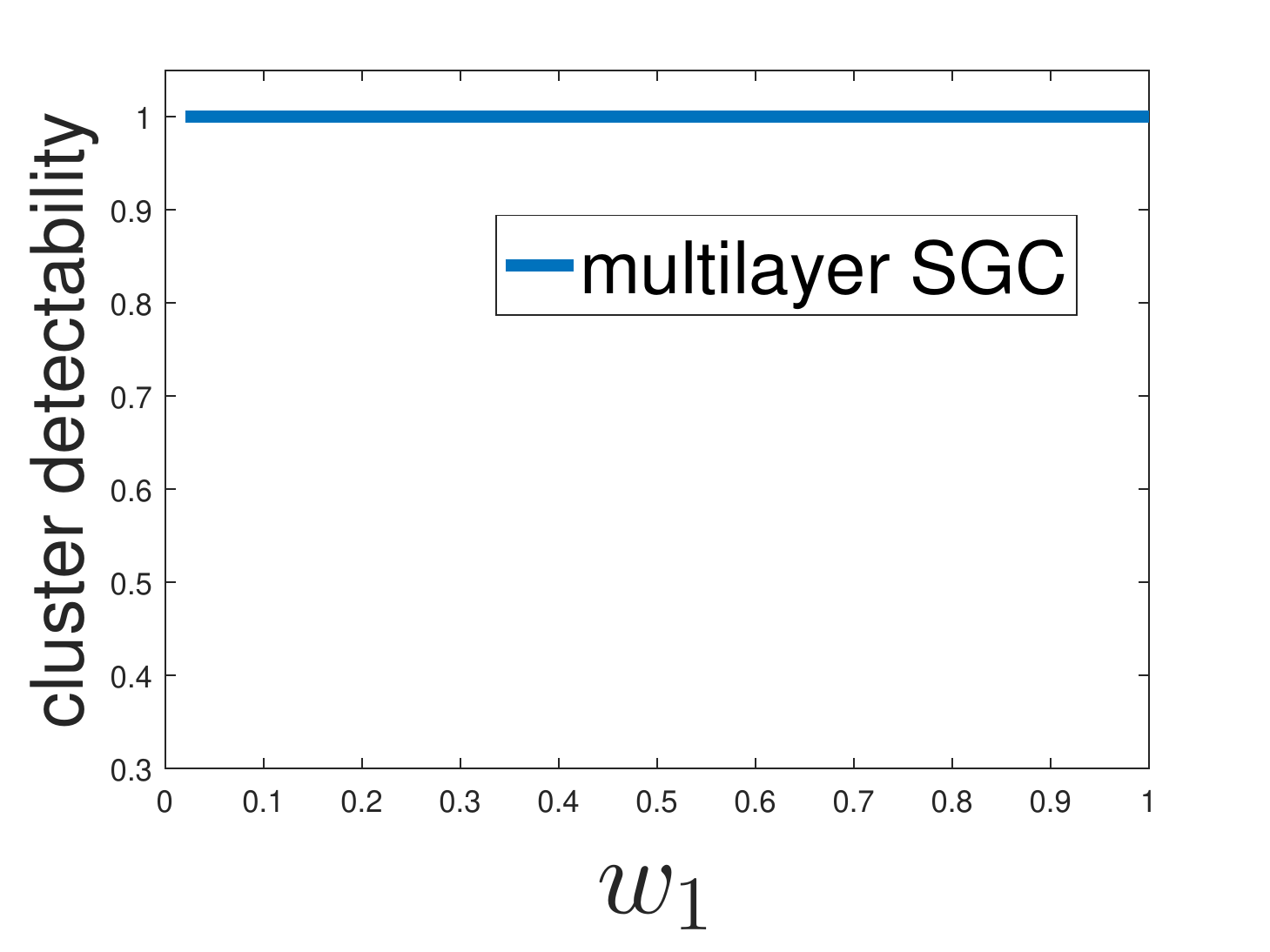}
		\caption{$(p^{(1)},p^{(2)})=(0.2,0.2)$}
	\end{subfigure}%
	\hspace{0.01cm}
	\centering
	\begin{subfigure}[b]{0.24\linewidth}
		\includegraphics[width=\textwidth]{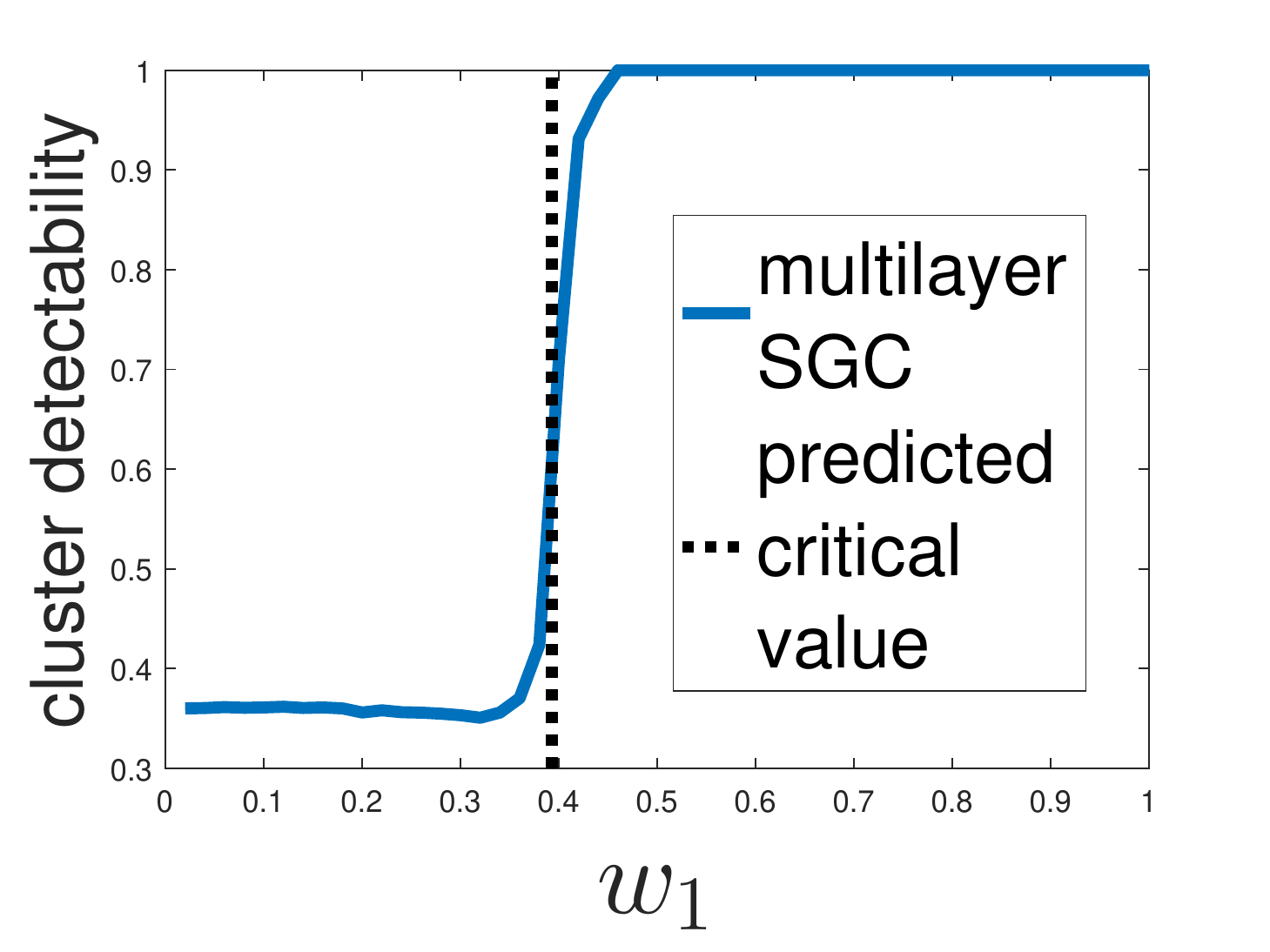}
		\caption{$(p^{(1)},p^{(2)})=(0.2,0.5)$}
	\end{subfigure}
	\hspace{0.01cm}
	\centering
	\begin{subfigure}[b]{0.24\linewidth}
		\includegraphics[width=\textwidth]{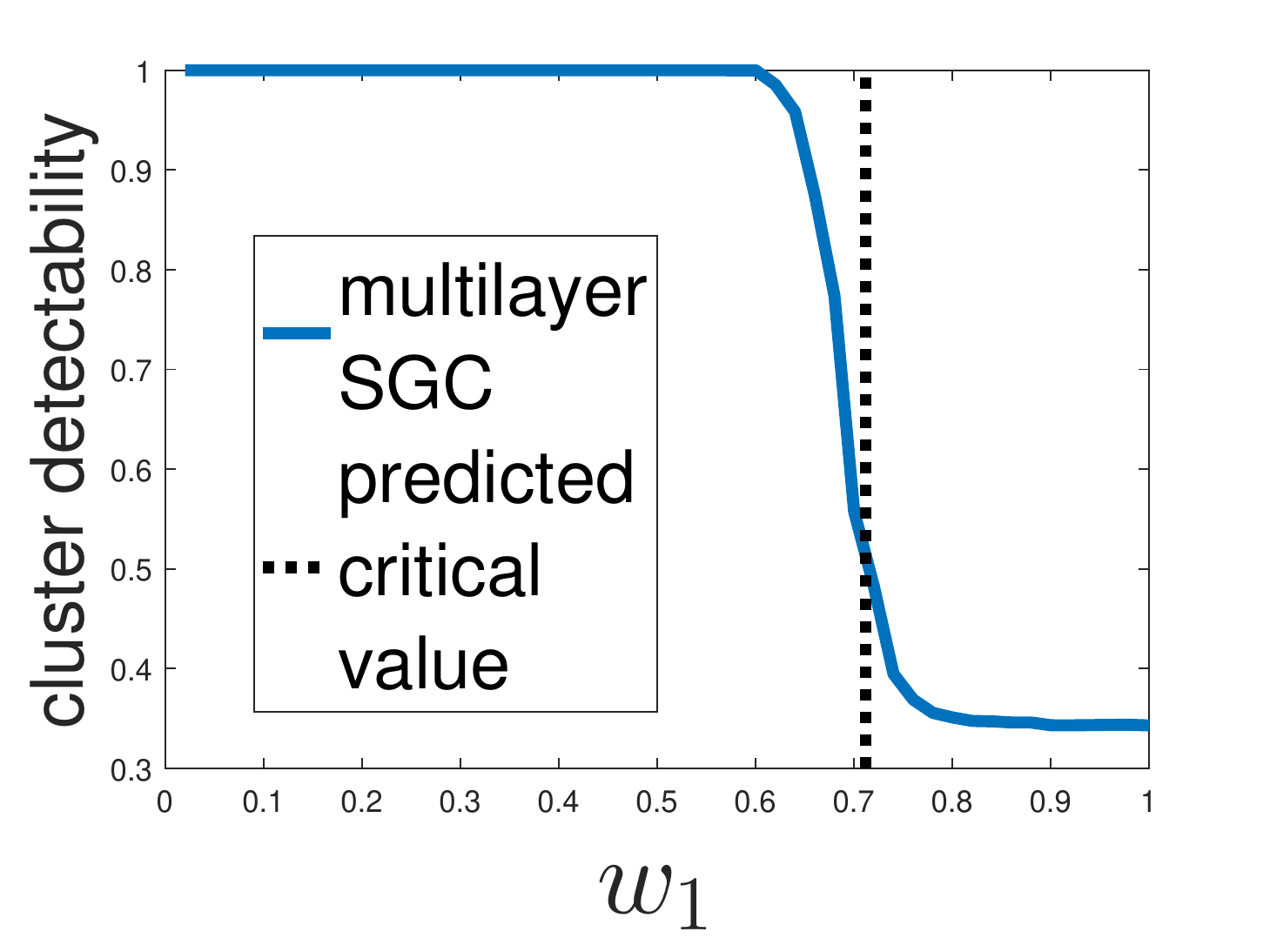}
		\caption{$(p^{(1)},p^{(2)})=(0.5,0.2)$}
	\end{subfigure}
	\hspace{0.01cm}
	\centering
	\begin{subfigure}[b]{0.24\linewidth}
		\includegraphics[width=\textwidth]{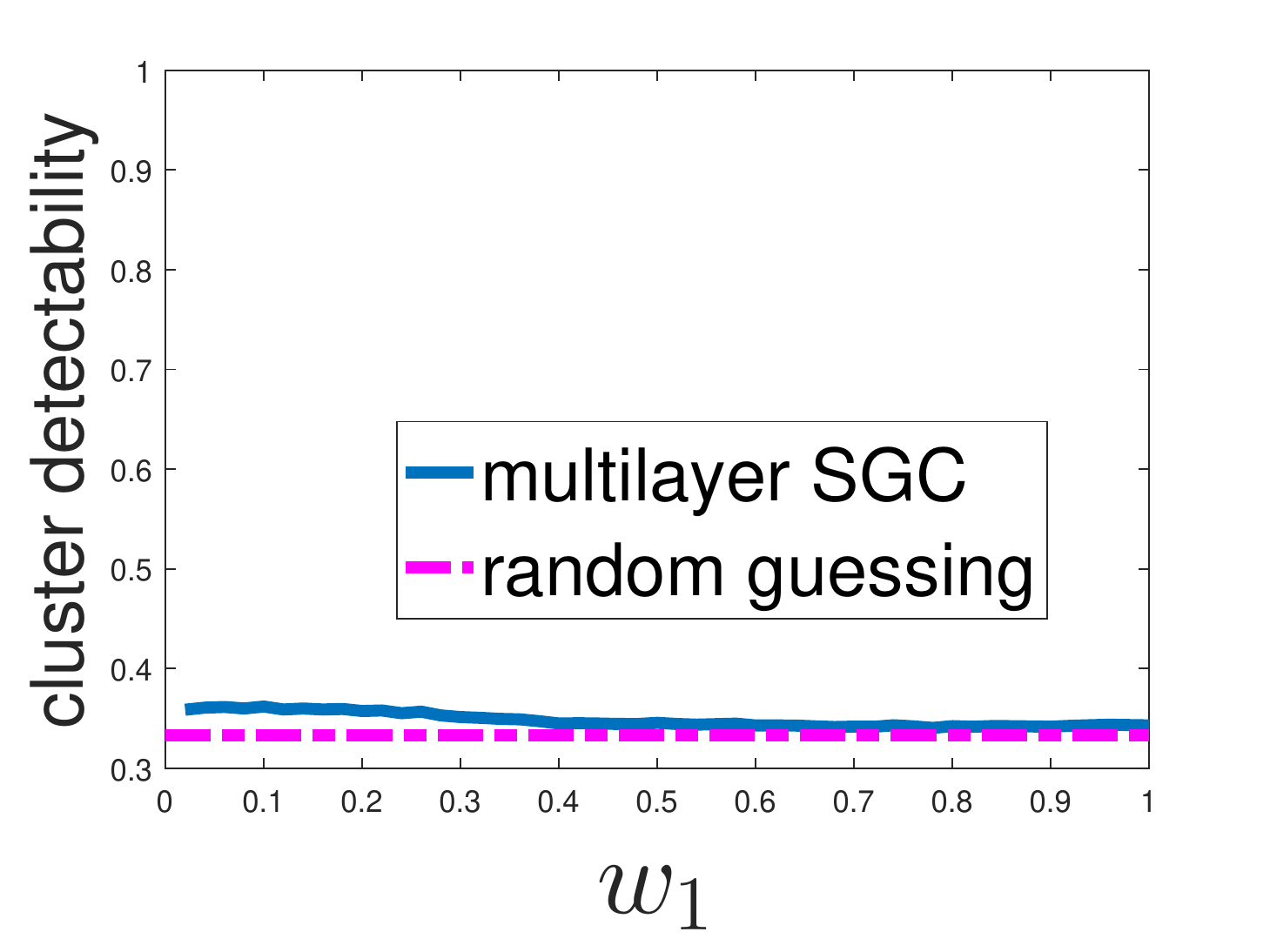}
		\caption{$(p^{(1)},p^{(2)})=(0.5,0.5)$}
	\end{subfigure}						
	\caption{The effect of the layer weight vector $\bw=[w_1~w_2]^T$ on the accuracy of multilayer SGC with respect to different noise levels $\{ \pl\}_{\ell=1}^2$ for the two-layer correlated graph model. $n_1=n_2=n_3=1000$, $q_{11}=0.3$, $q_{10}=0.2$, $q_{01}=0.1$, and $q_{00}=0.4$. The results are averaged over 50 runs. Fig. \ref{Fig_two_layer_weight} (a) shows that in the case of low noise level for each layer, any layer weight vector $\bw \in \cW_2$ can lead to correct clustering result.
		Fig. \ref{Fig_two_layer_weight} (b) and (c) show that 	
		if one layer has high noise level, then there may exist a critical value $w_1^* \in [0,1]$ that separates the cluster detectability into a reliable regime and an unreliable regime. Furthermore, the critical value $w_1^*$ is shown to satisfy the equation in (\ref{eqn_weight_equivalent}) derived from Theorem \ref{thm_spec_ML}. 
		Fig. \ref{Fig_two_layer_weight} (d) shows that in the case of high noise level for each layer, no layer weight vector can lead to correct clustering result, and the cluster detectability is similar to random guessing of clustering accuracy 33.33\%.
	} 
	\label{Fig_two_layer_weight} 
	\vspace*{-3mm}
\end{figure*}

\subsection{Phase transitions in multilayer SGC via convex layer aggregation}
By varying the noise level $\{\pl\}_{\ell=1}^2$, Fig. \ref{Fig_two_layer_detectability} shows the accuracy of multilayer SGC with respect to different layer weight vector $\bw=[w_1~w_2]^T$ and the averaged result over $\bw$, where the accuracy is evaluated in terms of cluster detectability.
Let $\{\cC_k\}_{k=1}^K$ and $\{\cC^\prime_{k}\}_{k=1}^K$ denote the detected and ground-truth clusters, respectively, and let $|\cC_k \cap \cC^\prime_{k}  |$ denote the number of common nodes in $\cC_k$ and $\cC^\prime_{k}$. Cluster detectability is defined as $\max_{\{\cC_{\widetilde{k}}\} \in \textnormal{Perm}(\{\cC_{{k}}\})} \frac{1}{n} \sum_{k=1}^K |\cC_{\widetilde{k}} \cap \cC^\prime_{k}  |$, where $\textnormal{Perm}(\{\cC_{{k}}\})$ is the set of all possible cluster label permutations of the detected clusters. In other words, cluster detectability requires consistency between the detected and ground-truth clusters.
 Given a fixed $\bw$, as proved in Theorem \ref{thm_spec_ML}, Fig. \ref{Fig_two_layer_detectability} (a)-(c) show that there is indeed a phase transition in cluster detectability that separates the noise level $\{\pl\}_{\ell=1}^2$ into two regimes: a reliable regime where high clustering accuracy is guaranteed, and an unreliable regime where high clustering accuracy is impossible. Furthermore, the critical value of $\{\pl\}_{\ell=1}^2$ that separates these two regimes are successfully predicted by Theorem \ref{thm_spec_ML} (c), which validates the phase transition analysis. Fig. \ref{Fig_two_layer_detectability} (d) shows the geometric mean of cluster detectability from different layer weight vectors. There is a universal region of perfect cluster detectability that includes the region specified by the universal phase transition lower bound in (\ref{eqn_LB_tLB}).

\begin{table}[]
	\centering
	\caption{Summary of real-world multilayer graph datasets.} 
	\label{table_multilayer}
	\begin{tabular}{c|c|cl}
		\cline{1-3}
		Dataset                                                                         & \# of layers & \begin{tabular}[c]{@{}c@{}}ground-truth cluster\\  labels and cluster sizes\end{tabular}               &  \\ \cline{1-3}
		VC 7th grader                                                                   & 3            & \begin{tabular}[c]{@{}c@{}}boys (12)\\ girls (17)\end{tabular}                                &  \\ \cline{1-3}
		\begin{tabular}[c]{@{}c@{}}Leskovec-Ng \\ collaboration \\ network\end{tabular} & 4            & \begin{tabular}[c]{@{}c@{}}Leskovec's collaborator \\(87)\\ Ng's collaborator (104)\end{tabular} &  \\ \cline{1-3}
		\begin{tabular}[c]{@{}c@{}}109th Congress \\ votes - Budget\end{tabular}        & 4            & \begin{tabular}[c]{@{}c@{}}Democratic (45)\\ Republican (55) \end{tabular}                     &  \\ \cline{1-3}
		\begin{tabular}[c]{@{}c@{}}109th Congress \\ votes - Energy\end{tabular}        & 2            & \begin{tabular}[c]{@{}c@{}}Democratic (45)\\ Republican (55)\end{tabular}                     &  \\ \cline{1-3}
		\begin{tabular}[c]{@{}c@{}}109th Congress \\ votes- Security\end{tabular}       & 2            & \begin{tabular}[c]{@{}c@{}}Democratic (45)\\ Republican (55)\end{tabular}                     &  \\ \cline{1-3}
		Reality mining                                                                  & 2            & None                                                                                &  \\ \cline{1-3}
		\begin{tabular}[c]{@{}c@{}}London \\ transportation \\ network\end{tabular}     & 2            & None                                                                                &  \\ \cline{1-3}
		\begin{tabular}[c]{@{}c@{}}Human H1V1 \\ genetic interaction\end{tabular}       & 5            & None                                                                                &  \\ \cline{1-3}
		\begin{tabular}[c]{@{}c@{}}Pierre Auger \\ coauthorship\end{tabular}            & 16           & None                                                                                &  \\ \cline{1-3}
	\end{tabular}
	\vspace{-4mm}
\end{table}

\subsection{The effect of layer weight vector on multilayer SGC  via convex layer aggregation}
Next we investigate the effect of layer weight vector $\bw$ on multilayer SGC via convex layer aggregation given fixed noise levels  $\{ \pl\}_{\ell=1}^2$.
In the two-layer graph setting, since by definition $w_2=1-w_1$, it suffices to study the effect of $w_1$ on clustering accuracy.
Fig. \ref{Fig_two_layer_weight} shows the clustering accuracy by varying $w_1$ under the two-layer correlated graph model. As shown in Fig. \ref{Fig_two_layer_weight} (a), if each layer has low noise level, then any layer weight vector $\bw \in \cW_2$ can lead to correct clustering result. If one layer has high noise level, Fig. \ref{Fig_two_layer_weight} (b) and (c) show that there exists a critical value $w_1^\star \in [0,1]$ that separates the cluster detectability into a reliable regime and an unreliable regime. In particular, Theorem \ref{thm_spec_ML} implies that the critical value $w_1^\star$, if existed, satisfies the condition $t^{\bw}= t^{\bw^*}$ when $\bw=[w_1^\star,1-w_1^\star]^T=\bw^*$, which is equivalent to
\begin{align}
\label{eqn_weight_equivalent}
&\frac{K-1}{K} \Lb w_1^\star p^{(1)} + (1-w_1^\star) p^{(2)}  \Rb= w_1^\star \cdot \min_{k \in \{1,2,\ldots,K\}} \SK \lb \frac{\bL^{(1)}_k}{n} \rb  \nonumber \\
&~~~+ (1-w_1^\star) \cdot \min_{k \in \{1,2,\ldots,K\}} \SK \lb \frac{\bL^{(2)}_k}{n} \rb. 
\end{align}
It is observed in Fig. \ref{Fig_two_layer_weight} (b) and (c) that the empirical critical value $w_1^\star$ matches the predicted value from   (\ref{eqn_weight_equivalent}). Lastly,  as shown in Fig. \ref{Fig_two_layer_weight} (d), if each layer has high noise level, then no layer weight vector can lead to correct clustering result, and the corresponding cluster detectability is similar to random guessing of clustering accuracy $\frac{1}{K} \approx$ 33.33\%.

\section{MIMOSA on Real-World Multi-Layer Graphs}
\label{sec_MIMOSA_data}

\subsection{Dataset descriptions}
\label{subsec_MIMOSA_dataset}

In this section, we apply MIMOSA to 9 real-world multilayer graphs and compute the external and internal clustering metrics for quality assessment. The statistics of the 9 real-world multilayer graphs are summarized in Table \ref{table_multilayer}, and the details are described as follows.
\begin{itemize}
	\item \textbf{VC 7th grader social network} \cite{vickers1981representing}: This dataset is based on a survey of social relations among 29 7th grade students in Victoria, Australia, including  12 boys and 17 girls. A 3-layer graph is created based on different relationships, including ``friends you get on with'', ``your best friends'', and ``friends you prefer to work with'' in the class. For each layer we only retain the edges where there is  mutual agreement between every student pair.
	
	\item \textbf{Leskovec-Ng collaboration network\footnote{The dataset can be downloaded from https://sites.google.com/site/pinyuchenpage/datasets}}: We collected the coauthors of Prof. Jure Leskovec or Prof. Andrew Ng at Stanford University
	from ArnetMiner \cite{tang2008arnetminer} from year 1995 to year 2014. In total, there are 191 researchers in this dataset. We partition coauthorship over a 20-year period into 4 different 5-year intervals and hence create a 4-layer multilayer graph. For each layer, there is an edge between two researchers if they coauthored at least one paper 
	in the 5-year interval. For every edge in each layer, we adopt the temporal collaboration strength as the edge weight \cite{Zhang.Saha.ea:14,Saha.Zhang.ea:15}. 
	Notably, while Prof.  Leskovec and Prof. Ng both were members of  the same department, there is no record of coauthorship between them on ArnetMiner. However, they are connected through a common co-author, Christopher Potts. As a result the full
	collaboration network among 191 researchers is a connected graph. 
	We manually label each researcher by either ``Leskovec's collaborator'' or ``Ng's collaborator'' based on the collaboration frequency, and use the labels as the ground-truth cluster assignment. The ground-truth clusters with researcher names are displayed in Fig. \ref{Fig_JureNg}.
	 
	 \item  \textbf{109th Congress votes}: We collected the votes of 100 senators of the 109th U.S. Congress to create 3 multilayer graph datasets based on the topic area of each	 
	 bill on which they voted, including ``Budget'', ``Energy'', and ``Security''. Only bills on which every senator has voting records are considered in these datasets. For each bill topic (a multilayer graph) we create a layer for each bill. In each layer, there is an edge between two senators if they vote  the same way. We use the party (Democratic or Republican) as the ground-truth cluster label. In addition, we label the one independent senator as Democratic since he  caucused with the Democrats.
	 
	 \item \textbf{Reality mining} \cite{pentland2009inferring}: The reality mining dataset contains mobile and social traces among 94 MIT students. We extract the largest connected component of students from this dataset to form a 2-layer graph, where one layer represents user connection via text messaging, and the other layer represents user connection via proximity (Bluetooth). For each layer we only retain  edges
	 for which there is mutual contact between student pairs. 
	 
	 \item \textbf{London transportation network} \cite{de2014navigability}: The London transportation network dataset contains different transportation routes through Tube stations in London.
	 We extract the largest connected component of stations that are either connected by 
	 Overground transportation or by Docklands Light Railway (DLR) to form a 2-layer graph, where one layer represents 	overground connectivity, and the other layer represents DLR connectivity.	
	 
	 \item  \textbf{Human H1V1 genetic interaction \cite{de2014muxviz}}: The human H1V1 genetic interaction dataset contains different types of genetic interactions among 1005 proteins.  We extract the largest connected genetic interaction network from this dataset to form a 5-layer graph, where each interaction type corresponds to one layer and for each layer we only retain the edge of mutual interaction.
	 
	 \item \textbf{Pierre Auger coauthorship \cite{PhysRevX.5.011027}}: The Pierre Auger coauthorship dataset contains the coauthorship among 514 researchers between 2010 and 2012 associated with the Pierre Auger Observatory, which involves 16 working research tasks (layers) related to studies of ultra-high energy cosmic rays. We extract the largest connected component from this network to form a 16-layer graph.
	  
\end{itemize}

\begin{figure}[!t]
	\centering
	\includegraphics[width=3.6in]{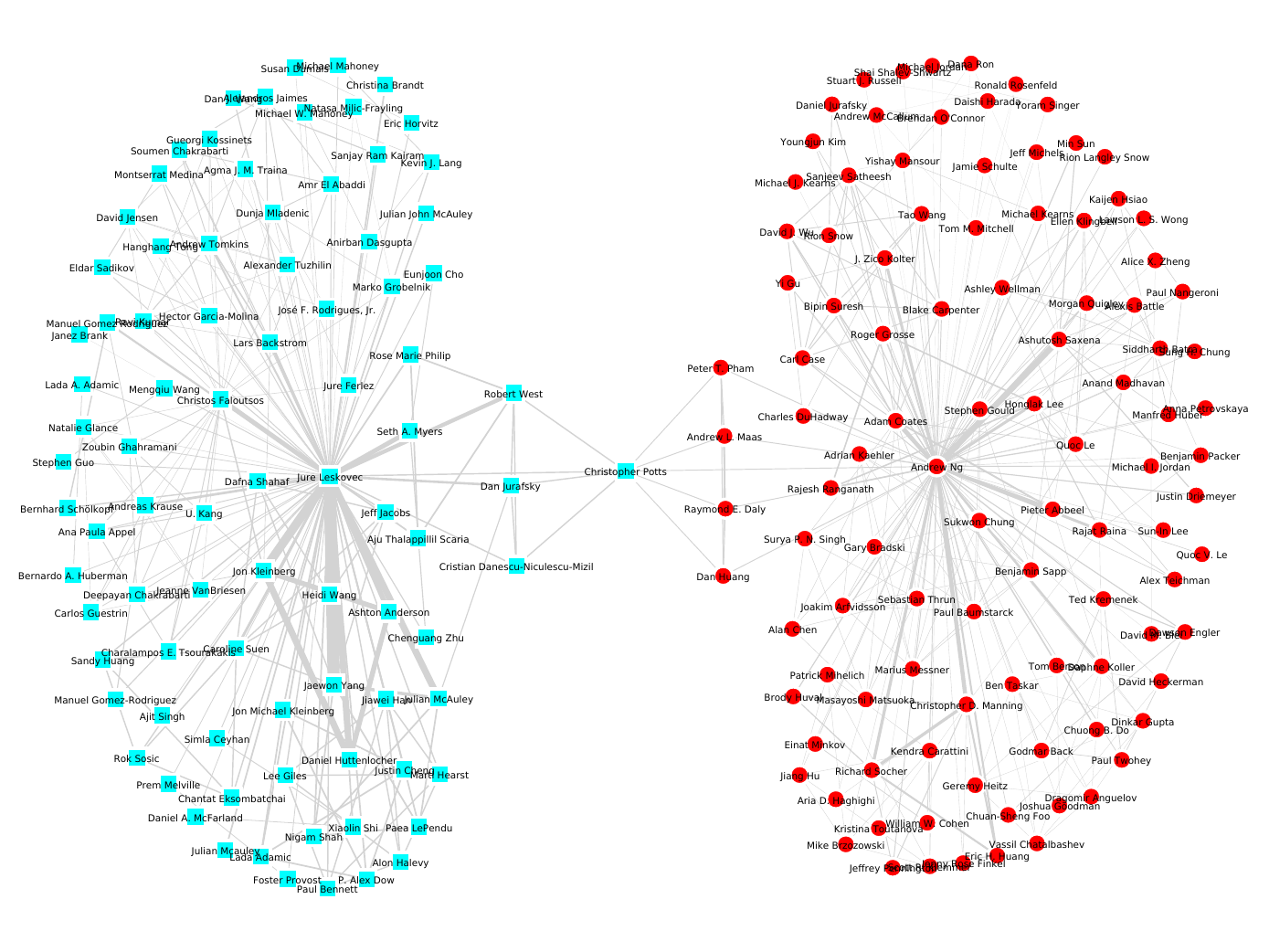}
	\caption{Ground-truth clusters of the collected Leskovec-Ng collaboration network. Nodes represent researchers, edges represent the strength of coauthorship  \cite{Zhang.Saha.ea:14,Saha.Zhang.ea:15}, and colors and shapes represent two clusters - ``Leskovec's collaborator'' (cyan square) or ``Ng's collaborator'' (red circle).}	
	\label{Fig_JureNg}
	\vspace{-4mm}
\end{figure}

Since MIMOSA allows the input multilayer graph to be weighted, for each layer $G_\ell$, if $G_\ell$ is unweighted, we adopt the degree normalization \cite{Luxburg07} such that the ($u,v$)-th entry in the weight matrix $\bWl$ is $[\bWl]_{uv}=\frac{[\bAl]_{uv}}{\sqrt{d^{(\ell)}_u \cdot d^{(\ell)}_v}}$ if $d^{(\ell)}_u,d^{(\ell)}_v>0$, and $[\bWl]_{uv}=0$ otherwise, where $\bAl$ is the adjacency matrix of $G_\ell$ and $d^{(\ell)}_u$ is the degree of node $u$ in $G_\ell$.

\subsection{Performance evaluation}
Using the multilayer graph datasets described in Table \ref{table_multilayer}, we compare the clustering performance of MIMOSA with four other methods.
The first method is the baseline approach that assigns uniform weight to each layer in the convex layer aggregation (i.e., $w_\ell = \frac{1}{L}$~$\forall~\ell$). Since this baseline approach is equivalent to MIMOSA with the setting $\bwini=\frac{\bone_L}{L}$ and $\cT=\{0\}$, we call this method \textit{MIMOSA-uniform}. The second method is a greedy multilayer modularity maximization approach that extends the Louvain method for clustering in single-layer graphs to multilayer graphs, which is called \textit{GenLouvain}\footnote{http://netwiki.amath.unc.edu/GenLouvain/GenLouvain}. GenLouvain aims to merge the nodes to maximize the multilayer modularity defined in \cite{mucha2010community} in a greedy manner. The third method is the multilayer graph clustering algorithm proposed in \cite{dong2014clustering}, called SC-ML. The fourth method is the Self-Tuning algorithm \cite{zelnik2004self} for graph clustering in single-layer graphs, where the single-layer graph is obtained by summing the edge weights across all layers.

For GenLouvain, we set the 
resolution parameter $\gamma \in \{0.5,1,2\}$ and the latent inter-layer coupling parameter $\omega=1$. For MIMOSA, we set $\bwini=\frac{\bone_L}{L}$ to be a uniform vector, $\eta=10^{-5}$, $\alpha_\ell=\alpha^\prime_\ell=0.05$~$\forall~\ell$, and the regularization set $\cT=\{0, 10^{-1}, 10^{0}, 10^{1},10^{2},10^{3},10^{4},10^{5}\}$. The effect of the parameters in MIMOSA on the output clusters are summarized as follows.  If one has some prior knowledge of the noise level in each layer, then adjusting $\bwini$ by assigning more weights to less noisy layers may yield better clustering results. Increasing $\eta$ or decreasing $\{\alpha_\ell\}$ and $\{\alpha^\prime_\ell\}$ tightens the clustering reliability constraint and may increase the number of output clusters. Expanding $\cT$ may yield better clustering results.
 Like MIMOSA, GenLouvain and Self-Tuning are automated clustering algorithms that do not require specifying the number of clusters $K$ \textit{a priori}. SC-ML requires the knowledge of $K$, and for performance comparison we set the value of $K$ in SC-ML to be the number of clusters found by MIMOSA.

We use the following external and internal clustering metrics to evaluate the performance of different methods. External metrics can be computed only when ground-truth cluster labels are known, whereas internal metrics can be computed in the absence of ground-truth cluster labels. In particular,
 since these internal metrics are designed for single-layer graphs, in the evaluation we extend these internal metrics to multilayer graphs by summing the metrics defined at each layer.  The clustering metrics are summarized as follows. Specifically,  we denote the $K$ clusters identified by a graph clustering algorithm by $\{\cC_k\}_{k=1}^K$, and denote the $K^\prime$ ground-truth clusters by $\{\cC^\prime_k\}_{k=1}^{K^\prime}$.

	$\bullet~$\textbf{External clustering metrics} 
	\begin{enumerate}
		\item normalized mutual information (NMI) \cite{strehl2002cluster}: NMI is defined as 
		\begin{align}
		\label{eqn_NMI}
		\textnormal{NMI}(\{\cC_k\}_{k=1}^K,\{\cC^\prime_k\}_{k=1}^{K^\prime})=\frac{2 \cdot I(\{\cC_k\},\{\cC^\prime_k\})}{|H(\{\cC_k\})+H(\{\cC^\prime_k\})|},
		\end{align}
		where $I$ is the mutual information between $\{\cC_k\}_{k=1}^K$ and $\{\cC^\prime_k\}_{k=1}^{K^\prime}$, and $H$ is the entropy of clusters. Larger NMI means better clustering performance.
		
		\item Rand index (RI) \cite{rand1971objective}: RI is defined as 
		\begin{align}
		\label{eqn_RI}
			\textnormal{RI}(\{\cC_k\}_{k=1}^K,\{\cC^\prime_k\}_{k=1}^{K^\prime})=\frac{TP+TN}{TP+TN+FP+FN},
		\end{align}
		where $TP$, $TN$, $FP$ and $FN$ represent true positive, true negative, false positive, and false negative decisions, respectively. Larger RI means better clustering performance.
		
		\item F-measure \cite{Rijsbergen:1979:IR:539927}:  F-measure is the harmonic mean of the precision and recall values for each cluster, which is defined as
		\begin{align}
		\label{eqn_Fmeasure}
			\textnormal{F-measure}(\{\cC_k\}_{k=1}^K,\{\cC^\prime_k\}_{k=1}^{K^\prime})=\frac{1}{K} \sum_{k=1}^K 	\textnormal{F-measure}_k,
		\end{align}
		where $	\textnormal{F-measure}_k= \frac{2 \cdot PREC_k \cdot RECALL_k}{PREC_k + RECALL_k}$, and $PREC_k $ and $RECALL_k$ are the precision and recall values for cluster $\cC_k$. Larger F-measure means better clustering performance.
	\end{enumerate}		

\begin{table*}[]
	\centering
	\caption{Summary of the number of identified clusters ($K$) and the external and internal clustering metrics. ``NA'' means ``not applicable'', and ``-'' means ``not available'' due to lack of ground-truth cluster labels. For each dataset, the method that leads to the highest clustering metric is highlighted in bold face.}
	\label{table_MIMOSA_performance}
\begin{tabular}{c|ccccccc}
	\hline
	Dataset                                                                                        & Method                    & K  & NMI             & RI              & F-measure       & conductance     & NC              \\ \hline
	\multirow{7}{*}{\begin{tabular}[c]{@{}c@{}}VC 7th grader \\ social network\end{tabular}}       & MIMOSA                    & 2  & \textbf{0.8123} & \textbf{0.9310} & \textbf{0.9317} & \textbf{0.2649} & \textbf{0.4330} \\
	& MIMOSA-uniform            & NA & NA              & NA              & NA              & NA              & NA              \\
	& GenLouvain ($\gamma=0.5$) & 3  & 0.6495          & 0.7833          & 0.7333          & 0.4487          & 0.6051          \\
	& GenLouvain ($\gamma=1$)   & 3  & 0.6495          & 0.7833          & 0.7333          & 0.4487          & 0.6051          \\
	& GenLouvain ($\gamma=2$)   & 8  & 0.4418          & 0.5911          & 0.3197          & 1.4295          & 1.6081          \\
	& SC-ML                     & 2  & 0.6119          & 0.8079          & 0.8040          & 0.2756          & 0.4618          \\
	& Self-Tuning               & 6  & 0.5345          & 0.6995          & 0.5764          & 0.4329          & 0.5510          \\ \hline
	\multirow{7}{*}{\begin{tabular}[c]{@{}c@{}}Leskovec-Ng\\ collaboration\\ network\end{tabular}} & MIMOSA                    & 2  & \textbf{1}      & \textbf{1}      & \textbf{1}      & \textbf{0.0213} & \textbf{0.0415} \\
	& MIMOSA-uniform            & NA & NA              & NA              & NA              & NA              & NA              \\
	& GenLouvain ($\gamma=0.5$) & 7  & 0.6824          & 0.8488          & 0.8243          & 0.1989          & 0.2663          \\
	& GenLouvain ($\gamma=1$)   & 16 & 0.4972          & 0.7156          & 0.6055          & 0.3054          & 0.3702          \\
	& GenLouvain ($\gamma=2$)   & 29 & 0.3553          & 0.5586          & 0.2173          & 0.4874          & 0.5569          \\
	& SC-ML                     & 2  & \textbf{1}      & \textbf{1}      & \textbf{1}      & \textbf{0.0213} & \textbf{0.0415} \\
	& Self-Tuning               & 2  & \textbf{1}      & \textbf{1}      & \textbf{1}      & \textbf{0.0213} & \textbf{0.0415} \\ \hline
	\multirow{7}{*}{\begin{tabular}[c]{@{}c@{}}109th Congress\\ votes - Budget\end{tabular}}       & MIMOSA                    & 2  & 0.7959          & 0.9224          & 0.9220          & 0.2713          & 0.4975          \\
	& MIMOSA-uniform            & 2  & \textbf{0.8778} & \textbf{0.9604} & \textbf{0.9603} & 0.2702          & 0.5055          \\
	& GenLouvain ($\gamma=0.5$) & 2  & 0.7959          & 0.9224          & 0.9220          & 0.2713          & 0.4978          \\
	& GenLouvain ($\gamma=1$)   & 2  & 0.7959          & 0.9224          & 0.9220          & 0.2713          & 0.4978          \\
	& GenLouvain ($\gamma=2$)   & 55 & 0.3822          & 0.6915          & 0.5539          & \textbf{0.1500} & \textbf{0.1959} \\
	& SC-ML                     & 2  & 0.7610          & 0.9040          & 0.9036          & 0.2742          & 0.5089          \\
	& Self-Tuning               & 3  & 0.8488          & 0.9164          & 0.9087          & 1.5046          & 1.8011          \\ \hline
	\multirow{7}{*}{\begin{tabular}[c]{@{}c@{}}109th Congress\\ votes - Energy\end{tabular}}       & MIMOSA                    & 2  & \textbf{0.7290} & \textbf{0.8861} & \textbf{0.8855} & \textbf{0.1151} & \textbf{0.2086} \\
	& MIMOSA-uniform            & 2  & 0.6716          & 0.8513          & 0.8508          & 0.1154          & 0.2178          \\
	& GenLouvain ($\gamma=0.5$) & 2  & 0.5403          & 0.8182          & 0.8173          & \textbf{0.1151} & \textbf{0.2086} \\
	& GenLouvain ($\gamma=1$)   & 2  & 0.5403          & 0.8182          & 0.8173          & \textbf{0.1151} & \textbf{0.2086} \\
	& GenLouvain ($\gamma=2$)   & 7  & 0.6371          & 0.8521          & 0.8422          & 0.3145          & 0.3593          \\
	& SC-ML                     & 2  & 0.6716          & 0.8513          & 0.8508          & 0.1154          & 0.2178          \\
	& Self-Tuning               & 4  & 0.6310          & 0.8521          & 0.8424          & 1.0204          & 1.0970          \\ \hline
	\multirow{7}{*}{\begin{tabular}[c]{@{}c@{}}109th Congress\\ votes - Security\end{tabular}}     & MIMOSA                    & 2  & 0.6105          & 0.8513          & 0.8506          & 0.0400            & 0.0785          \\
	& MIMOSA-uniform            & 2  & 0.6304          & 0.8513          & 0.8506          & 0.0400            & 0.0785          \\
	& GenLouvain ($\gamma=0.5$) & 2  & 0.5816          & 0.8345          & 0.8337          & 0.0400            & 0.0770        \\
	& GenLouvain ($\gamma=1$)   & 2  & \textbf{0.6598} & \textbf{0.8685} & \textbf{0.8678} & 0.0400            & 0.0770           \\
	& GenLouvain ($\gamma=2$)   & 4  & 0.6181          & 0.8515          & 0.8477          & \textbf{0.0204} & \textbf{0.0492} \\
	& SC-ML                     & 2  & 0.6304          & 0.8513          & 0.8506          & 0.0400            & 0.0785          \\
	& Self-Tuning               & 2  & 0.6304          & 0.8513          & 0.8506          & 0.0400            & 0.0785          \\ \hline
	\multirow{7}{*}{Reality mining}                                                                & MIMOSA                    & 2  & -               & -               & -               & \textbf{0.0819} & \textbf{0.1573} \\
	& MIMOSA-uniform            & 2  & -               & -               & -               & \textbf{0.0819} & \textbf{0.1573} \\
	& GenLouvain ($\gamma=0.5$) & 3  & -               & -               & -               & 0.2239          & 0.3165          \\
	& GenLouvain ($\gamma=1$)   & 3  & -               & -               & -               & 0.2239          & 0.3165          \\
	& GenLouvain ($\gamma=2$)   & 6  & -               & -               & -               & 0.1240          & 0.2011          \\
	& SC-ML                     & 2  & -               & -               & -               & \textbf{0.0819} & \textbf{0.1573} \\
	& Self-Tuning               & 4  & -               & -               & -               & 0.4267          & 0.5247          \\ \hline
	\multirow{7}{*}{\begin{tabular}[c]{@{}c@{}}London \\ transportation\\ network\end{tabular}}    & MIMOSA                    & 5  & -               & -               & -               & 0.0553          & 0.0801          \\
	& MIMOSA-uniform            & 5  & -               & -               & -               & 0.0553          & 0.0801          \\
	& GenLouvain ($\gamma=0.5$) & 9  & -               & -               & -               & 0.1046          & 0.1286          \\
	& GenLouvain ($\gamma=1$)   & 14 & -               & -               & -               & 0.1558          & 0.1763          \\
	& GenLouvain ($\gamma=2$)   & 21 & -               & -               & -               & 0.2001          & 0.2181          \\
	& SC-ML                     & 5  & -               & -               & -               & 0.1044          & 0.1425          \\
	& Self-Tuning               & 26 & -               & -               & -               & \textbf{0.0154} & \textbf{0.0798} \\ \hline
\end{tabular}
	\vspace*{-4mm}
\end{table*}

\begin{table*}[]
	\centering
	\begin{tabular}{c|ccccccc}
		\hline
		Dataset                                                                                        & Method                    & K  & NMI             & RI              & F-measure       & conductance     & NC              \\ \hline
\multirow{7}{*}{\begin{tabular}[c]{@{}c@{}}Human H1V1\\ genetic interaction\end{tabular}}      & MIMOSA                    & 2  & -               & -               & -               & \textbf{0.0346} & \textbf{0.0666} \\
& MIMOSA-uniform            & 2  & -               & -               & -               & \textbf{0.0346} & \textbf{0.0666} \\
& GenLouvain ($\gamma=0.5$) & 4  & -               & -               & -               & 0.1822          & 0.2292          \\
& GenLouvain ($\gamma=1$)   & 4  & -               & -               & -               & 0.1822          & 0.2292          \\
& GenLouvain ($\gamma=2$)   & 5  & -               & -               & -               & 0.1458          & 0.3167          \\
& SC-ML                     & 2  & -               & -               & -               & 0.1161          & 0.2027          \\
& Self-Tuning               & 7  & -               & -               & -               & 0.5627          & 0.8722          \\ \hline
\multirow{7}{*}{\begin{tabular}[c]{@{}c@{}}Pierre Auger \\ coauthorship\end{tabular}}          & MIMOSA                    & 2  & -               & -               & -               & \textbf{0.0113} & \textbf{0.1888} \\
& MIMOSA-uniform            & NA & -               & -               & -               & NA              & NA              \\
& GenLouvain ($\gamma=0.5$) & 9  & -               & -               & -               & 1.5207          & 1.8423          \\
& GenLouvain ($\gamma=1$)   & 13 & -               & -               & -               & 1.2655          & 1.4699          \\
& GenLouvain ($\gamma=2$)   & 61 & -               & -               & -               & 0.5717          & 0.6356          \\
& SC-ML                     & 2  & -               & -               & -               & 1.2939          & 2.5181          \\
& Self-Tuning               & 63 & -               & -               & -    & 0.8400          & 0.9321         \\ \hline
	\end{tabular}
	\vspace*{-4mm}
\end{table*}

 $\bullet~$\textbf{Internal clustering metrics}	
	\begin{enumerate}
		\item conductance \cite{Shi00}: conductance  is defined as 
		\begin{align}
		\label{eqn_conductance}
		\textnormal{conductance}(\{\cC_k\}_{k=1}^K)=\frac{1}{K} \sum_{k=1}^K	\textnormal{conductance}_k,
		\end{align}
		where $\textnormal{conductance}_k=\frac{W^{out}_k}{2 \cdot W^{in}_k+W^{out}_k}$, and 
		$W^{in}_k$ and $W^{out}_k$ are the sum of within-cluster and between-cluster edge weights of cluster $\cC_k$, respectively. Lower conductance means better clustering performance.
		
		\item  normalized cut (NC) \cite{Shi00}:  NC is defined as 
		\begin{align}
		\label{eqn_NC}
		\textnormal{NC}(\{\cC_k\}_{k=1}^K)=\frac{1}{K} \sum_{k=1}^K	\textnormal{NC}_k,
		\end{align}		
		where 	$\textnormal{NC}_k=\frac{W^{out}_k}{2 \cdot W^{in}_k+W^{out}_k}+\frac{W^{out}_k}{2 \cdot (W^{all}_k-W^{in}_k)+W^{out}_k}$, and 	$W^{in}_k$, $W^{out}_k$ and $W^{all}_k$ are the sum of within-cluster, between-cluster and total edge weights of cluster $\cC_k$, respectively. Lower NC means better clustering performance.
	\end{enumerate}

Table \ref{table_MIMOSA_performance} summarizes the external and internal clustering metrics obtained after multilayer graph clustering by  the four methods for the datasets listed in Table \ref{table_multilayer}. For MIMOSA and MIMOSA-uniform, we terminate the iterative process and report the clustering result as  ``not applicable'' (NA) when the number of clusters $K$ exceeds $\frac{n}{2}$, where $n$ is the number of nodes. As a result, NA means that before termination no clustering results have passed the clustering reliability tests.

It is observed from Table \ref{table_MIMOSA_performance} that MIMOSA has the best clustering performance among 6 out of 9 datasets. For the Congress-votes-Budget and Congress-votes-Security datasets, MIMOSA performs somewhat worse than MIMOSA-uniform. 
For the VC 7th grader social network, Leskovec-Ng collaboration network and Pierre Auger coauthorship datasets, MIMOSA-uniform fails to find a reliable clustering result, whereas  MIMOSA has superior clustering metrics.
The robustness of MIMOSA implies the utility of layer weight adaptation, and it also suggests that assigning uniform weight to every layer regardless of the noise level may lead to unreliable clustering results. Comparing MIMOSA to SC-ML with the same number of clusters, the clusters found by MIMOSA have better clustering metrics. MIMOSA also outperforms Self-Tuning in most of the datasets, suggesting that simply summing a multilayer graph to create a single-layer graph does not necessarily benefit multilayer graph clustering. 
 In addition, we also observe that GenLouvain tends to identify more clusters than the number of ground-truth clusters.
 The fact that MIMOSA-uniform and Self-Tuning outperform MIMOSA in some cases is likely due to the fact that these particular datasets have similar connectivity in each layer. For example, in the Congres-votes-Budget dataset almost every senator voted along party lines on all budget related legislation. 

	\begin{figure*}[t]
		\centering
		\begin{subfigure}[b]{0.24\linewidth}
			\includegraphics[width=\textwidth]{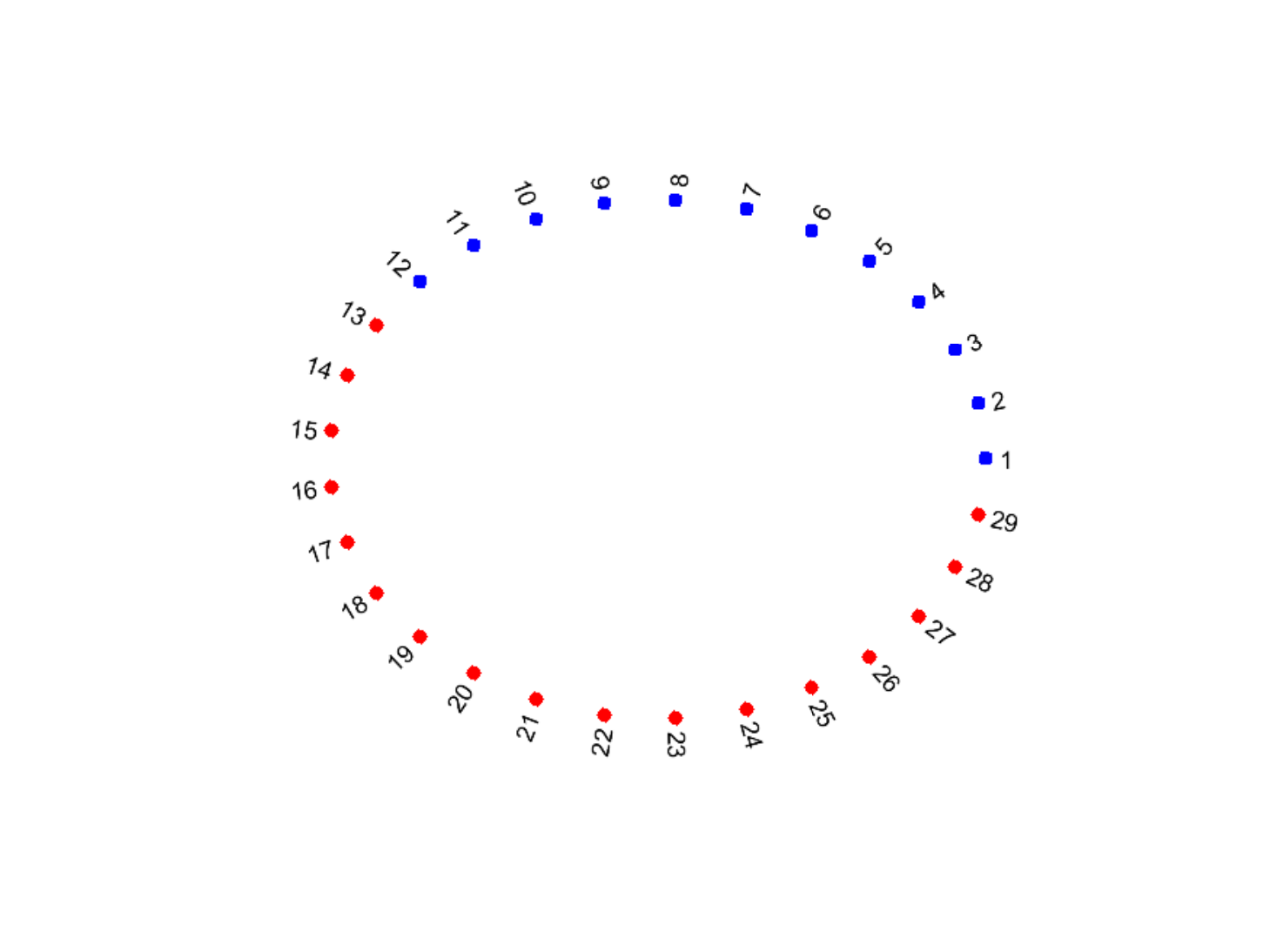}
			\vspace{-10mm}
			\caption{Ground-truth clusters}
		\end{subfigure}%
		\hspace{0.01cm}					
		\centering
		\begin{subfigure}[b]{0.24\linewidth}
			\includegraphics[width=\textwidth]{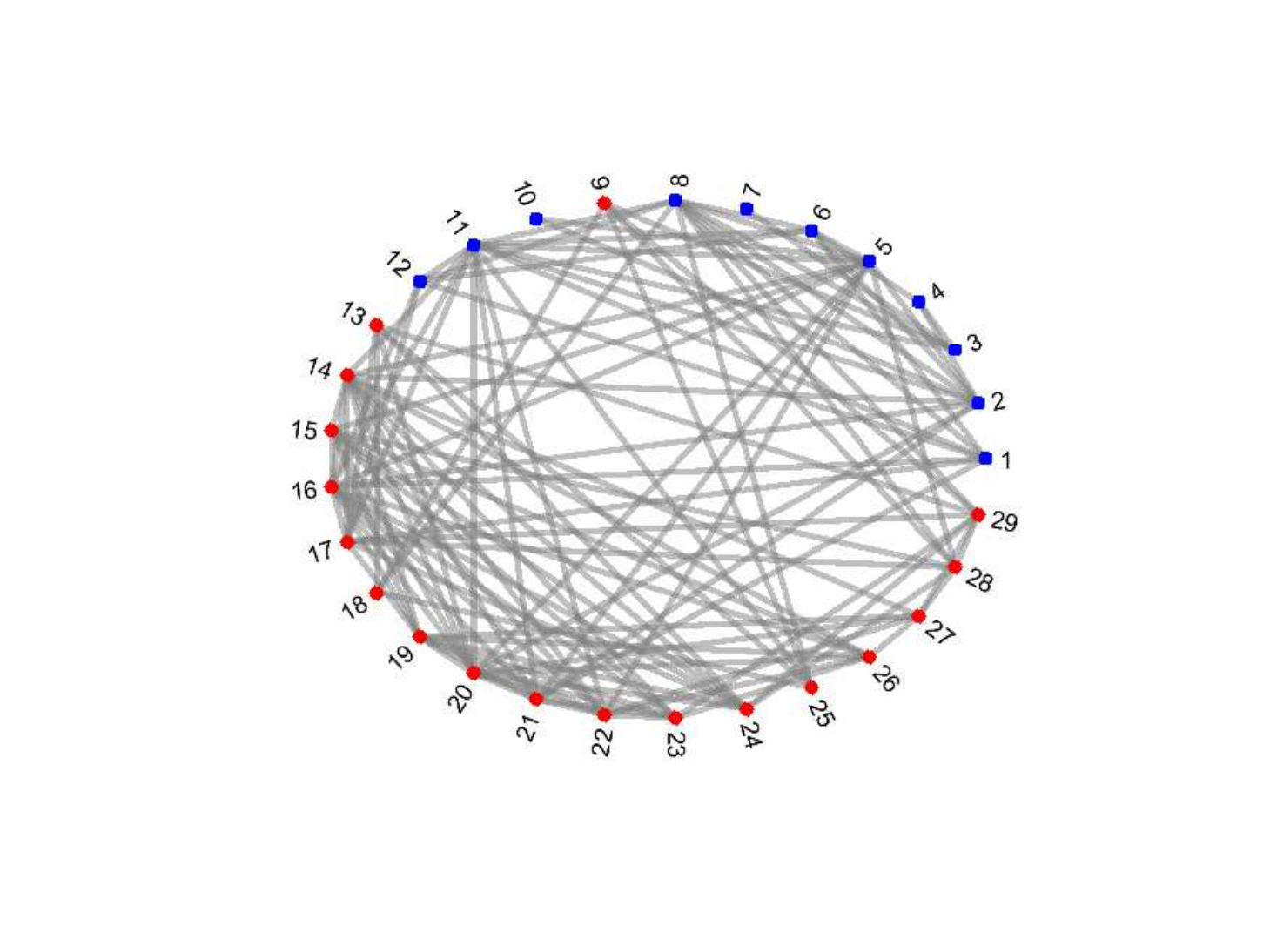}
			\vspace{-10mm}
			\caption{Friends you get on with}
		\end{subfigure}%
		\hspace{0.01cm}
		\centering
		\begin{subfigure}[b]{0.24\linewidth}
			\includegraphics[width=\textwidth]{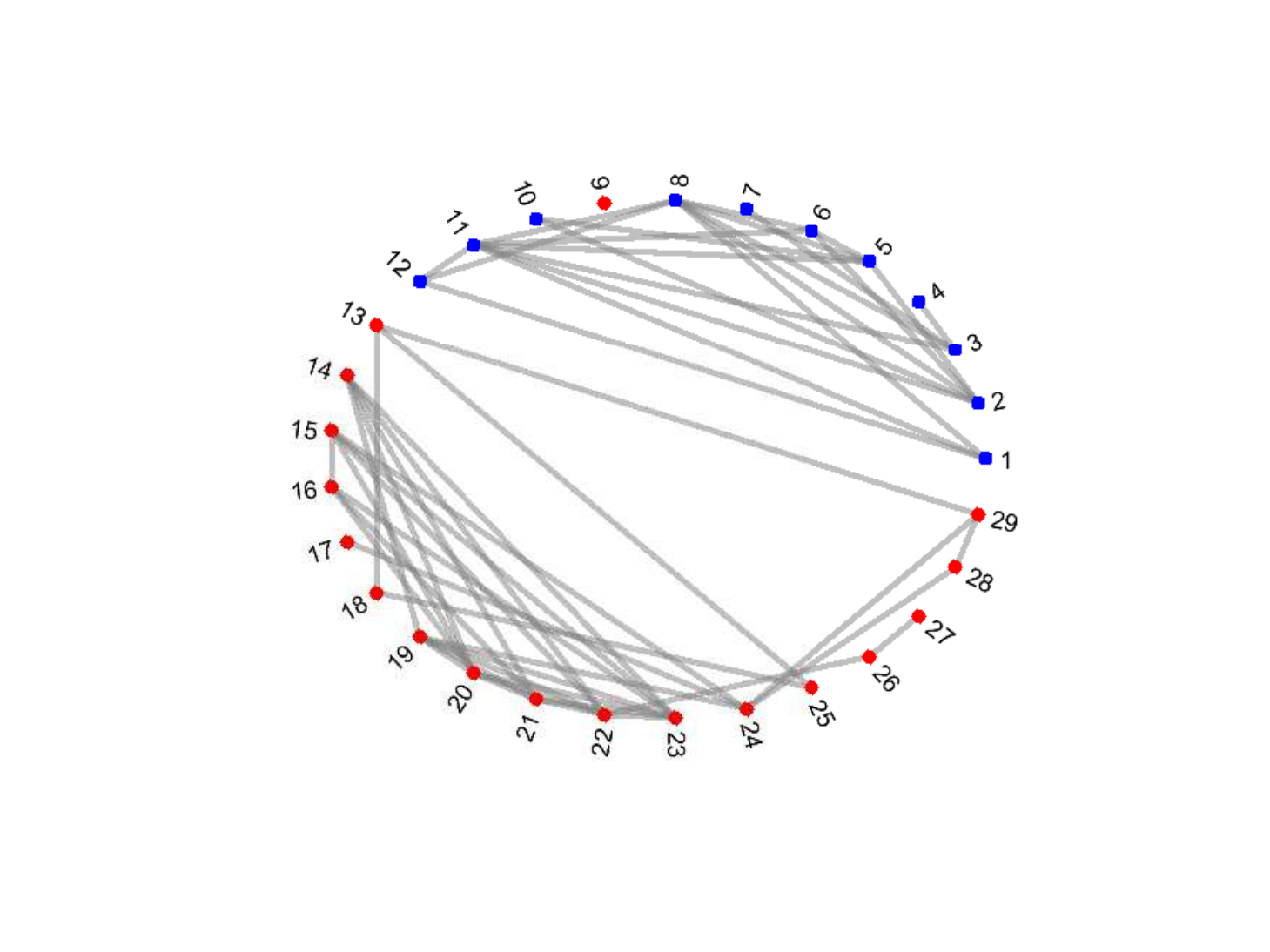}
			\vspace{-10mm}
			\caption{Your best friends}
		\end{subfigure}
		\hspace{0.01cm}
		\centering
		\begin{subfigure}[b]{0.24\linewidth}
			\includegraphics[width=\textwidth]{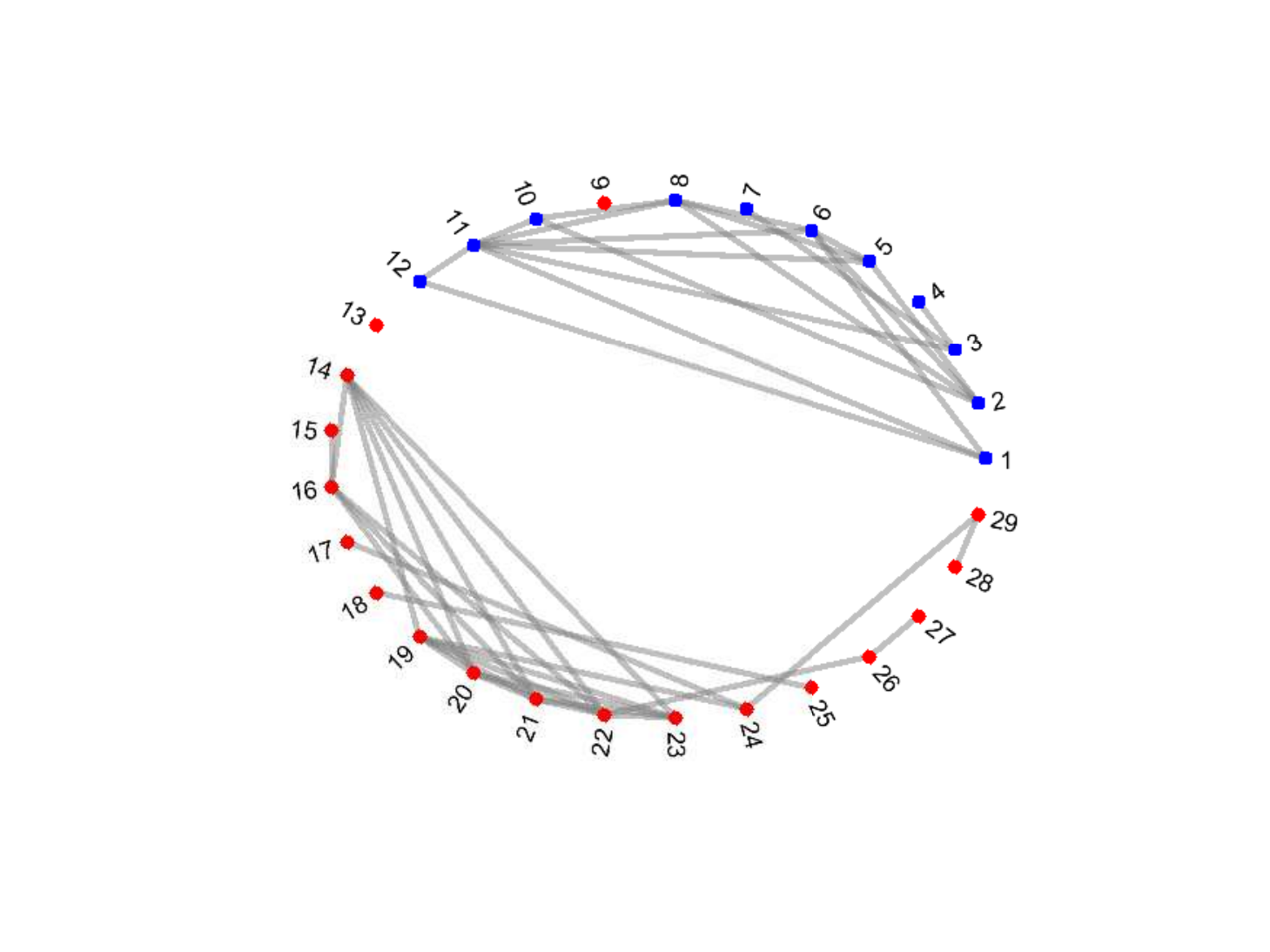}
			\vspace{-10mm}
			\caption{Friends you work with}
		\end{subfigure}		
		\caption{Illustration of the ground-truth clusters and the clusters found by MIMOSA for the VC 7th grader social network dataset. Fig. \ref{Fig_VC_plot} (a) displays the ground-truth clusters, where nodes 1 to 12 are boys (labeled by blue color) and nodes 13 to 29 are girls (labeled by red color).  Fig. \ref{Fig_VC_plot} (b) to (d) display the clusters (labeled by different colors) found by MIMOSA in each layer. Comparing to the  ground-truth clusters, MIMOSA correctly group all nodes into 2 clusters except for node 9, since node 9 has no edge connections in Fig. \ref{Fig_VC_plot} (c) and (d), and has more connections to girls than boys in Fig. \ref{Fig_VC_plot} (b). Enlarged plots are displayed in the supplementary material.
		}
		\label{Fig_VC_plot}
		\vspace*{-5mm}
	\end{figure*}

As a visual illustration, Fig. \ref{Fig_VC_plot} displays the ground-truth clusters and the clusters identified by MIMOSA for each layer of the VC 7th grader social network dataset.
The number of clusters identified by MIMOSA is 2, which is consistent with the ground truth.  The optimal layer weight vector obtained from step 5 of MIMOSA in Algorithm \ref{algo_MIMOSA} is $\bw^*=[ 0.0531~0.1608~0.7861]^T$. Comparing each layer with the ground-truth clusters, it can be observed that the connectivity patterns in Fig. \ref{Fig_VC_plot} (c) and (d) are more consistent with the ground truth, whereas the connectivity pattern in Fig. \ref{Fig_VC_plot} (b) is less informative, which explains why MIMOSA adapts more weights to the second and the third layers. Furthermore,  Fig. \ref{Fig_VC_plot} also explains why MIMOSA-uniform does not yield reliable clustering results, since it assigns uniform weight to each layer and is insensitive to the noise distribution.
It is worth noting that MIMOSA correctly groups all nodes into 2 clusters except for node 9. However, we also observe that node 9 has no edge connections in the two informative layers as shown in Fig. \ref{Fig_VC_plot} (c) and (d), and indeed has more connections to girls than boys in the first layer as  shown in Fig. \ref{Fig_VC_plot} (b), which leads to the misclassification of node 9 when compared with the ground-truth clusters.

\section{Conclusion}
\label{sec_conclusion}
We have characterized the phase transition that governs the accuracy of a convex aggregation method of multilayer spectral graph clustering (SGC). By varying the noise level, as measured by the edge connection probability of spurious between-cluster edges, we specified the critical value that separates the performance of multilayer  SGC into a reliable regime and an unreliable regime. The phase transition was validated via numerical experiments. Furthermore, based on the phase transition analysis, we proposed MIMOSA,  a multilayer SGC algorithm that provides automated model order selection for cluster assignment and layer weight adaptation with statistical clustering reliability guarantees. Applying MIMOSA to real-world multilayer graphs shows competitive or better clustering performance with respect to several 
baseline methods, including
the uniform weight assignment, a greedy multilayer modularity maximization method, and a subspace approach. Our future work will include extending the phase transition analysis and MIMOSA to  other multilayer block models.

\section*{Acknowledgment}
The first author would like to thank Baichuan Zhang at the Department of Computer and Information Science, Indiana University - Purdue University Indianapolis, for his help in analyzing the Leskovec-Ng collaboration network dataset\footnotemark[2].

\bibliographystyle{IEEEtran}
\bibliography{IEEEabrv20160824,CPY_ref_20170319}

\begin{thebibliography}{10}
\providecommand{\url}[1]{#1}
\csname url@rmstyle\endcsname
\providecommand{\newblock}{\relax}
\providecommand{\bibinfo}[2]{#2}
\providecommand\BIBentrySTDinterwordspacing{\spaceskip=0pt\relax}
\providecommand\BIBentryALTinterwordstretchfactor{4}
\providecommand\BIBentryALTinterwordspacing{\spaceskip=\fontdimen2\font plus
\BIBentryALTinterwordstretchfactor\fontdimen3\font minus
  \fontdimen4\font\relax}
\providecommand\BIBforeignlanguage[2]{{%
\expandafter\ifx\csname l@#1\endcsname\relax
\typeout{** WARNING: IEEEtran.bst: No hyphenation pattern has been}%
\typeout{** loaded for the language `#1'. Using the pattern for}%
\typeout{** the default language instead.}%
\else
\language=\csname l@#1\endcsname
\fi
#2}}

\bibitem{Oselio14}
B.~Oselio, A.~Kulesza, and A.~O. Hero, ``Multi-layer graph analysis for dynamic
  social networks,'' \emph{{IEEE} J. Sel. Topics Signal Process.}, vol.~8,
  no.~4, pp. 514--523, Aug 2014.

\bibitem{xu2014dynamic}
K.~S. Xu and A.~O. Hero, ``Dynamic stochastic blockmodels for time-evolving
  social networks,'' \emph{{IEEE} J. Sel. Topics Signal Process.}, vol.~8,
  no.~4, pp. 552--562, 2014.

\bibitem{Domenico13tensor}
M.~De~Domenico, A.~Sol\'e-Ribalta, E.~Cozzo, M.~Kivel\"a, Y.~Moreno, M.~A.
  Porter, S.~G\'omez, and A.~Arenas, ``Mathematical formulation of multilayer
  networks,'' \emph{Phys. Rev. X}, vol.~3, p. 041022, Dec 2013.

\bibitem{oselio2015information}
B.~Oselio, A.~Kulesza, and A.~Hero, ``Information extraction from large
  multi-layer social networks,'' in \emph{IEEE International Conference on
  Acoustics, Speech and Signal Processing (ICASSP)}, 2015, pp. 5451--5455.

\bibitem{zhou2007spectral}
D.~Zhou and C.~J. Burges, ``Spectral clustering and transductive learning with
  multiple views,'' in \emph{International Conference on Machine Learning},
  2007, pp. 1159--1166.

\bibitem{leonardi2013tight}
N.~Leonardi and D.~Van De~Ville, ``Tight wavelet frames on multislice graphs,''
  \emph{{IEEE} Trans. Signal Process.}, vol.~61, no.~13, pp. 3357--3367, 2013.

\bibitem{benzi2016principal}
K.~Benzi, B.~Ricaud, and P.~Vandergheynst, ``Principal patterns on graphs:
  Discovering coherent structures in datasets,'' \emph{{IEEE} Trans. Signal
  Inf. Process. Netw.}, vol.~2, no.~2, pp. 160--173, 2016.

\bibitem{CPY16ICASSP}
P.-Y. Chen, S.~Choudhury, and A.~O. Hero, ``Multi-centrality graph spectral
  decompositions and their application to cyber intrusion detection,'' in
  \emph{IEEE International Conference on Acoustics, Speech and Signal
  Processing (ICASSP)}, 2016, pp. 4553--4557.

\bibitem{park2013anomaly}
Y.~Park, C.~E. Priebe, and A.~Youssef, ``Anomaly detection in time series of
  graphs using fusion of graph invariants,'' \emph{{IEEE} J. Sel. Topics Signal
  Process.}, vol.~7, no.~1, pp. 67--75, 2013.

\bibitem{kivela2014multilayer}
M.~Kivel{\"a}, A.~Arenas, M.~Barthelemy, J.~P. Gleeson, Y.~Moreno, and M.~A.
  Porter, ``Multilayer networks,'' \emph{Journal of complex networks}, vol.~2,
  no.~3, pp. 203--271, 2014.

\bibitem{kim2015community}
J.~Kim and J.-G. Lee, ``Community detection in multi-layer graphs: A survey,''
  \emph{ACM SIGMOD Record}, vol.~44, no.~3, pp. 37--48, 2015.

\bibitem{mucha2010community}
P.~J. Mucha, T.~Richardson, K.~Macon, M.~A. Porter, and J.-P. Onnela,
  ``Community structure in time-dependent, multiscale, and multiplex
  networks,'' \emph{Science}, vol. 328, no. 5980, pp. 876--878, 2010.

\bibitem{dong2014clustering}
X.~Dong, P.~Frossard, P.~Vandergheynst, and N.~Nefedov, ``Clustering on
  multi-layer graphs via subspace analysis on grassmann manifolds,''
  \emph{{IEEE} Trans. Signal Process.}, vol.~62, no.~4, pp. 905--918, 2014.

\bibitem{cai2005community}
D.~Cai, Z.~Shao, X.~He, X.~Yan, and J.~Han, ``Community mining from
  multi-relational networks,'' in \emph{European Conference on Principles of
  Data Mining and Knowledge Discovery}.\hskip 1em plus 0.5em minus 0.4em\relax
  Springer, 2005, pp. 445--452.

\bibitem{tang2009uncoverning}
L.~Tang, X.~Wang, and H.~Liu, ``Uncoverning groups via heterogeneous
  interaction analysis,'' in \emph{IEEE International Conference on Data
  Mining}.\hskip 1em plus 0.5em minus 0.4em\relax IEEE, 2009, pp. 503--512.

\bibitem{wu2015discovering}
Z.~Wu, Z.~Bu, J.~Cao, and Y.~Zhuang, ``Discovering communities in
  multi-relational networks,'' in \emph{User Community Discovery}.\hskip 1em
  plus 0.5em minus 0.4em\relax Springer, 2015, pp. 75--95.

\bibitem{tang2012community}
L.~Tang, X.~Wang, and H.~Liu, ``Community detection via heterogeneous
  interaction analysis,'' \emph{Data Mining and Knowledge Discovery}, vol.~25,
  no.~1, pp. 1--33, 2012.

\bibitem{de2015structural}
M.~De~Domenico, V.~Nicosia, A.~Arenas, and V.~Latora, ``Structural reducibility
  of multilayer networks,'' \emph{Nature Communications}, vol.~6, 2015.

\bibitem{Taylor16}
D.~Taylor, S.~Shai, N.~Stanley, and P.~J. Mucha, ``Enhanced detectability of
  community structure in multilayer networks through layer aggregation,''
  \emph{Phys. Rev. Lett.}, vol. 116, p. 228301, Jun 2016.

\bibitem{kim2016differential}
J.~Kim, J.-g. Lee, and S.~Lim, ``Differential flattening: A novel framework for
  community detection in multi-layer graphs,'' \emph{ACM Transactions on
  Intelligent Systems and Technology (TIST)}, vol.~8, no.~2, p.~27, 2016.

\bibitem{Holland83}
P.~W. Holland, K.~B. Laskey, and S.~Leinhardt, ``Stochastic blockmodels: First
  steps,'' \emph{Social Networks}, vol.~5, no.~2, pp. 109--137, 1983.

\bibitem{han2015consistent}
Q.~Han, K.~Xu, and E.~Airoldi, ``Consistent estimation of dynamic and
  multi-layer block models,'' in \emph{International Conference on Machine
  Learning}, 2015, pp. 1511--1520.

\bibitem{paul2015community}
S.~Paul and Y.~Chen, ``Community detection in multi-relational data with
  restricted multi-layer stochastic blockmodel,'' \emph{arXiv preprint
  arXiv:1506.02699}, 2015.

\bibitem{barbillon2016stochastic}
P.~Barbillon, S.~Donnet, E.~Lazega, and A.~Bar-Hen, ``Stochastic block models
  for multiplex networks: an application to a multilevel network of
  researchers,'' \emph{Journal of the Royal Statistical Society: Series A
  (Statistics in Society)}, 2016.

\bibitem{PhysRevX.6.011036}
T.~Vall\`es-Catal\`a, F.~A. Massucci, R.~Guimer\`a, and M.~Sales-Pardo,
  ``Multilayer stochastic block models reveal the multilayer structure of
  complex networks,'' \emph{Phys. Rev. X}, vol.~6, p. 011036, Mar 2016.

\bibitem{Stanley16}
N.~Stanley, S.~Shai, D.~Taylor, and P.~J. Mucha, ``Clustering network layers
  with the strata multilayer stochastic block model,'' \emph{IEEE Transactions
  on Network Science and Engineering}, vol.~3, no.~2, pp. 95--105, Apr 2016.

\bibitem{papalexakis2013more}
E.~E. Papalexakis, L.~Akoglu, and D.~Ience, ``Do more views of a graph help?
  community detection and clustering in multi-graphs,'' in \emph{International
  Conference on Information Fusion}.\hskip 1em plus 0.5em minus 0.4em\relax
  IEEE, 2013, pp. 899--905.

\bibitem{PhysRevE.92.042806}
J.~Iacovacci, Z.~Wu, and G.~Bianconi, ``Mesoscopic structures reveal the
  network between the layers of multiplex data sets,'' \emph{Phys. Rev. E},
  vol.~92, p. 042806, Oct 2015.

\bibitem{greene2013producing}
D.~Greene and P.~Cunningham, ``Producing a unified graph representation from
  multiple social network views,'' in \emph{ACM Web Science Conference}, 2013,
  pp. 118--121.

\bibitem{ni2015flexible}
J.~Ni, H.~Tong, W.~Fan, and X.~Zhang, ``Flexible and robust multi-network
  clustering,'' in \emph{ACM SIGKDD International Conference on Knowledge
  Discovery and Data Mining}.\hskip 1em plus 0.5em minus 0.4em\relax ACM, 2015,
  pp. 835--844.

\bibitem{PhysRevX.5.011027}
M.~De~Domenico, A.~Lancichinetti, A.~Arenas, and M.~Rosvall, ``Identifying
  modular flows on multilayer networks reveals highly overlapping organization
  in interconnected systems,'' \emph{Phys. Rev. X}, vol.~5, p. 011027, Mar
  2015.

\bibitem{tang2009clustering}
W.~Tang, Z.~Lu, and I.~S. Dhillon, ``Clustering with multiple graphs,'' in
  \emph{IEEE International Conference on Data Mining}.\hskip 1em plus 0.5em
  minus 0.4em\relax IEEE, 2009, pp. 1016--1021.

\bibitem{kuncheva2015community}
Z.~Kuncheva and G.~Montana, ``Community detection in multiplex networks using
  locally adaptive random walks,'' in \emph{IEEE/ACM International Conference
  on Advances in Social Networks Analysis and Mining}.\hskip 1em plus 0.5em
  minus 0.4em\relax ACM, 2015, pp. 1308--1315.

\bibitem{dong2012clustering}
X.~Dong, P.~Frossard, P.~Vandergheynst, and N.~Nefedov, ``Clustering with
  multi-layer graphs: A spectral perspective,'' \emph{{IEEE} Trans. Signal
  Process.}, vol.~60, no.~11, pp. 5820--5831, 2012.

\bibitem{boden2012mining}
B.~Boden, S.~G{\"u}nnemann, H.~Hoffmann, and T.~Seidl, ``Mining coherent
  subgraphs in multi-layer graphs with edge labels,'' in \emph{ACM SIGKDD
  international conference on knowledge discovery and data mining}, 2012, pp.
  1258--1266.

\bibitem{zelnik2004self}
L.~Zelnik-Manor and P.~Perona, ``Self-tuning spectral clustering,'' in
  \emph{Advances in neural information processing systems (NIPS)}, 2004, pp.
  1601--1608.

\bibitem{blondel2008fast}
V.~D. Blondel, J.-L. Guillaume, R.~Lambiotte, and E.~Lefebvre, ``Fast unfolding
  of communities in large networks,'' \emph{Journal of Statistical Mechanics:
  Theory and Experiment}, no.~10, 2008.

\bibitem{Krzakala2013}
F.~Krzakala, C.~Moore, E.~Mossel, J.~Neeman, A.~Sly, L.~Zdeborova, and
  P.~Zhang, ``Spectral redemption in clustering sparse networks,'' \emph{Proc.
  National Academy of Sciences}, vol. 110, pp. 20\,935--20\,940, 2013.

\bibitem{CPY16AMOS}
P.-Y. Chen and A.~O. Hero, ``Phase transitions and a model order selection
  criterion for spectral graph clustering,'' \emph{arXiv preprint
  arXiv:1604.03159}, 2016.

\bibitem{Luxburg07}
U.~Luxburg, ``A tutorial on spectral clustering,'' \emph{Statistics and
  Computing}, vol.~17, no.~4, pp. 395--416, Dec. 2007.

\bibitem{hartigan1979algorithm}
J.~A. Hartigan and M.~A. Wong, ``A k-means clusterin algorithm,'' \emph{Applied
  statistics}, pp. 100--108, 1979.

\bibitem{Fiedler73}
M.~Fiedler, ``Algebraic connectivity of graphs,'' \emph{Czechoslovak
  Mathematical Journal}, vol.~23, no.~98, pp. 298--305, 1973.

\bibitem{jennings1992matrix}
A.~Jennings and J.~J. McKeown, \emph{Matrix computation}.\hskip 1em plus 0.5em
  minus 0.4em\relax John Wiley \& Sons Inc, 1992.

\bibitem{Tropp_matrix_concentrate}
\BIBentryALTinterwordspacing
J.~A. Tropp, ``An introduction to matrix concentration inequalities,''
  \emph{Foundations and Trends in Machine Learning}, vol.~8, no. 1-2, pp.
  1--230, 2015. [Online]. Available: \url{http://dx.doi.org/10.1561/2200000048}
\BIBentrySTDinterwordspacing

\bibitem{potthoff1966testing}
R.~F. Potthoff and M.~Whittinghill, ``Testing for homogeneity: I. the binomial
  and multinomial distributions,'' \emph{Biometrika}, vol.~53, no. 1-2, pp.
  167--182, 1966.

\bibitem{CPY_16KDDMLG}
P.-Y. Chen, B.~Zhang, M.~A. Hasan, and A.~O. Hero, ``Incremental method for
  spectral clustering of increasing orders,'' in \emph{ACM International
  Conference on Knowledge Discovery and Data Mining (KDD) Workshop on Mining
  and Learning with Graphs}, 2016, arXiv preprint arXiv:1512.07349.

\bibitem{wu2016primme_svds}
L.~Wu, E.~Romero, and A.~Stathopoulos, ``Primme\_svds: A high-performance
  preconditioned svd solver for accurate large-scale computations,'' \emph{SIAM
  Journal on Scientific Computing, accepted. ArXiv preprint arXiv:1607.01404},
  2016.

\bibitem{Zaki.Jr:14}
M.~J. Zaki and W.~M. Jr, \emph{Data Mining and Analysis: Fundamental Concepts
  and Algorithms}.\hskip 1em plus 0.5em minus 0.4em\relax Cambridge University
  Press, 2014.

\bibitem{vickers1981representing}
\BIBentryALTinterwordspacing
M.~Vickers and S.~Chan, ``Representing classroom social structure,''
  \emph{Victoria Institute of Secondary Education, Melbourne}, 1981. [Online].
  Available: \url{http://deim.urv.cat/~manlio.dedomenico/data.php}
\BIBentrySTDinterwordspacing

\bibitem{tang2008arnetminer}
J.~Tang, J.~Zhang, L.~Yao, J.~Li, L.~Zhang, and Z.~Su, ``Arnetminer: extraction
  and mining of academic social networks,'' in \emph{ACM SIGKDD international
  conference on Knowledge discovery and data mining}, 2008, pp. 990--998.

\bibitem{Zhang.Saha.ea:14}
B.~Zhang, T.~K. Saha, and M.~Al~Hasan, ``Name disambiguation from link data in
  a collaboration graph,'' in \emph{IEEE/ACM International Conference on
  Advances in Social Networks Analysis and Mining (ASONAM)}, 2014, pp. 81--84.

\bibitem{Saha.Zhang.ea:15}
T.~K. Saha, B.~Zhang, and M.~Al~Hasan, ``Name disambiguation from link data in
  a collaboration graph using temporal and topological features,'' \emph{Social
  Network Analysis and Mining}, vol.~5, pp. 1--14, 2015.

\bibitem{pentland2009inferring}
\BIBentryALTinterwordspacing
A.~Pentland, N.~Eagle, and D.~Lazer, ``Inferring social network structure using
  mobile phone data,'' \emph{Proceedings of the National Academy of Sciences
  (PNAS)}, vol. 106, no.~36, pp. 15\,274--15\,278, 2009. [Online]. Available:
  \url{http://realitycommons.media.mit.edu}
\BIBentrySTDinterwordspacing

\bibitem{de2014navigability}
\BIBentryALTinterwordspacing
M.~De~Domenico, A.~Sol{\'e}-Ribalta, S.~G{\'o}mez, and A.~Arenas,
  ``Navigability of interconnected networks under random failures,''
  \emph{Proceedings of the National Academy of Sciences (PNAS)}, vol. 111,
  no.~23, pp. 8351--8356, 2014. [Online]. Available:
  \url{http://deim.urv.cat/~manlio.dedomenico/data.php}
\BIBentrySTDinterwordspacing

\bibitem{de2014muxviz}
M.~De~Domenico, M.~A. Porter, and A.~Arenas, ``Muxviz: a tool for multilayer
  analysis and visualization of networks,'' \emph{Journal of Complex Networks},
  p. cnu038, 2014.

\bibitem{strehl2002cluster}
A.~Strehl and J.~Ghosh, ``Cluster ensembles-a knowledge reuse framework for
  combining multiple partitions,'' \emph{Journal of Machine Learning Research},
  vol.~3, no. Dec, pp. 583--617, 2002.

\bibitem{rand1971objective}
W.~M. Rand, ``Objective criteria for the evaluation of clustering methods,''
  \emph{Journal of the American Statistical association}, vol.~66, no. 336, pp.
  846--850, 1971.

\bibitem{Rijsbergen:1979:IR:539927}
C.~J.~V. Rijsbergen, \emph{Information Retrieval}, 2nd~ed.\hskip 1em plus 0.5em
  minus 0.4em\relax Newton, MA, USA: Butterworth-Heinemann, 1979.

\bibitem{Shi00}
J.~Shi and J.~Malik, ``Normalized cuts and image segmentation,'' \emph{{IEEE}
  Trans. Pattern Anal. Mach. Intell.}, vol.~22, no.~8, pp. 888--905, 2000.

\bibitem{Davis70}
W.~M.~K. Chandler~Davis, ``The rotation of eigenvectors by a perturbation.
  iii,'' \emph{SIAM Journal on Numerical Analysis}, vol.~7, no.~1, pp. 1--46,
  1970.

\bibitem{anscombe1948transformation}
F.~J. Anscombe, ``The transformation of poisson, binomial and negative-binomial
  data,'' \emph{Biometrika}, vol.~35, no. 3/4, pp. 246--254, 1948.

\bibitem{chang2000generalized}
Y.-P. Chang and W.-T. Huang, ``Generalized confidence intervals for the largest
  value of some functions of parameters under normality,'' \emph{Statistica
  Sinica}, pp. 1369--1383, 2000.

\end{thebibliography}

\clearpage
\setcounter{equation}{0}
\setcounter{figure}{0}
\setcounter{table}{0}
\setcounter{page}{1}
\makeatletter
\renewcommand{\theequation}{S\arabic{equation}}
\renewcommand{\thefigure}{S\arabic{figure}}

\section*{{\LARGE Supplementary Material} for Multilayer Spectral Graph Clustering via Convex Layer Aggregation: \\Theory and Algorithms\\
	Pin-Yu Chen and Alfred O. Hero}
\label{sec_appendix}
\appendices
\subsection{Proof of Theorem \ref{thm_impossible_ML}}
\label{appen_unsuccess_ML}
Given a layer weight vector $\bw \in \cW_L$, using (\ref{eqn_network_model_multilayer_weight}) the graph Laplacian matrix $\bLw$ of the graph $G^\bw$ via convex layer aggregation can be written in the block representation such that its ($i,j$)-th block of dimension $n_i \times n_j$, denoted by $\bB^\bw_{ij}$, satisfies 
\begin{align}
\label{eqn_Laplacian_multi_simple_ML}
\bB^\bw_{ij}=
\left\{
\begin{array}{ll}
\bLw_i+\sum_{z=1,~z \neq i}^K \bD^\bw_{iz}   , & \hbox{if}~i=j, \\
-\bFw_{ij}, & \hbox{if}~i \neq j,
\end{array}
\right.
\end{align}
for $1 \leq i,j \leq K$, where  $\bD^\bw_{ij}= \textnormal{diag}(\sum_{\ell=1}^{L} w_\ell \bFl_{ij}\bone_{n_j})$ is the diagonal nodal strength matrix contributed by the inter-cluster edges between clusters $i$ and $j$ of the graph $G^\bw$, and $\bFw_{ij}=\sum_{\ell=1}^{L} w_\ell \bFl_{ij}$. 

Applying the block representation in (\ref{eqn_Laplacian_multi_simple_ML}) to the minimization problem in (\ref{eqn_spectral_clustering_multi_formulation_ML}), 
let $\bnu \in \mathbb{R}^{(K-1)}$ and $\bU \in \mathbb{R}^{(K-1) \times (K-1)}$ with $\bU=\bU^T$ be the Lagrange multiplier of the constraints $\bX^T \bone_n=\bzero_{K-1}$ and $\bX^T \bX= \bI_{K-1}$, respectively.  The Lagrangian function is
\begin{align}
\label{eqn_Lagrangian_multi_ML}
\Gamma(\bX)&=\trace(\bX^T \bLw \bX)-\bnu^T \bX^T \bone_n \nonumber \\
&~~~ - \trace \lb \bU (\bX^T \bX-\bI_{K-1}) \rb.
\end{align}
Let $\bY \in \mathbb{R}^{n \times (K-1)}$ be the solution of (\ref{eqn_spectral_clustering_multi_formulation_ML}).
Differentiating (\ref{eqn_Lagrangian_multi_ML}) with respect to $\bX$ and substituting $\bY$ into the equations, we obtain the optimality condition
\begin{align}
\label{eqn_Lagrangian_multi_nu_ML}
2 \bLw \bY - \bone_n \bnu^T - 2 \bY \bU = \bO,
\end{align}
where $\bO$ is a matrix of zero entries.
Left multiplying (\ref{eqn_Lagrangian_multi_nu_ML}) by $\bone_n^T$, we obtain
\begin{align}
\label{eqn_Lagrangian_multi_nu_2_ML}
\bnu=\bzero_{K-1}.
\end{align}
Left multiplying (\ref{eqn_Lagrangian_multi_nu_ML}) by $\bY^T$ and using (\ref{eqn_Lagrangian_multi_nu_2_ML}), we have
\begin{align}
\label{eqn_Lagrangian_multi_mu_ML}
\bU = \bY^T \bLw \bY=\diag(\lambda_2(\bLw),\lambda_3(\bLw),\ldots,\lambda_K(\bLw)),
\end{align}
which we denote by the diagonal matrix $\bLambda$.
Therefore, by (\ref{eqn_spectral_clustering_multi_formulation_ML}) we have
\begin{align}
\label{eqn_Lagrangian_multi_mu_2_ML}
\SK(\bLw)=\trace (\bU).
\end{align}

Now let $\bX=[\bX_1^T,\bX_2^T,\ldots,\bX_K^T]^T$ and $\bY=[\bY_1^T,\bY_2^T,\ldots,\bY_K^T]^T$, where $\bX_k \in \mathbb{R}^{{n_k} \times (K-1)}$ and $\bY_k \in \mathbb{R}^{{n_k} \times (K-1)}$. With  (\ref{eqn_Lagrangian_multi_mu_ML}), the Lagrangian function in (\ref{eqn_Lagrangian_multi_ML}) can be written as
\begin{align}
\label{eqn_Lagrangian_multi_K_ML}
\Gamma(\bX)&=\sum_{k=1}^{K} \trace( \bX_k^T \bLw_k \bX_k) + \sum_{k=1}^{K} \sum_{j=1,j \neq k}^{K} \trace(\bX_k^T  \bD^\bw_{kj} \bX_k) \nonumber \\
&~~~-\sum_{k=1}^{K} \sum_{j=1,j \neq k}^{K} \trace(\bX_k^T  \bFw_{kj} \bX_j)-\sum_{k=1}^{K} \trace(\bU \bX_k^T \bX_k) \nonumber \\
&~~~+\trace(\bU).
\end{align}
Differentiating (\ref{eqn_Lagrangian_multi_K_ML}) with respect to $\bX_k$ and substituting $\bY_k$ into the equation, we obtain the optimality condition that for all $k \in \{1,2,\ldots,K\}$,
\begin{align}
\label{eqn_Lagrangian_multi_K_diff_ML}
\bLw_k \bY_k + \sum_{j=1,j \neq k}^{K} \bD^\bw_{kj} \bY_k-\sum_{j=1,j \neq k}^{K} \bFw_{kj} \bY_j- \bY_k \bU=\bO.
\end{align}

Using  the  bounded fourth moment assumption 
for $\bFl_{ij}$, it has been proved in \cite{CPY16AMOS} that 
\begin{align}
\label{eqn_C_multi_conv_ML}
\frac{\bFl_{ij}}{\sqrt{n_i n_j}} \asconv \tijl \frac{\bone_{n_i} \bone_{n_j}^T}{\sqrt{n_i n_j}}
\end{align}
as $n_i,n_j \ra \infty$ and $\frac{\nmin}{\nmax} \ra c>0$, where $\asconv$ denotes almost sure convergence in the spectral norm\footnotemark[1], we have
\begin{align}
\label{eqn_C_multi_conv_ML_2}
\frac{\bFw_{ij}}{\sqrt{n_i n_j}}=\frac{\sum_{\ell}^{L} w_\ell \bFl_{ij}}{\sqrt{n_i n_j}} \asconv \sum_{\ell}^{L} w_\ell \tijl \frac{\bone_{n_i} \bone_{n_j}^T}{\sqrt{n_i n_j}}
\end{align}
and 
\begin{align}
\label{eqn_Degree_matrix_concentrate_ML}
\frac{\bD^\bw_{ij}}{n_j}= \frac{\textnormal{diag}(\sum_{\ell}^L w_\ell
	\bFl_{ij} \bone_{n_j})}{n_j} \asconv \sum_{\ell}^{L} w_\ell \tijl \bI.
\end{align}

Using (\ref{eqn_Degree_matrix_concentrate_ML}) and left multiplying (\ref{eqn_Lagrangian_multi_K_diff_ML}) by $\frac{\bone_{n_k}^T}{n}$ gives
\begin{align}
\label{eqn_Lagrangian_multi_K_one_ML}
&\frac{1}{n} \Lb \sum_{\ell=1}^{L}\sum_{j=1,j \neq k}^{K} n_j w_\ell t_{kj}^{(\ell)} \bone_{n_k}^T \bY_k- \sum_{\ell=1}^{L} \sum_{j=1,j \neq k}^{K} n_k w_\ell t_{kj}^{(\ell)} \bone_{n_j}^T\bY_j \right. \nonumber \\
& \left.- \bone_{n_k}^T\bY_k \bU \Rb 
\asconv \bzero_{K-1}^T,~\forall~k \in \{1,\ldots,K\}.
\end{align}

Using the centrality relation $\bone_{n_K}^T \bY_K=-\sum_{j=1}^{K-1} \bone_{n_j}^T \bY_{j}$ and (\ref{eqn_Lagrangian_multi_mu_2_ML}), 
(\ref{eqn_Lagrangian_multi_K_one_ML}) can be represented as an asymptotic form of  Sylvester's equation 
\begin{align}
\label{eqn_Sylvester_form_ML}
\frac{1}{n} \lb \bWt^\bw \bZ - \bZ \bLambda \rb \asconv \bO,
\end{align}
where $\bZ=[ \bY_{1}^T\bone_{n_1}, \bY_{2}^T \bone_{n_2},\ldots, \bY_{{K-1}}^T \bone_{n_{K-1}} ]^T \in \mathbb{R}^{(K-1) \times (K-1)}$ and $\bWt^\bw$ is the matrix defined in Theorem \ref{thm_impossible_ML}.

Let $\otimes$ denote the Kronecker product and let $\vectorize (\bZ) $ denote the vectorization operation of $\bZ$ by stacking the columns of $\bZ$ into a column vector. Then (\ref{eqn_Sylvester_form_ML}) can be represented as  
\begin{align}
\label{eqn_Sylvester_vector_form_ML}
\frac{1}{n}(\bI_{K-1} \otimes \bWt^\bw - \bLambda \otimes \bI_{K-1}) \vectorize(\bZ) \asconv \bzero,
\end{align}
where the matrix $\bI_{K-1} \otimes \bWt^\bw - \bLambda \otimes \bI_{K-1}$ is the Kronecker sum, denoted by $\bWt^\bw \oplus -\bLambda$.
Observe that $\vectorize(\bZ) = \bzero$ is always a trivial solution  to (\ref{eqn_Sylvester_vector_form_ML}), and 
if $\bWt^\bw \oplus -\bLambda$ is non-singular, $\vectorize(\bZ) = \bzero$ is the unique solution to (\ref{eqn_Sylvester_vector_form_ML}). Since  $\vectorize(\bZ) = \bzero$ and $\sum_{k=1}^K \bone_{n_k}^T \bY_{k}=\bzero^T_{K-1}$ imply $\bone_{n_k}^T \bY_{k}= \bzero^T_{K-1}$ for all $k=1,2,\ldots,K$, the centroid $\frac{\bone_{n_k}^T \bY_{k}}{n_k}$ of each cluster in the eigenspace is a zero vector, the clusters are not perfectly separable, and therefore correct clustering is not possible. Therefore, a sufficient condition for multilayer SGC with layer weight vector $\bw$ to fail is that the matrix $\bI_{K-1} \otimes \bWt^\bw - \bLambda \otimes \bI_{K-1}$ be non-singular. Moreover, using the property of the Kronecker sum that the eigenvalues of $\bWt^\bw \oplus -\bLambda$ satisfy
$\{ \lambda_\ell (\bWt^\bw \oplus -\bLambda) \}_{z=1}^{(K-1)^2}=\{\lambda_i(\bWt^\bw)-\lambda_j(\bLambda)\}_{i,j=1}^{K-1}$, the sufficient condition on failure of multilayer SGC is that for every $\bw \in \cW_L$, $\lambda_i\lb \frac{\bWt^\bw}{n} \rb \neq \lambda_j \lb \frac{\bL^\bw}{n} \rb$ for all $i = 1,2,\ldots,K-1$ and $j =2,3,\ldots,K$.

\subsection{Proof of Theorem  \ref{thm_spec_ML}}
\label{proof_thm_spec_ML}
Following the derivations in Appendix\ref{appen_unsuccess_ML}, since $ \bone_{n_k}^T \bY_k=-\sum_{j=1,j \neq k}^{K}\bone_{n_j}^T\bY_j$ by the centrality constraint, under the block-wise identical noise model (i.e., $\tijl=\tl$ for all $\ell=1,2,\ldots,L$),
the optimality condition in (\ref{eqn_Lagrangian_multi_K_one_ML}) can be simplified to
\begin{align}
\label{eqn_Lagrangian_multi_K_one_homo_ML}
\lb \tw \bI_{K-1} - \frac{\bU}{n} \rb \bY_k^T \bone_{n_k} 
\asconv \bzero_{K-1},~\forall~k,
\end{align}
where $\tw=\sum_{\ell=1}^L w_\ell \tl$ is the aggregated noise level given a layer weight vector $\bw$.
The optimality condition in (\ref{eqn_Lagrangian_multi_K_one_homo_ML})
implies that one of the two cases below has to hold:
\begin{align}
\label{eqn_case1_ML}
&\text{Case 1:~}  \frac{\bU}{n}  \asconv  \tw \bI_{K-1}; \\
\label{eqn_case2_ML}
&\text{Case 2:~}  \bY_k^T \bone_{n_k}  \asconv  \bzero_{K-1},~\forall~k.
\end{align}
Note that with (\ref{eqn_Lagrangian_multi_mu_2_ML}), Case 1 implies
\begin{align}
\label{eqn_partial_eig_sum_ML}
\frac{\SK(\bLw)}{n} =\frac{\trace(\bU)}{n}\asconv (K-1) \tw. 
\end{align}

Furthermore, in Case 1, left multiplying (\ref{eqn_Lagrangian_multi_K_diff_ML}) by $\frac{\bY_k^T}{n}$ and using (\ref{eqn_C_multi_conv_ML}) and (\ref{eqn_Degree_matrix_concentrate_ML}) gives
\begin{align}
\label{eqn_spec_multi_eigvector_ML}
&\frac{1}{n} \Lb \bY_k^T \bLw_k \bY_k + \sum_{j=1,j \neq k}^{K} n_j \tw \bY_k^T\bY_k \right. \nonumber \\
&~~~\left.  -\sum_{j=1,j \neq k}^{K} \tw \bY_k^T \bone_{n_k}\bone_{n_j}^T  \bY_j - \bY_k^T\bY_k \bU \Rb \asconv \bO,~\forall~k.
\end{align}
Since $ \bone_{n_k}^T \bY_k=-\sum_{j=1,j \neq k}^{K}\bone_{n_j}^T\bY_j$,
(\ref{eqn_spec_multi_eigvector_ML}) can be simplified as 
\begin{align}
\label{eqn_spec_multi_eigvector_2_ML}
&\frac{1}{n} \Lb \bY_k^T \bLw_k \bY_k + (n-n_k) \tw \bY_k^T\bY_k + 
\tw \bY_k^T \bone_{n_k}\bone_{n_k}^T  \bY_k \right. \nonumber \\
&~~~\left.-   \bY_k^T\bY_k \bU \Rb \asconv \bO,~\forall~k.
\end{align}
Taking the trace of (\ref{eqn_spec_multi_eigvector_2_ML}) and using (\ref{eqn_case1_ML}), we have
\begin{align}
\label{eqn_spec_multi_eigvector_trace_ML}
&\frac{1}{n} \Lb \trace(\bY_k^T \bLw_k \bY_k) \Rb + \frac{\tw}{n} \Lb \trace(\bY_k^T \bone_{n_k}\bone_{n_k}^T  \bY_k)  \right.  \nonumber \\ 
&~~~\left.  -n_k  \trace(\bY_k^T\bY_k)
\Rb \asconv 0,~\forall~k.
\end{align}
Since (\ref{eqn_spec_multi_eigvector_trace_ML}) has to be satisfied for all values of $\tw$ in Case 1, this implies the following two conditions have to hold simultaneously:
\begin{align}
\label{eqn_spec_multi_eigvector_trace_2_ML}
\left.
\begin{array}{ll}
\frac{1}{n} \Lb \trace(\bY_k^T \bLw_k \bY_k) \Rb \asconv 0,~\forall~k;  \\
\frac{1}{n} \Lb \trace(\bY_k^T \bone_{n_k}\bone_{n_k}^T  \bY_k)-n_k \trace(\bY_k^T\bY_k) 
\Rb  \asconv 0,~\forall~k.
\end{array}
\right.
\end{align}
Since $\bLw_k=\sum_{\ell=1}^{L} w_\ell \bLl_k$ is a positive semidefinite (PSD) matrix, $\bLw_k \bone_{n_k}=\bzero_{n_k}$, and $\lambda_2(\bLw_k)>0$, $\frac{1}{n} \Lb \trace(\bY_k^T \bLw_k \bY_k) \Rb \asconv 0$ implies that every column of $\bLw_k$ is a constant vector. Therefore, (\ref{eqn_spec_multi_eigvector_trace_2_ML}) implies that in Case 1,
\begin{align}
\label{eqn_spec_multi_eigenvec_conv_ML}
\bY_k \asconv \bone_{n_k} \bone_{K-1}^T \bV_k=\Lb v^k_1 \bone_{n_k},v^k_2 \bone_{n_k},\ldots,v^k_{K-1} \bone_{n_k} \Rb, 
\end{align}
where $\bV=\diag(v_1^k,v_2^k,\ldots,v_{K-1}^k)$ is a diagonal matrix.

To prove the phase transition results in Theorem \ref{thm_spec_ML} (a),
let $\cS=\{\bX \in \mathbb{R}^{n \times (K-1)}:~\bX^T \bX= \bI_{K-1},~\bX^T \bone_n=\bzero_{K-1}\}$.
In Case 2, since $\bY_k^T \bone_{n_k}  \asconv  \bzero_{K-1}~\forall~k$ from (\ref{eqn_case2_ML}), we have 
\begin{align}
\label{eqn_S2K_lower_ML}
&\frac{\SK(\bLw)}{n} 
\asconv \min_{\bX \in \cS}
\LB \frac{1}{n} \Lb \sum_{k=1}^K \trace(\bX_k^T \bLw_k \bX_k) \right. \right. \nonumber \\
&~~\left. \left. + \tw \sum_{k=1}^K (n-n_k) \trace(\bX_k^T\bX_k)  \Rb \RB \\
& \geq \min_{\bX \in \cS}
\LB \frac{1}{n} \sum_{k=1}^K \trace(\bX_k^T \bLw_k \bX_k)  \RB  \nonumber \\
&~~~+\min_{\bX \in \cS}
\LB \frac{\tw}{n}  \sum_{k=1}^K (n-n_k) \trace(\bX_k^T\bX_k) \RB \\
& = \min_{k \in \{1,2,\ldots,K\}} \LB \frac{\SK(\bLw_k)}{n}  \RB
+ \frac{(K-1)\tw}{n} \min_{k \in \{1,2,\ldots,K\}}(n-n_k) \\
\label{eqn_S2K_lower_4_ML}
& = \min_{k \in \{1,2,\ldots,K\}} \LB \frac{\SK(\bLw_k)}{n}  \RB 
+ \frac{(K-1)(n-\nmax)\tw}{n}, 
\end{align}
where $n_{\max}=\max_{k \in \{1,2,\ldots,K\}}n_k$.

Similarly, let $\cS_k=\{\bX \in \mathbb{R}^{n \times (K-1)}:~\bX_k^T \bX_k= \bI_{K-1},~\bX_j=\bO_{n_j \times (K-1)}~\forall~j \neq k,~\bX^T \bone_n=\bzero_{K-1}\}$. Since $\cS_k \subseteq \cS$, in Case 2, we have
\begin{align}
\label{eqn_S2K_upper_ML}
&\frac{\SK(\bLw)}{n} 
\asconv \min_{\bX \in \cS}
\LB \frac{1}{n} \Lb \sum_{k=1}^K \trace(\bX_k^T \bLw_k \bX_k) \right. \right.  \nonumber \\
&~~\left. \left.+ \tw \sum_{k=1}^K (n-n_k) \trace(\bX_k^T\bX_k)  \Rb \RB \\
& \leq \min_{k \in \{1,2,\ldots,K\}} \min_{\bX \in \cS_k}
\LB \frac{1}{n} \Lb \sum_{k=1}^K \trace(\bX_k^T \bLw_k \bX_k) \right. \right.  \nonumber \\
&~~\left. \left.+ \tw \sum_{k=1}^K (n-n_k) \trace(\bX_k^T\bX_k)  \Rb \RB \\
&=\min_{k \in \{1,2,\ldots,K\}} 
\LB \frac{1}{n} \Lb \SK(\bLw_k) + (K-1) (n-n_k) \tw  \Rb \RB \\
&\leq \min_{k \in \{1,2,\ldots,K\}} 
\LB \frac{1}{n} \Lb \SK(\bLw_k) + (K-1) (n-\nmin) \tw  \Rb \RB \\
&=\min_{k \in \{1,2,\ldots,K\}} 
\LB \frac{\SK(\bLw_k)}{n}   \RB +\frac{(K-1)(n-\nmin)\tw}{n},
\label{eqn_S2K_upper_4_ML}
\end{align}
where $n_{\min}=\min_{k \in \{1,2,\ldots,K\}}n_k$. Therefore, we obtain the phase transition results in Theorem \ref{thm_spec_ML} (a). The visual illustration of Theorem \ref{thm_spec_ML} (a) is displayed in Fig. \ref{Fig_slope}.

\begin{figure}[!t]
	\centering
	\includegraphics[width=3.5in]{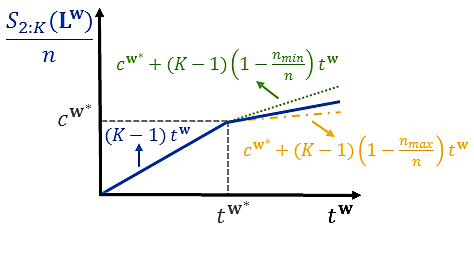}
	\caption{Visual illustration of Theorem \ref{thm_spec_ML} (a).}
	\label{Fig_slope}
\end{figure}

Proceeding to Theorem \ref{thm_spec_ML} (b), we first note that each cluster-wise eigenvector component $\bY_k$ in $\bY$ has to either satisfy the  cluster-wise separability in  (\ref{eqn_spec_multi_eigenvec_conv_ML}) or the zero row-sum condition  in (\ref{eqn_case2_ML}). To show the conditions (b-1) to (b-3) in Theorem \ref{thm_spec_ML} (b), 
recall the eigenvector matrix $\bY=[\bY_1^T,\bY_2^T,\ldots,\bY_K^T]^T$, where $\bY_k$ is the $n_k \times (K-1)$ matrix with row vectors representing the nodes from cluster $k$.
Since $\bY^T \bY =\sum_{k=1}^K \bY_{k}^T \bY_{k} = \bI_{(K-1) \times (K-1)}$, $\bY^T \bone_n=\sum_{k=1}^K \bY_k^T \bone_{n_k}=\bzero_{K-1}$, and from (\ref{eqn_spec_multi_eigenvec_conv_ML}) when $\tw<\twstar$ the matrix $\bY_k \asconv \bone_{n_k} \bone_{K-1}^T \bV_k=\Lb v^k_1 \bone_{n_k},v^k_2 \bone_{n_k},\ldots,v^k_{K-1} \bone_{n_k} \Rb$ as $n_k \ra \infty$~$\forall~k$ and $\frac{\nmin}{\nmax} \ra c>0$, we have 
\begin{align}
\label{eqn_spec_multi_eigvec_conv_2_ML}
\left.
\begin{array}{ll}
\sum_{k=1}^K n_k \bvk \bvk^T= \bI_{K-1}; \\
\sum_{k=1}^K n_k \bvk = \bzero_{K-1},
\end{array}
\right.
\end{align}
where $\bvk=\bV_k \bone_{n_k}=[v_1^k,v_2^k,\ldots,v_{K-1}^k]^T$.
(\ref{eqn_spec_multi_eigvec_conv_2_ML}) suggests that some $\bvk$ cannot be a zero vector since $\sum_{k=1}^K n_k {(v^k_j)}^2=1$ for all $j \in\{1,2,\ldots,K-1\}$, and from (\ref{eqn_spec_multi_eigvec_conv_2_ML}) we have
\begin{align}
\label{eqn_spec_multi_eigvec_coeff_ML}
\left.
\begin{array}{ll}
\sum_{k:v^k_j>0} n_k v^k_j = - \sum_{k: v^k_j <0} n_k v^k_j,\\~~~\forall~j \in\{1,2,\ldots,K-1\}; \\
\sum_{k:v^k_i v^k_j>0} n_k v^k_i v^k_j = - \sum_{k: v^k_i v^k_j <0} n_k v^k_i v^k_j,\\~~~\forall~i,j \in\{1,2,\ldots,K-1\}, i \neq j.
\end{array}
\right.
\end{align}
As a results, the optimality conditions of $\bv_k$ in (\ref{eqn_spec_multi_eigvec_conv_2_ML}) and (\ref{eqn_spec_multi_eigvec_coeff_ML}) lead to the conditions (b-1) to (b-3) in Theorem \ref{thm_spec_ML}.

Lastly, comparing (\ref{eqn_partial_eig_sum_ML}) with (\ref{eqn_S2K_lower_4_ML}) and (\ref{eqn_S2K_upper_4_ML}), as a function of $\tw$ the slope of $\frac{\SK(\bLw)}{n} $ changes at some critical value $\twstar$ that separates Case 1 and Case 2. By the continuity of $\frac{\SK(\bLw)}{n} $, a lower bound on $\twstar$ is 
\begin{align}
\label{eqn_spec_multi_LB_ML}
\tLBw=\frac{\min_{k \in \{1,2,\ldots,K\}} \SK(\bLw_k)}{(K-1)\nmax}, 
\end{align}
and an upper bound on $\twstar$ is
\begin{align}
\label{eqn_spec_multi_UB_ML}
\tUBw=\frac{\min_{k \in \{1,2,\ldots,K\}} \SK(\bLw_k)}{(K-1)\nmin}.
\end{align}
In particular, if $c=1$, then $\nmax=\nmin=\frac{n}{K}$ and hence the expressions in  (\ref{eqn_S2K_lower_4_ML}) and (\ref{eqn_S2K_upper_4_ML}) are identical, which completes Theorem \ref{thm_spec_ML} (c).

\subsection{Proof of Theorem \ref{thm_principal_angle_ML}}
\label{proof_thm_principal_angle_ML}

The following lemma provides bounds on the smallest $K-1$ nonzero eigenvalues of $\bLw$ under the block-wise non-identical noise model.

\begin{lemma}
	\label{lem_early_breakdown_ML}
	Under the block-wise non-identical noise model in Sec. \ref{subsec_ML_signal_noise} with maximum noise level $\{ \tlmax \}_{\ell=1}^L$ for each layer,  given a layer weight vector $\bw \in \cW_L$, let 
	$\twmin=\sum_{\ell=1}^{L} w_\ell \min_{i \neq j}\tijl$,	$\twmax=\sum_{\ell=1}^{L} w_\ell \max_{i \neq j} \tijl$, and let $\twstar$ be the critical threshold value for the block-wise identical noise model specified by Theorem \ref{thm_spec_ML}.
	If $\twmax < \twstar$, the following statement holds almost surely as
	$n_k \ra \infty$~$\forall~k$ and $\frac{\nmin}{\nmax} \ra c >0$:
	\begin{align}
	\label{eqn_bound_inhom_ML}
	\twmin  \leq \lambda_j \lb {\frac{\bLw}{n}} \rb \leq \twmax,~\forall~ j=2,3,\ldots,K.
	\end{align}
\end{lemma}
\begin{proof}
	We first show that when $\twmax<\twstar$, the second eigenvalue of $\frac{\bLw}{n}$, $\lambda_2(\frac{\bLw}{n})$, lies within the interval $[\twmin,\twmax]$ almost surely as $n_k \ra \infty$~$\forall~k$ and $\frac{\nmin}{\nmax} \ra c >0$.	Under the block-wise non-identical noise model in Sec. \ref{subsec_ML_signal_noise}, by (\ref{eqn_C_multi_conv_ML}) with proper scaling the entries of each interconnection matrix $\bFl_{ij}$ converge to $\tijl$ almost surely as $n_k \ra \infty$~$\forall~k$ and $\frac{\nmin}{\nmax} \ra c >0$. Let $\bWw(\tw)$ be the weight matrix of the aggregated graph $G^\bw$ under the block-wise identical noise model with aggregated noise level $\tw$. Then the weight matrix $\bWw$ can be written as $\bWw=\bWw(\twmin)+ \bDelta \bWw$,
	and the corresponding graph Laplacian matrix can be written as $\bLw=\bLw(\twmin)+\bDelta \bLw$, where $\bLw(\twmin)$ and $\bDelta \bLw$ are associated with $\bWw(\tw)$ and $\bDelta \bWw$, respectively.
	Since $\twmin=\sum_{\ell=1}^{L} w_\ell \min_{i \neq j}\tijl$, as $n_k \ra \infty$~$\forall~k$ and $\frac{\nmin}{\nmax} \ra c >0$,
	$\frac{\bDelta \bWw}{n}$ is a symmetric nonnegative matrix almost surely, and $\frac{\bDelta \bLw}{n}$ is a graph Laplacian matrix almost surely. By the PSD property of a graph Laplacian matrix, we obtain $\lambda_2(\frac{\bLw}{n}) \geq \twmin$ almost surely as $n_k \ra \infty$~$\forall~k$ and $\frac{\nmin}{\nmax} \ra c >0$. Similarly, following the same procedure we can show that $\lambda_2(\frac{\bLw}{n}) \leq \twmax$ almost surely as $n_k \ra \infty$~$\forall~k$ and $\frac{\nmin}{\nmax} \ra c >0$. Lastly, when $\tw<\twstar$, using the fact from (\ref{eqn_case1_ML}) that $\lambda_j(\frac{\bLw(\tw)}{n}) \asconv \tw$ for all $j \in \{2,3,\ldots,K\}$, we obtain 
	\begin{align}
	\label{eqn_bound_inhom_2_ML}
	\twmin=\lambda_j \lb \frac{\bL(\twmin)}{n} \rb \leq \lambda_j \lb \frac{\bLw}{n} \rb \leq \lambda_j \lb \frac{\bL(\twmax)}{n} \rb=\twmax
	\end{align}
	almost surely  for all $j \in \{2,3,\ldots,K\}$ as $n_k \ra \infty$~$\forall~k$ and $\frac{\nmin}{\nmax} \ra c >0$.
\end{proof}

Proceeding to proving Theorem \ref{thm_principal_angle_ML}, 
applying the Davis-Kahan $\sin \theta$ theorem \cite{Davis70} to the eigenvector matrices $\bY$ and $\btY$ associated with the graph Laplacian matrices $\frac{\bLw}{n}$ and $\frac{\btLw}{n}$, respectively, we obtain an upper bound on the distance of column spaces spanned by $\bY$ and $\btY$, which is 
$\|\sin\mathbf{\Theta}(\bY,\btY\|_F \leq \frac{\| \bLw - \btLw \|_F} {n \delta}$, where  $\delta=\inf\{|x-y|: x \in \{0\} \cup [\lambda_{K+1}(\frac{\bLw}{n}),\infty),~y \in [\lambda_2(\frac{\btLw}{n}),\lambda_K(\frac{\btLw}{n}) ]\}$. Under the block-wise identical noise model, if $\tw<\twstar$, using the fact from (\ref{eqn_case1_ML}) that $\lambda_j(\frac{\btLw}{n}) \asconv \tw$ for all $j \in \{2,3,\ldots,K\}$ as $n_k \ra \infty$~$\forall~k$ and $\frac{\nmin}{\nmax} \ra c >0$, the interval 
$[\lambda_2(\frac{\btLw}{n}),\lambda_K(\frac{\btLw}{n})]$ reduces to a point $\tw$ almost surely. Therefore,
$\delta$ reduces to $\delta_{\tw}$ as defined in Theorem \ref{thm_principal_angle_ML}. Furthermore, if $\twmax \leq \twstar$, then (\ref{eqn_principal_angle_bound_ML}) holds for all $\tw \leq \twmax$. Taking the minimum over all upper bounds in (\ref{eqn_principal_angle_bound_ML}) for every $\tw \leq \twmax$, we obtain (\ref{eqn_principal_angle_bound_2_ML}).

\subsection{Details of clustering reliability test under the block-wise non-identical noise model}
\label{proof_equivalent_condition_anconsbe}

For each layer $\ell$, we use $\htmaxl$ to test the null hypothesis $H_0^{(\ell)}$: \emph{$ \tlmax  < \tLBw$} against the alternative hypothesis $H_1^{(\ell)}$: \emph{$ \tlmax  \geq \tLBw$}. The test accepts $H_0^{(\ell)}$ if the condition in (\ref{eqn_spectral_multi_confidence_interval_ingomogeneous_RIM_ML}) holds, and rejects $H_0^{(\ell)}$ otherwise. 
Using the Anscombe transformation on $\{\hpijl\}$ for variance stabilization \cite{anscombe1948transformation}, 
let $A_{ij}(x)=\sin^{-1} \sqrt{\frac{x+\frac{c^\prime}{\hn_i \hn_j}}{1+\frac{2c^\prime}{\hn_i \hn_j}}}$, where $c^\prime=\frac{3}{8}$.
By the central limit theorem,
$\sqrt{4\hn_i \hn_j+2} \cdot \lb A_{ij}(\hpijl)- A_{ij}(\pijl) \rb
\convd N(0,1)$
for all $\pijl \in (0,1)$ as $\hn_i,\hn_j \ra \infty$, 
where $\convd$ denotes convergence in distribution and $N(0,1)$ denotes the standard normal distribution \cite{anscombe1948transformation}.
Therefore, under the null hypothesis $H_0^{(\ell)}$, from \cite[Theorem 2.1]{chang2000generalized}
an asymptotic $100(1-\alpha^\prime)\%$ confidence interval for $\htmaxl$ is $[0,\psi_\ell]$,
where  $\psi(\alpha^\prime_{\ell},\{\htijl\})$ is a function of the precision parameter $\alpha^\prime_\ell \in [0,1]$ and $\{\htijl\}$, which satisfies $\prod_{i=1}^K \prod_{j=i+1}^K \Phi \lb  \sqrt{4\hn_i \hn_j+2} \cdot \lb A_{ij}(\psi_\ell)- A_{ij} \lb \frac{\htijl}{\hWbarijl} \rb \rb \rb \\ =1-\alpha^\prime_\ell$, where $\Phi(\cdot)$ is the cdf of the standard normal distribution, and we use the relation $\htijl=\hpijl \cdot \hWbarijl$.

As a result, if $\psi_\ell < \tLBw$, then $\htmaxl<\tLBw$ with probability at least $1-\alpha^\prime_\ell$. Note that verifying $\psi_\ell < \tLBw$ is equivalent to checking the condition
\begin{align}
\label{eqn_spectral_multi_confidence_interval_ingomogeneous_RIM_ML}
\prod_{i=1}^K \prod_{j=i+1}^K F_{ij} \lb \frac{\tLBw}{\Wbarijl},\hpijl \rb \geq 1-\alpha^\prime_\ell,
\end{align}
where 
\begin{align}
\label{eqn_eqn_spectral_multi_confidence_interval_ingomogeneous_RIM_ML_2}
F_{ij}(\frac{\tLBw}{\Wbarijl},\hpijl)&= \Phi \lb  \sqrt{4\hn_i \hn_j+2} \cdot \lb A_{ij}(\frac{\tLBw}{\Wbarijl}) - A_{ij}(\hpijl)\rb \rb \nonumber \\
&~~~\cdot \mathbb{I}_{\{\hpijl \in (0,1)\}}+ \mathbb{I}_{\{\htijl<\tLBw\}}\mathbb{I}_{\{\hpijl \in \{0,1\}\}},
\end{align}
and $\mathbb{I}_{E}$ is the event indicator function of an event $E$.
Finally, we replace   $\tLBw$ and $\Wbarijl$ in (\ref{eqn_eqn_spectral_multi_confidence_interval_ingomogeneous_RIM_ML_2}) with the empirical estimates $\htLBw$ and $\hWbarijl$, respectively, which leads to (\ref{eqn_spectral_multi_confidence_interval_ingomogeneous_RIM_ML}).

	\begin{figure*}[t]
			\vspace{-10mm}
		\centering
		\begin{subfigure}[b]{0.55\linewidth}
			\includegraphics[width=\textwidth]{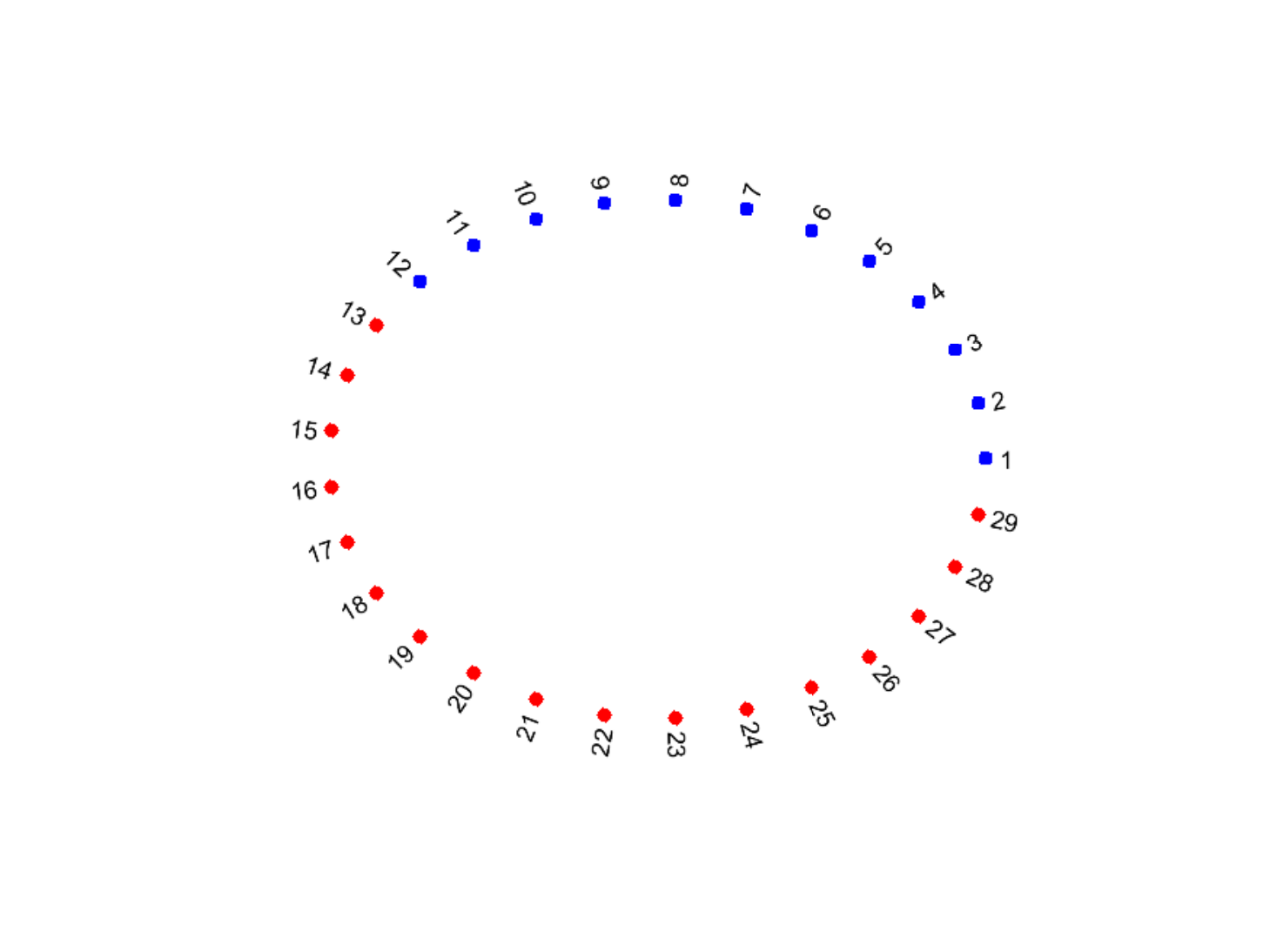}
			\vspace{-16mm}
			\caption{Ground-truth clusters}
		\end{subfigure}%
		\hspace{-22mm}							
		\centering
		\begin{subfigure}[b]{0.55\linewidth}
			\includegraphics[width=\textwidth]{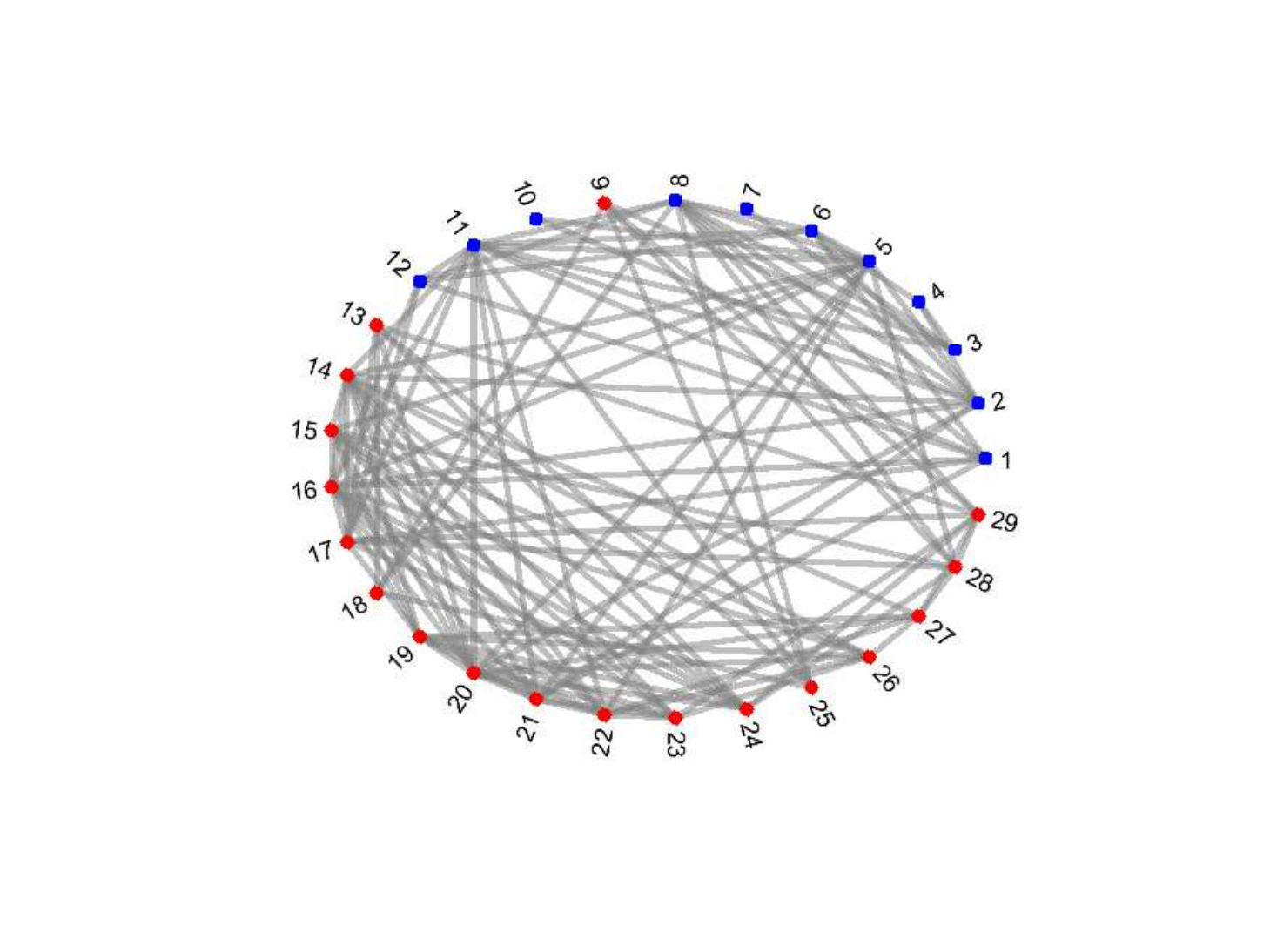}
			\vspace{-16mm}
			\caption{Friends you get on with}
		\end{subfigure}%
		\\
		\centering
		\begin{subfigure}[b]{0.55\linewidth}
			\includegraphics[width=\textwidth]{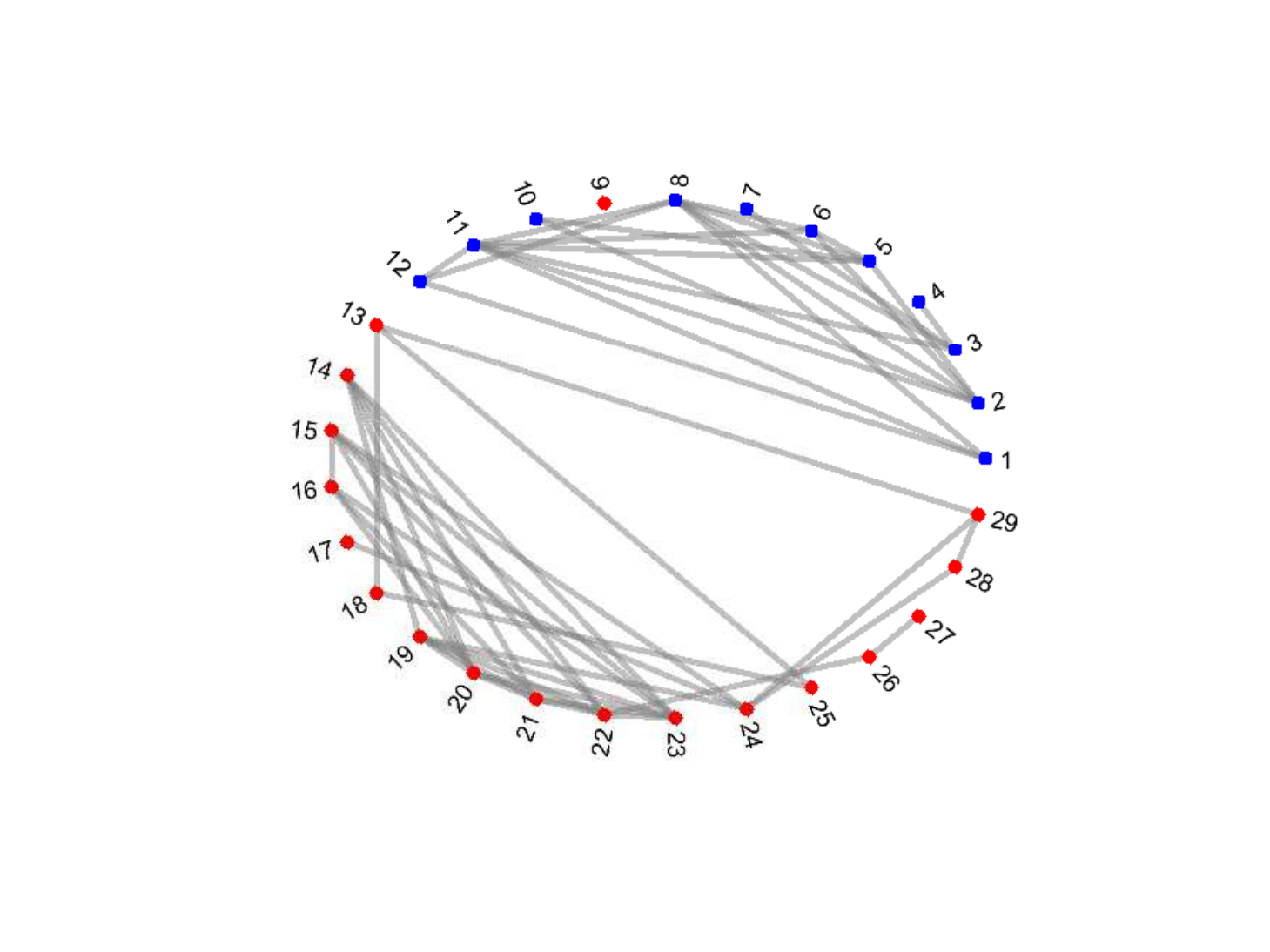}
			\vspace{-16mm}
			\caption{Your best friends}
		\end{subfigure}
		\hspace{-22mm}
		\centering
		\begin{subfigure}[b]{0.55\linewidth}
			\includegraphics[width=\textwidth]{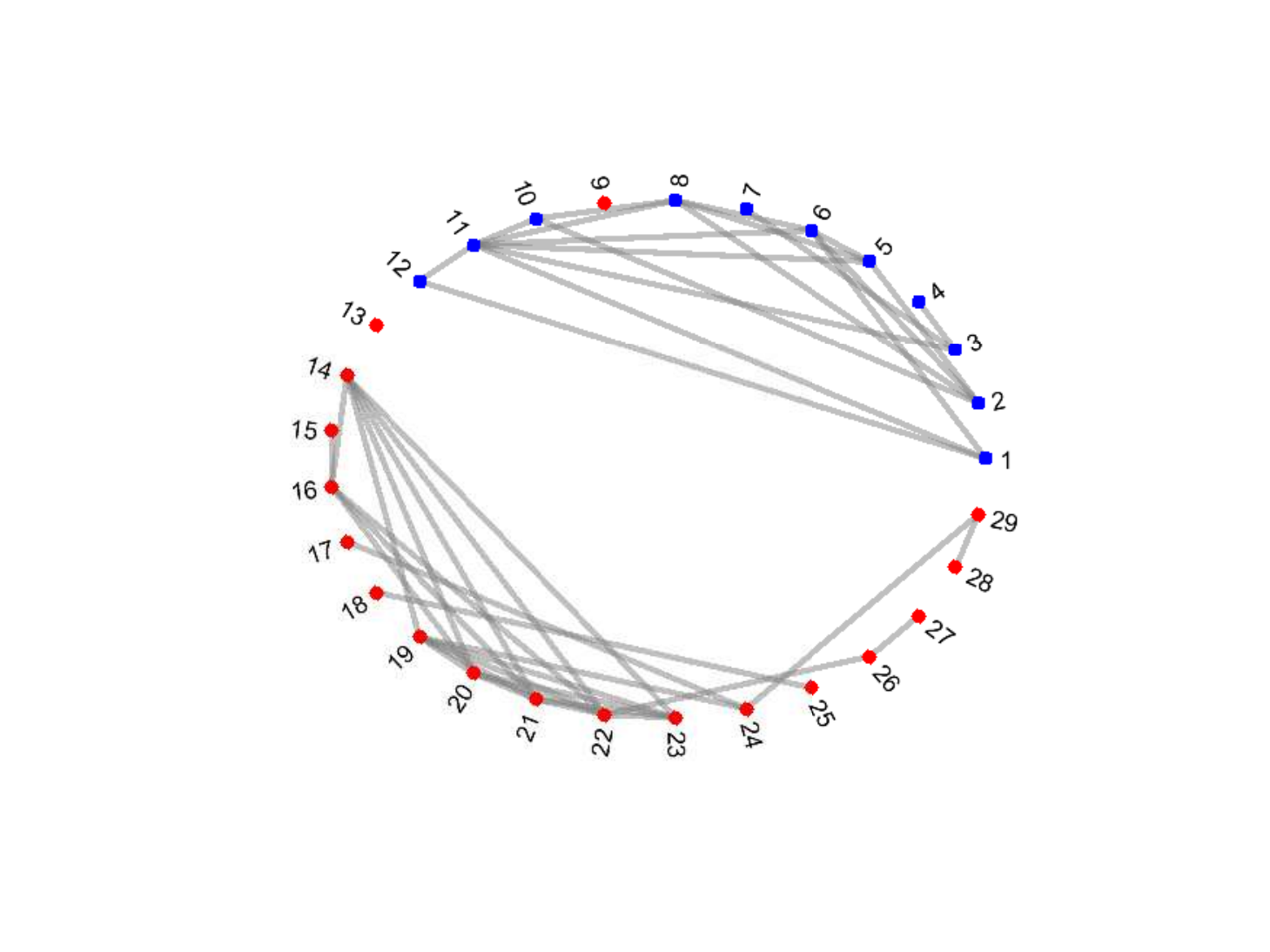}
			\vspace{-16mm}
			\caption{Friends you work with}
		\end{subfigure}		
		\caption{Illustration of the ground-truth clusters and the clusters found by MIMOSA for the VC 7th grader social network dataset (enlarged version).}	
		\label{Fig_VC_plot_2}
	\end{figure*}

%
%

\end{document}